\documentclass[msom]{INFORMS-IJOC-Template-26-02-2026/informs4} %
\usepackage{INFORMS-IJOC-Template-26-02-2026/eqndefns-left}
\Equationvalidatefalse
\RequirePackage{tgtermes}
\RequirePackage{newtxtext}
\RequirePackage{newtxmath}
\RequirePackage{bm}

\OneAndAHalfSpacedXI
\EquationsNumberedThrough
\TheoremsNumberedThrough
\ECRepeatTheorems

\usepackage{graphicx}
\setkeys{Gin}{draft=false}
\graphicspath{{INFORMS-IJOC-Template-26-02-2026/}{figures/}}
\usepackage{amsmath,amssymb}
\usepackage{tabularx,booktabs}
\usepackage{mathtools}
\usepackage{xcolor}
\usepackage{algorithm,algpseudocode}
\usepackage{xr-hyper}
\usepackage[hidelinks]{hyperref}
\usepackage[font=small]{caption}
\usepackage{subcaption}
\usepackage{tikz}
\usetikzlibrary{arrows.meta,calc}
\usetikzlibrary{positioning,fit,backgrounds}
\usepackage{pgfplots}
\usepackage{multirow}
\usepackage{xparse}
\pgfplotsset{compat=1.18}

\usepackage{natbib}
\bibpunct[, ]{(}{)}{,}{a}{}{,}
\def\bibfont{\small}

\newcommand{\Joren}[1]{}
\newcommand{\pat}[1]{}
\newcommand{\jan}[1]{}
\newcommand{\minner}[1]{}
\renewcommand{\Joren}[1]{\noindent\textcolor{blue}{$\ll$(JG) {#1}$\gg$}}
\renewcommand{\pat}[1]{\noindent\textcolor{red}{$\ll$(PH) {#1}$\gg$}}
\renewcommand{\jan}[1]{\noindent\textcolor{purple}{$\ll$(JD) {#1}$\gg$}}
\renewcommand{\minner}[1]{\noindent\textcolor{green}{$\ll$(pp.SM) {#1}$\gg$}}
\newcommand{\change}[1]{\noindent\textcolor{black}{#1}}
\NewDocumentCommand{\parencite}{O{} O{} m}{\citep[#1][#2]{#3}}
\NewDocumentCommand{\textcite}{O{} O{} m}{\citet[#1][#2]{#3}}

\algtext*{EndWhile}
\algtext*{EndFor}
\algtext*{EndIf}
\newcommand{\expect}[2][]{\mathbb{E}_{#1}\!\left[#2\right]}
\newcommand{\weights}{\boldsymbol{\theta}}

\usepackage{array}
\newenvironment{conditions*}
  {\par\vspace{\abovedisplayskip}\noindent
   \begin{tabular}{>{$}l<{$} @{\hspace{1.5em}} >{\raggedright\arraybackslash}p{0.78\linewidth}}}
  {\end{tabular}\par\vspace{\belowdisplayskip}}
\makeatletter
\providecommand{\setword}[2]{\phantomsection\def\@currentlabel{#1}\label{#2}#1}
\makeatother

\usepackage{enumitem}
\usepackage[normalem]{ulem}

\definecolor{projectblue}{RGB}{33,88,255}
\definecolor{gradientgray}{RGB}{145,145,145}
\definecolor{constraintorange}{RGB}{230,110,0}
\definecolor{constraintbrown}{RGB}{150, 75, 0}
\definecolor{feasiblegray}{RGB}{232,232,232}
\definecolor{softgreen}{RGB}{115,205,55}

\colorlet{fwdblue}{projectblue}
\colorlet{bwdgray}{gradientgray}
\colorlet{softblue}{feasiblegray}

\tikzset{
    block/.style={
        draw,
        rounded corners=3pt,
        line width=0.9pt,
        fill=white
    },
    fwd/.style={
        -{Latex[length=2.4mm,width=1.6mm]},
        line width=1.05pt,
        draw=fwdblue
    },
    bwd/.style={
        -{Latex[length=2.2mm,width=1.5mm]},
        line width=0.9pt,
        draw=bwdgray
    },
    bwdconn/.style={
        line width=0.85pt,
        draw=bwdgray
    },
    target action/.style={circle, fill=black, inner sep=1.5pt},
    projected action/.style={circle, fill=projectblue, inner sep=1.8pt},
    perturbed target/.style={circle, draw=black, fill=none, inner sep=1.5pt},
    perturbed projected/.style={circle, draw=projectblue, fill=none, inner sep=1.8pt},
    integer action/.style={circle, fill=green!70!black, inner sep=1.35pt},
    feasible region/.style={fill=feasiblegray},
    perturbed feasible region/.style={fill=feasiblegray!75},
    projection guide/.style={dashed, draw=gray},
    projection arrow/.style={
        -{Latex[length=1.4mm,width=1mm]},
        line width=1.05pt,
        draw=red
    },
    perturbation zarrow/.style={
        -{Latex[length=1.4mm,width=1mm]},
        line width=1.05pt,
        draw=black
    },
    perturbation xarrow/.style={
        -{Latex[length=1.4mm,width=1mm]},
        line width=1.05pt,
        draw=projectblue
    },
    neutral constraint/.style={line width=0.85pt, draw=black},
    highlighted constraint/.style={line width=0.85pt, draw=constraintbrown},
    perturbed constraint/.style={line width=0.85pt, draw=black!75},
    network item/.style={circle, draw, minimum size=5mm, inner sep=0pt, font=\small},
    resource box/.style={draw, rounded corners=0pt, inner sep=2pt},
    point/.style={target action},
    proj/.style={projected action},
    old/.style={black},
    new/.style={constraintbrown},
    feasible/.style={feasible region},
    guide/.style={projection guide},
    arr/.style={projection arrow},
    netnode/.style={network item},
    capbox/.style={resource box}
}

\newcommand{\NNblock}[3]{%
\begin{scope}[shift={({#1},{#2})}]
    \node[minimum width=2.05cm, minimum height=1.65cm, inner sep=0pt] (#3box) at (0,0) {};

    \foreach \ya in {0.48,0,-0.48}{
        \foreach \yb in {0.66,0.22,-0.22,-0.66}{
            \draw[black!60,thin] (-0.62,\ya) -- (0,\yb);
        }
    }
    \foreach \ya in {0.66,0.22,-0.22,-0.66}{
        \foreach \yb in {0.22,-0.22}{
            \draw[black!60,thin] (0,\ya) -- (0.62,\yb);
        }
    }

    \foreach \y in {0.48,0,-0.48}{
        \draw[thin,fill=white] (-0.62,\y) circle (0.13);
    }
    \foreach \y in {0.66,0.22,-0.22,-0.66}{
        \draw[thin,fill=white] (0,\y) circle (0.13);
    }
    \foreach \y in {0.22,-0.22}{
        \draw[thin,fill=white] (0.62,\y) circle (0.13);
    }
\end{scope}
}

\newcommand{\QPblockCoordinates}[1]{%
    \pgfmathsetmacro{\qpC}{(#1-2)*0.05}
    \pgfmathsetmacro{\qpD}{(#1-2)*0.04}

    \coordinate (qpAxisOrigin) at (-1.15,-0.72);
    \coordinate (qpXAxisEnd) at (0.94,-0.72);
    \coordinate (qpYAxisEnd) at (-1.15,0.95);
    \coordinate (qpPlotUpperRight) at (0.94,0.95);

    \coordinate (qpConstraintOneStart) at (-1.15,0.76+\qpC);
    \coordinate (qpConstraintOneEnd) at (0.94,-0.02+\qpC);
    \coordinate (qpConstraintTwoStart) at (-0.09+\qpD,0.95);
    \coordinate (qpConstraintTwoEnd) at (0.82+\qpD,-0.48);
    \coordinate (qpConstraintThreeStart) at (-1.15,0.5+\qpC);
    \coordinate (qpConstraintThreeEnd) at (0.5,0.5+\qpC);

    \coordinate (qpConstraintOneThreeIntersection) at (-0.45,0.50+\qpC);
    \coordinate (qpConstraintOneTwoIntersection) at (0.40+1.3*\qpD-0.8*\qpC,0.18-0.5*\qpD+1.3*\qpC);
    \coordinate (qpFeasibleLowerRight) at (0.82+\qpD,-0.72);

    \pgfmathsetmacro{\qpTx}{0.66+(#1-2)*0.10}
    \pgfmathsetmacro{\qpTy}{0.63+(#1-2)*(-0.06)}
    \coordinate (qpTargetPoint) at (\qpTx,\qpTy);
    \coordinate (qpTargetLabel) at (\qpTx+0.10,\qpTy+0.02);

    \pgfmathsetmacro{\qpPx}{0.47+(#1-2)*0.04}
    \pgfmathsetmacro{\qpPy}{0.07+(#1-2)*(-0.02)}
    \coordinate (qpProjectedPoint) at (\qpPx,\qpPy);
    \coordinate (qpProjectedLabel) at (\qpPx+0.06,\qpPy+0.09);

    \coordinate (qpFeasiblePoint) at (0.17,-0.15);
    \coordinate (qpFeasibleLabelTop) at (0.27,-0.06);
    \coordinate (qpFeasibleLabelBottom) at (0.27,-0.24);

    \def\qpContourXRadiusInner{0.21}
    \def\qpContourYRadiusInner{0.45}
    \def\qpContourXRadiusOptimal{0.322}
    \def\qpContourYRadiusOptimal{0.693}
    \def\qpContourXRadiusOuter{0.43}
    \def\qpContourYRadiusOuter{0.93}

    \def\qpFeasibleContourXRadiusInner{0.07}
    \def\qpFeasibleContourYRadiusInner{0.15}
    \def\qpFeasibleContourXRadiusMiddle{0.14}
    \def\qpFeasibleContourYRadiusMiddle{0.30}
    \def\qpFeasibleContourXRadiusOuter{0.21}
    \def\qpFeasibleContourYRadiusOuter{0.45}
}

\newcommand{\QPblock}[5]{%
\begin{scope}[shift={({#1},{#2})}]
    \QPblockCoordinates{#5}

    \draw[-{Latex[length=1.8mm,width=1.2mm]}, line width=0.8pt]
        (qpAxisOrigin) -- (qpXAxisEnd) node[right] {\scriptsize ${x}_{1t}$};
    \draw[-{Latex[length=1.8mm,width=1.2mm]}, line width=0.8pt]
        (qpAxisOrigin) -- (qpYAxisEnd) node[above] {\scriptsize ${x}_{2t}$};

    \draw[neutral constraint]
        (qpConstraintOneStart) -- (qpConstraintOneEnd);
    \draw[neutral constraint]
        (qpConstraintTwoStart) -- (qpConstraintTwoEnd);
    \draw[neutral constraint]
        (qpConstraintThreeStart) -- (qpConstraintThreeEnd);

    \fill[feasible region]
        (qpAxisOrigin) --
        (qpConstraintThreeStart) --
        (qpConstraintOneThreeIntersection) --
        (qpConstraintOneTwoIntersection) --
        (qpConstraintTwoEnd) --
        (qpFeasibleLowerRight) -- cycle;

    \begin{scope}
        \clip (qpAxisOrigin) rectangle (qpPlotUpperRight);
        \draw[black!35,densely dotted,line width=0.75pt]
            (qpTargetPoint)
            ellipse[x radius=\qpContourXRadiusInner,
                    y radius=\qpContourYRadiusInner];
        \draw[black!55,densely dotted,line width=1pt]
            (qpTargetPoint)
            ellipse[x radius=\qpContourXRadiusOptimal,
                    y radius=\qpContourYRadiusOptimal];
        \draw[black!35,densely dotted,line width=0.75pt]
            (qpTargetPoint)
            ellipse[x radius=\qpContourXRadiusOuter,
                    y radius=\qpContourYRadiusOuter];
    \end{scope}

    \draw[neutral constraint]
        (qpConstraintThreeStart) --
        (qpConstraintOneThreeIntersection) --
        (qpConstraintOneTwoIntersection) --
        (qpConstraintTwoEnd);

    \node[target action] at (qpTargetPoint) {};
    \node[anchor=west] at (qpTargetLabel) {\scriptsize $\mathbf{z}_{#4}$};

    \node[projected action] at (qpProjectedPoint) {};
    \node[anchor=north east] at (qpProjectedLabel) {\scriptsize $\mathbf{x}_{#4}^{*}$};

    \draw[projection arrow] (qpTargetPoint) -- (qpProjectedPoint);
\end{scope}
}

\newcommand{\QPfeasblock}[5]{%
\begin{scope}[shift={({#1},{#2})}]
    \QPblockCoordinates{#5}

    \draw[-{Latex[length=1.8mm,width=1.2mm]}, line width=0.8pt]
        (qpAxisOrigin) -- (qpXAxisEnd) node[right] {\scriptsize $\mathbf{x}_{1t}$};
    \draw[-{Latex[length=1.8mm,width=1.2mm]}, line width=0.8pt]
        (qpAxisOrigin) -- (qpYAxisEnd) node[above] {\scriptsize $\mathbf{x}_{2t}$};

    \draw[neutral constraint]
        (qpConstraintOneStart) -- (qpConstraintOneEnd);
    \draw[neutral constraint]
        (qpConstraintTwoStart) -- (qpConstraintTwoEnd);
    \draw[neutral constraint]
        (qpConstraintThreeStart) -- (qpConstraintThreeEnd);

    \fill[feasible region]
        (qpAxisOrigin) --
        (qpConstraintThreeStart) --
        (qpConstraintOneThreeIntersection) --
        (qpConstraintOneTwoIntersection) --
        (qpConstraintTwoEnd) --
        (qpFeasibleLowerRight) -- cycle;

    \begin{scope}
        \clip (qpAxisOrigin) rectangle (qpPlotUpperRight);
        \draw[black!25,densely dotted,line width=0.55pt]
            (qpFeasiblePoint)
            ellipse[x radius=\qpFeasibleContourXRadiusInner,
                    y radius=\qpFeasibleContourYRadiusInner];
        \draw[black!35,densely dotted,line width=0.55pt]
            (qpFeasiblePoint)
            ellipse[x radius=\qpFeasibleContourXRadiusMiddle,
                    y radius=\qpFeasibleContourYRadiusMiddle];
        \draw[black!25,densely dotted,line width=0.55pt]
            (qpFeasiblePoint)
            ellipse[x radius=\qpFeasibleContourXRadiusOuter,
                    y radius=\qpFeasibleContourYRadiusOuter];
    \end{scope}

    \draw[neutral constraint]
        (qpConstraintThreeStart) --
        (qpConstraintOneThreeIntersection) --
        (qpConstraintOneTwoIntersection) --
        (qpConstraintTwoEnd);

    \fill (qpFeasiblePoint) circle (1.9pt);
    \fill[projectblue] (qpFeasiblePoint) circle (1.2pt);
    \node[anchor=north east] at (qpFeasibleLabelTop) {\scriptsize $\mathbf{z}_{#4} = \mathbf{x}_{#4}^{*}$};
\end{scope}
}

\newcommand{\IMblockCoordinates}[1]{%
    \pgfmathsetmacro{\imC}{(#1-2)*0.04}

    \def\imGridValues{-0.72,0,0.72}
    \def\imGuideMinX{-0.90}
    \def\imGuideMaxX{0.90}
    \def\imGuideMinY{-0.84}
    \def\imGuideMaxY{0.84}

    \coordinate (imConstraintStart) at (0.10+\imC, 0.80);
    \coordinate (imConstraintEnd)   at (0.80+\imC,-0.60);

    \pgfmathsetmacro{\imPx}{(#1<1.5) ? 0.30 : ((#1<2.5) ? 0.20 : 0.62)}
    \pgfmathsetmacro{\imPy}{(#1<1.5) ? 0.32 : ((#1<2.5) ? 0.30 : -0.18)}
    \coordinate (imProjectedPoint) at (\imPx,\imPy);
    \coordinate (imProjectedLabel) at (\imPx-0.03,\imPy-0.05);

    \coordinate (imProjectionStart) at (0.17,0.13);
    \coordinate (imProjectionEnd)   at (-0.02,-0.02);

    \pgfmathsetmacro{\imIntX}{(#1>3.5) ? 0.72 : 0}
    \pgfmathsetmacro{\imIntY}{(#1<1.5) ? 0.72 : 0}
    \coordinate (imIntegerPoint) at (\imIntX,\imIntY);
    \coordinate (imIntegerLabel) at (\imIntX+0.15,\imIntY-0.16);
}

\newcommand{\IMblock}[5]{%
\begin{scope}[shift={({#1},{#2})}]
    \node[minimum width=2.35cm, minimum height=2.2cm, inner sep=0pt] (#3box) at (0,0) {};
    \IMblockCoordinates{#5}

    \foreach \x in \imGridValues{
        \draw[black!25,densely dotted] (\x,\imGuideMinY) -- (\x,\imGuideMaxY);
    }
    \foreach \y in \imGridValues{
        \draw[black!25,densely dotted] (\imGuideMinX,\y) -- (\imGuideMaxX,\y);
    }

    \ifnum#5=2\relax\else
        \draw[neutral constraint] (imConstraintStart) -- (imConstraintEnd);
    \fi

    \foreach \x in \imGridValues{
        \foreach \y in \imGridValues{
            \fill (\x,\y) circle (1.0pt);
        }
    }

    \node[projected action] at (imProjectedPoint) {};
    \node[anchor=west] at (imProjectedPoint) {\scriptsize $\mathbf{x}_{#4}^{*}$};

    \draw[projection arrow] (imProjectedPoint) -- (imIntegerPoint);
    \node[integer action] at (imIntegerPoint) {};
    \node[anchor=north] at (imIntegerPoint) {\scriptsize $\mathbf{x}_{#4}$};
\end{scope}
}

\renewcommand{\TITLEfont}{\fs.17.5.\rm}

\RUNAUTHOR{Helm et al.}
\RUNTITLE{Learning Feasible Policies via Differentiable Projection}
\TITLE{Hard Constraints, Smooth Gradients: Learning Feasible Inventory Policies via Differentiable Projection}

\ARTICLEAUTHORS{%
\AUTHOR{Patrick P. Helm, Jan-Niklas Doerr}
\AFF{TUM School of Management, Technical University of Munich, \EMAIL{patrick.helm@tum.de, jan.doerr@tum.de}}

\AUTHOR{Joren Gijsbrechts}
\AFF{Esade, Ramon Llull University, \EMAIL{joren.gijsbrechts@esade.edu}}

\AUTHOR{Stefan Minner}
\AFF{TUM School of Management, Technical University of Munich, \EMAIL{stefan.minner@tum.de}}
\AFF{Munich Data Science Institute (MDSI), Technical University of Munich}
}

\ABSTRACT{\textbf{Problem definition:} Many operational problems are constrained sequential decision processes with combinatorial action spaces and interdependent constraints. Mixed-integer linear programs (MILPs) handle such constraints flexibly but scale poorly in stochastic environments. Deep reinforcement learning (DRL) promises scalable decision rules, but existing methods either penalize constraints rather than enforce them, or rely on feasibility mechanisms that break down once constraints interact. \textbf{Methodology/results:} We bridge this gap by embedding a differentiable optimization module inside the policy: a differentiable decision rule (e.g., a neural network) proposes continuous action targets, a quadratic program projects them onto the relaxed feasible set, and a dual-informed integer mapping restores integrality while preserving feasibility. Given a differentiable simulator, the policy trains end-to-end from sampled trajectories using pathwise gradients, while handling hard constraints with similar flexibility to MILPs. We show that our approach has bounded error relative to an exact integer projection and ensures the entire feasible action space is reachable. We apply the method to multi-echelon production--inventory planning under shared resource and material constraints. Our policy attains an average optimality gap {of $1.0\%$} on 3-item instances and consistently outperforms a rolling-horizon multi-stage stochastic program by at least {$5.6\%$} in 10-item networks. {Compared to state-of-the-art echelon base-stock policies, savings} reach {$15.3\%$} on 40-item instances. On an industry-scale case study with nonstationary demand from ASML, we reduce average cost by up to {$1.8\%$} relative to the best-known benchmark policy. \textbf{Managerial implications:} Our results show that traditional inventory policies struggle in tightly capacitated systems with high demand variability, while our learning-based approach remains cost-effective and runs orders of magnitude faster than MILPs. More broadly, we {show} that DRL can deliver economically significant savings compared to state-of-the-art baselines in sequential decision problems with interdependent hard constraints, which are widespread in practice.
}


\FUNDING{This work received financial support from Deutsche Forschungsgemeinschaft (DFG) -- Projektnummer 277991500.}

\makeatletter
\def\theARTICLETOPLEFT{}
\def\theARTICLETOPRIGHT{}
\def\theARTICLEABSTRACT{%
  \HOOKb
  \vspace*{18pt}
  \begin{minipage}[t]{\textwidth}\parindent1em
    \ABSfont
    \noindent\theABSTRACT\endgraf
    \vskip5pt
    \if@BLINDREV\else\theFUNDING\fi
    \theKEYWORDS
    \theSUBJECTCLASS
    \theAREAOFREVIEW
    \theMSCCLASS
    \theORMSCLASS
    \if@BLINDREV\else\theHISTORY\fi
    \noindent\hrulefill
  \end{minipage}%
}
\if@BLINDREV
\RRHFirstLine{{\it\theRUNTITLE}}
\LRHFirstLine{{\it\theRUNTITLE}}
\else
\RRHFirstLine{\bf\theRUNAUTHOR:\enskip {\it\theRUNTITLE}}
\LRHFirstLine{\bf\theRUNAUTHOR:\enskip {\it\theRUNTITLE}}
\fi
\RRHSecondLine{}
\LRHSecondLine{}
\def\setoddRH{\hbox to \textwidth{\fs.7.8.\tabcolsep0pt
  \begin{tabular*}{\textwidth}[b]{l@{\extracolsep\fill}r}
  {\theRRHFirstLine}&\raisebox{0pt}[0pt][0pt]{\fs.10.10.\thepage}\\[-2pt]
  \rlap{\VRHDW{0.5pt}{0pt}{\textwidth}}&\\
  \end{tabular*}}}
\def\setevenRH{\hbox to \textwidth{\fs.7.8.\tabcolsep0pt
  \begin{tabular*}{\textwidth}[b]{l@{\extracolsep\fill}r}
  \raisebox{0pt}[0pt][0pt]{\fs.10.10.\thepage}&{\theLRHFirstLine}\\[-2pt]
  \rlap{\VRHDW{0.5pt}{0pt}{\textwidth}}&\\
  \end{tabular*}}}
\def\theARTICLETITLE{%
  \HOOKtop
  \begin{center}
  \vspace*{0pt}%
  \TITLEfont\HD{24}{0}\sffamily\bfseries\theTITLE\HD{0}{15}%
  \end{center}}
\def\AUTHOR#1{%
  \begin{center}
  \AUTHORfont\HD{15}{0}#1\HD{0}{6}\relax
  \end{center}}
\def\AFF#1{%
  \begin{center}
  \AFFfont{#1}\relax
  \vskip1.6pt
  \end{center}}
\def\ps@firstpage{%
  \let\@mkboth\markboth
  \def\@oddfoot{}%
  \def\@evenfoot{}%
  \def\@oddhead{}%
  \def\@evenhead{}%
  \def\sectionmark##1{}%
  \def\subsectionmark##1{}%
}
\makeatother

\begin{document}
\pagestyle{headings}

\maketitle

\setlength{\abovedisplayskip}{2pt}
\setlength{\belowdisplayskip}{2pt}
\setlength{\abovedisplayshortskip}{0pt}
\setlength{\belowdisplayshortskip}{0pt}
\setlength{\jot}{2pt}

\section{Introduction}

Deep Reinforcement Learning (DRL), which uses deep neural networks to optimize sequential decisions under uncertainty, is emerging as a complement to classical Operations Research (OR) tools. Reported gains in large-scale supply chain applications include a 12\% reduction in inventory at Amazon without revenue loss \parencite{Madeka.2022}, lower out-of-stock rates and turnover times at Alibaba \parencite{Liu.2023, Xie.2026}, and a 40\% cost reduction at JD.com \parencite{Qi.2023}. These results show that DRL can handle scale and stochasticity in real supply chains, but only in settings without hard combinatorial action constraints.

For constrained problems, however, mixed-integer linear programs (MILPs) remain the workhorse, as they encode operational requirements directly in the formulation and rely on a solver to guarantee feasibility. State-of-the-art DRL algorithms generally lack this capability. This is a central gap, because many OR problems involve sequential decisions which jointly compete for shared resources.

Besides constraints, a second defining feature of many sequential OR problems is that the transitions and costs are differentiable with respect to the state and action. This enables reliable policy optimization via pathwise gradients, which are more sample-efficient and less noisy than statistical policy-gradient estimates. Exemplary applications include queueing-network control \parencite{Che.2024}, financial hedging \parencite{Buehler.2019}, predictive process control \parencite{Drgona.2022}, and inventory management \parencite{Bottcher.2023}. In the absence of constraints, Hindsight Differentiable Policy Optimization (HDPO) is the leading approach for learning parameterized differentiable policies for these problems \parencite{Madeka.2022, Alvo.2023}. 

Incorporating hard action constraints into differentiable policies nevertheless remains an open challenge. Lagrangian relaxation penalizes violations but does not enforce strict feasibility \parencite{Eisenach.2024, Liu.2025}, while the differentiable scaling mechanism of \textcite{Alvo.2023} applies only when scaling is sufficient to restore feasibility. Figure~\ref{fig:constraints} illustrates this limitation. In panel~(a), a single joint constraint can be made feasible by scaling the target action toward the origin. In panel~(b), multiple constraints are present so it is not clear a priori which ones are binding after feasibility is enforced. In general, identifying the active constraints requires solving an optimization problem, which must itself be differentiable for gradients to flow during training.

\begin{figure}[t]
\centering

\begin{subfigure}[t]{0.34\textwidth}
\centering
\vspace{0pt}
\begin{tikzpicture}[scale=0.6, >=stealth]
    \draw[->] (0,0) -- (3.45,0) node[right, font=\scriptsize, align=left] {dimension 1};
    \draw[->] (0,0) -- (0,2.7) node[above, font=\scriptsize, align=center] {dimension 2};

    \coordinate (A1) at (0,2.25);
    \coordinate (B1) at (2.25,0);
    \coordinate (A2) at (0,1.65);
    \coordinate (B2) at (1.65,0);

    \fill[feasible] (0,0) -- (A2) -- (B2) -- cycle;

    \draw[neutral constraint] (A2) -- (B2);

    \coordinate (z1) at (2.05,1.35);
    \coordinate (x1) at (1.025,0.675);
    \coordinate (z2) at (1.15,2.05);
    \coordinate (x2) at (0.575,1.025);

    \draw[guide] (z1) -- (x1);
    \draw[guide] (z2) -- (x2);
    \draw[arr] (z1) -- (x1);
    \draw[arr] (z2) -- (x2);

    \node[point, old] at (z1) {};
    \node[proj] at (x1) {};
    \node[point, old] at (z2) {};
    \node[proj] at (x2) {};

\end{tikzpicture}
\caption{Single shared resource.}
\end{subfigure}
\hfill
\begin{minipage}[t]{0.28\textwidth}
\vspace{0pt}
\centering
\begin{tikzpicture}[>=stealth, font=\footnotesize]
    \node[anchor=west, font=\footnotesize\bfseries] at (0, 0) {Action space};

    \node[point] at (0.25, -0.45) {};
    \node[anchor=west] at (0.55, -0.45) {target action};

    \node[proj] at (0.25, -0.85) {};
    \node[anchor=west] at (0.55, -0.85) {projected action};

    \draw[arr] (0.1, -1.25) -- (0.4, -1.25);
    \node[anchor=west] at (0.55, -1.25) {projection step};

    \draw[neutral constraint] (0.1, -1.65) -- (0.4, -1.65);
    \node[anchor=west, align=left] at (0.55, -1.65) {shared constraint};

    \draw[highlighted constraint] (0.1, -2.05) -- (0.4, -2.05);
    \node[anchor=west, align=left] at (0.55, -2.05) {individual constraints};

    \fill[feasible] (0.1, -2.55) rectangle (0.4, -2.35);
    \node[anchor=west] at (0.55, -2.45) {feasible action set};

\end{tikzpicture}
\end{minipage}
\hfill
\begin{subfigure}[t]{0.34\textwidth}
\centering
\vspace{0pt}
\begin{tikzpicture}[scale=0.6, >=stealth]
    \draw[->] (0,0) -- (3.45,0) node[right, font=\scriptsize, align=left] {dimension 1};
    \draw[->] (0,0) -- (0,2.7) node[above, font=\scriptsize, align=center] {dimension 2};

    \coordinate (T)  at (0,2.25);
    \coordinate (R)  at (2.25,0);
    \coordinate (Vt) at (1.85,1.8);
    \coordinate (Vr) at (1.85,0);
    \coordinate (Hl) at (0,1.25);
    \coordinate (Hr) at (2.5,1.25);

    \coordinate (P1) at (0,0);
    \coordinate (P2) at (0,1.25);
    \coordinate (P3) at (1.0,1.25);
    \coordinate (P4) at (1.85,0.4);
    \coordinate (P5) at (1.85,0);

    \fill[feasible] (P1) -- (P2) -- (P3) -- (P4) -- (P5) -- cycle;

    \draw[neutral constraint] (T) -- (R);
    \draw[highlighted constraint] (Vt) -- (Vr);
    \draw[highlighted constraint] (Hl) -- (Hr);

    \coordinate (z1) at (2.4,0.95);
    \coordinate (x1) at (1.85,0.4);
    \coordinate (z2) at (0.48,2.25);
    \coordinate (x2) at (0.48,1.25);

    \draw[guide] (z1) -- (x1);
    \draw[guide] (z2) -- (x2);
    \draw[arr] (z1) -- (x1);
    \draw[arr] (z2) -- (x2);

    \node[point, old] at (z1) {};
    \node[proj] at (x1) {};
    \node[point, old] at (z2) {};
    \node[proj] at (x2) {};

\end{tikzpicture}
\caption{Individual and shared resources.}
\end{subfigure}

\caption{Feasible action sets for a two-dimensional action space with shared and individual constraints. In (a), a shared constraint limits sum of actions, so scaling restores feasibility. In (b), dedicated constraints limit individual actions. The active constraints depend on the target, so restoring feasibility requires solving an optimization problem.}
\label{fig:constraints}
\end{figure}

To address action feasibility under state-dependent, interdependent constraints, we propose a differentiable policy with three sequential modules (Figure~\ref{fig:pipeline_intro}). First, a differentiable decision rule (in our case a neural network) maps the system state to continuous action targets. Second, a convex differentiable projection module maps these targets onto the feasible set of the continuous relaxation and provides analytical sensitivities via implicit differentiation of its optimality conditions. We use a quadratic program (QP) rather than a linear program (LP) relaxation because strong convexity yields a unique optimizer that varies smoothly with the QP inputs, whereas LP optima are prone to degeneracy and discontinuous jumps that destabilize sensitivity-based learning \parencite{Wilder.2019}. We show that the projection sensitivities flowing through the action target are nonexpansive, and thus do not cause gradient instability in the backward pass. Third, a dual-informed integer mapping converts the continuous projected action into feasible discrete decisions by using the projection's dual variables to prioritize items. Because discrete operations do not provide useful learning signals, we provide surrogate gradients via a straight-through estimator (STE) \parencite{Bengio.2013}, enabling end-to-end gradient propagation. Because the constraints enter the policy only through the specification of the projection, additional or alternative resource structures leave the architecture untouched. The policy thereby retains the formulation-driven workflow of MILPs, where operational requirements are added to the model rather than to the algorithm.

We provide three theoretical guarantees concerning the discrete policy output. First, the two-step mapping stays close to the exact integer projection, both in the projection distance and in the resulting decision, which supports avoiding per-state integer programming. Second, the integer mapping is complete: every feasible action is selected on a positive-measure set of target actions, and hence within reach of a smoothly parameterized decision rule, whereas naive flooring can leave actions that fully utilize some resource unreachable. Third, when the target overshoots the projected action in every coordinate, the mapping returns an action that is limited in at least one coordinate, so the decision rule can push the action up against binding resource limits.

We apply this method to inventory management, which is one of the most developed OR applications, where DRL already has real-world impact \parencite{Gijsbrechts.2026}. Specifically, we consider production--inventory planning in multi-echelon supply networks with finite resources. {This setting combines a combinatorial action space with interdependent hard constraints.} Manufacturers must coordinate production decisions across products and locations while accounting for inventory positions, shared capacities, material flows, and long lead times \parencite{deKok.2018}. The problem arises, for example, in semiconductor manufacturing at ASML \parencite{Fleuren.2025}, on whose benchmark we later validate our method. No optimal policy structure or {universally well-performing} heuristic is known for this problem, precisely because the combinatorial constrained action space is not natively handled by existing methods. {Because our policy solves a single-period QP per decision rather than optimizing over a lookahead horizon and scenario tree, it generates an action within milliseconds even for networks with 100 items, more than $100{,}000$ times faster than a multi-stage stochastic program.}

Our method attains an average optimality gap {of $1.0\%$} on 3-item instances without instance-specific hyperparameter tuning. On instances based on the supply networks from \textcite{Tempelmeier.1996}, we create cost savings of up to {$9.9\%$} compared to the state-of-the-art optimized echelon base-stock policy and consistently outperform a rolling-horizon multi-stage stochastic program by at least {$5.6\%$} for 10 items. {The maximum savings over the base-stock policy increases for 40-item instances, reaching up to $15.3\%$}. The advantage is strongest in tightly capacitated networks with high demand variability and resources shared across echelons, where the base-stock policy struggles. We further evaluate the policy on the 14-item instance of the 2025 ASML Research Challenge, a realistic industry benchmark with nonstationary demand \parencite{ASML_supply_chain_planning_challenge}. Across {four problem variants}, our policy reduces average cost by {$0.9$--$1.8\%$} relative to the best-known policy for this problem, namely a forecast-based echelon base-stock policy with min-max relative shortfall allocation \parencite{vanDijck.2026}. {Compared to rolling-horizon multi-stage stochastic programming, our method creates savings of $2.4$--$5.5\%$.} Whereas the baselines require forecasts of future demand to make decisions, our method learns a data-driven mapping from demand features directly to production decisions. An ablation study confirms that all policy modules are essential: ignoring the projection sensitivities causes training divergence, and replacing the dual-informed {integer mapping} with naive floor rounding increases cost by up to {$35.9\%$}, consistent with our theoretical results.

\begin{figure}[t]
\centering
\begin{tikzpicture}[font=\small]

\def\gap{2.0}

\def\xnn{2.00}
\pgfmathsetmacro{\xqp}{\xnn + 1.10 + \gap + 1.58}
\pgfmathsetmacro{\xim}{\xqp + 1.58 + \gap + 1.26}
\pgfmathsetmacro{\xend}{\xim + 1.26 + 1.50}

\def\yone{0}
\def\ybwd{-0.36}

\NNblock{\xnn}{\yone}{nnA}
\QPblock{\xqp}{\yone}{qpA}{t}{3}
\IMblock{\xim}{\yone}{imA}{t}{3}

\draw[fwd] (0.10,\yone) -- (\xnn-1.10,\yone)
    node[midway,above=2pt,black,font=\footnotesize,align=center] {State};
\draw[fwd] (\xnn+1.10,\yone) -- (\xqp-1.58,\yone)
    node[midway,above=2pt,black,font=\footnotesize,align=center]
        {proposed\\[-2pt]decision};
\draw[fwd] (\xqp+1.58,\yone) -- (\xim-1.26,\yone)
    node[midway,above=2pt,black,font=\footnotesize,align=center]
        {continuous\\[-2pt]feasible\\[-2pt]decision};
\draw[fwd] (\xim+1.26,\yone) -- (\xend,\yone)
    node[midway,above=2pt,black,font=\footnotesize,align=center]
        {integer\\[-2pt]feasible\\[-2pt]decision};

\draw[bwd] (\xend,\ybwd)      -- (\xim+1.26,\ybwd);
\draw[bwd] (\xim-1.26,\ybwd)  -- (\xqp+1.58,\ybwd);
\draw[bwd] (\xqp-1.58,\ybwd)  -- (\xnn+1.10,\ybwd)
    node[midway,below=2pt,black,font=\footnotesize] {QP gradients};
\draw[bwd] (\xnn-1.10,\ybwd)  -- (0.10,\ybwd);

\draw[bwd, rounded corners=6pt]
    (\xim+1.26,\ybwd) --
    (\xim+1.26,-1.00) --
    (\xim-1.26,-1.00) --
    (\xim-1.26,\ybwd);
\node[font=\footnotesize,black,align=center] at (\xim,-1.50)
    {straight-through\\[-2pt]gradients};

\begin{scope}[shift={(1.3,-1.5)}]
    \draw[fwd] (-1.12,0.18) -- (-0.32,0.18);
    \node[anchor=west] at (-0.4,0.18) {forward};
    \draw[bwd] (-0.32,-0.16) -- (-1.12,-0.16);
    \node[anchor=west] at (-0.4,-0.16) {backward};
\end{scope}

\end{tikzpicture}

\caption{Overview of the proposed differentiable policy. A neural network proposes action targets, which are first projected onto the relaxed feasible set by a QP, and then mapped to discrete decisions. Gradients flow end-to-end using surrogate gradients for the non-differentiable integer mapping.}
\label{fig:pipeline_intro}
\end{figure}

In sum, our main contributions are:
\begin{itemize}[leftmargin=*, align=parleft]
    \item A differentiable policy architecture that produces feasible discrete production decisions under interdependent action constraints{, while empirically showing sublinear inference-time scaling with the number of items.} It integrates a differentiable decision rule (e.g., a neural network), a convex projection onto the continuous feasible set, and a dual-informed integer mapping, enabling end-to-end training via pathwise gradients without requiring per-state mixed-integer optimization.
    \item {Theoretical guarantees on the projection and integer-mapping modules. We bound the projection optimality and decision-distance gaps introduced by integer mapping, show that every feasible integer action is reachable from a positive-measure target region, and provide a sufficient condition on the target under which the policy chooses a frontier action.}
    \item Numerical validation across three sets of experiments of increasing scale and complexity: small instances with known optimal policies, literature benchmarks based on \textcite{Tempelmeier.1996}, and a realistic industry-scale instance from the 2025 ASML Research Challenge including an ablation study of the policy modules.
\end{itemize}

\section{Literature}
\label{sec:literature}

{Our work connects three streams of literature: capacitated multi-echelon inventory planning defines the application domain (Section~\ref{sec:literature_multi_echelon}), DRL for (constrained) inventory control provides the methodological context (Section~\ref{sec:literature_drl}), and decision-focused learning supplies the technical backbone (Section~\ref{sec:literature_diff_opt}). We review each in turn and position our contribution at their intersection (Section~\ref{sec:gap}).}

\subsection{Multi-Echelon Inventory Planning in Capacitated Networks}
\label{sec:literature_multi_echelon}

{The resource-constrained production--inventory application we consider in this paper belongs to the class of capacitated multi-echelon inventory planning problems. While capacitated inventory problems have been extensively studied in the literature \parencite{Song.2020}, their extension to multi-echelon settings has largely been analyzed for specific network structures \parencite{deKok.2018}.} 

In convergent or serial systems, most contributions focus on echelon base-stock policies, which are optimal in uncapacitated systems under standard assumptions \parencite{Clark.1960}. Near-optimal base-stock levels in capacitated systems have been derived heuristically by \textcite{Huh.2016} and using infinitesimal perturbation analysis by \textcite{Glasserman.1995}. Recently, \textcite{vanDijck.2024} propose a balanced echelon base-stock policy for capacitated assembly systems and demonstrate promising performance in a case study from ASML. In divergent and general networks, allocation policies determine how upstream stock is distributed across downstream demand streams. \textcite{VanDerHeijden.1997} find that balanced stock rationing performs best in pure divergent systems, especially when downstream service is highly differentiated. \textcite{handbook} extend this to balanced linear rationing for general supply networks and show that it outperforms other policies. However, it remains unclear how these allocation rules can be extended to settings with shared resource constraints. 

General capacitated networks have predominantly been addressed in the multi-level dynamic lot-sizing literature under deterministic demand \parencite{Tempelmeier.1996, Malicki.2026}. However, \textcite{Thevenin.2021} demonstrate that deterministic MILPs are inadequate under stochastic demand and propose rolling-horizon stochastic programming (SP), which yields substantial cost reductions. \textcite{Fleuren.2025} extend this to stochastic lead times and new product introductions, reporting state-of-the-art performance on a case study from ASML. {However, \textcite{vanDijck.2026} show that this approach breaks down at industry scale: on the 2025 ASML Research Challenge, rolling-horizon multi-stage SP is outperformed by an optimized forecast-based echelon base-stock policy with a shortfall-minimizing allocation rule, while requiring multiple orders of magnitude more compute. Since this base-stock policy is the best-known policy for the studied problem class, we use it as our main benchmark alongside rolling-horizon SP in Section \ref{sec:experiments}.}

\subsection{Deep Reinforcement Learning for Inventory Control}
\label{sec:literature_drl}

{Inventory control is one of the most developed application areas for DRL in OR, where real-world impact is already visible \parencite{Gijsbrechts.2026}. Early work used DRL mainly as a flexible policy-learning tool for classical inventory settings \parencite{Oroojlooyjadid.2022, Gijsbrechts.2022}. These studies show that DRL can learn competitive replenishment policies without explicitly deriving problem-specific rules. While the discrete action encodings used in these early approaches make hard constraints easy to enforce through masking, they are intractable to scale when many items or locations must be managed jointly.} 

\textcite{Kaynov.2024} instead {rely} on random rationing to enforce feasibility in a one-warehouse multiple-retailer (OWMR) setting, which enables a more scalable action representation. While this nudges the policy to learn feasible actions, the allocation does not provide a direct gradient signal for learning. \textcite{Harsha.2025} explicitly {treat} combinatorial constrained action spaces by replacing the parametric policy with a stochastic MILP guided by a neural value function. While proving successful for three-echelon distribution systems with up to 6 locations, the method relies on a MILP with an embedded neural network for action selection, which makes computational effort grow with the expressiveness needed for accurate predictions.

A separate research stream uses pathwise gradients from differentiable simulators for end-to-end inventory policy learning. \textcite{Madeka.2022} propose HDPO for inventory control and demonstrate state-of-the-art performance in data-driven optimization with lost sales and stochastic lead times based on a large-scale case study at Amazon. \textcite{Eisenach.2024} extend this to settings with dynamic storage capacity using a Lagrangian relaxation, where a second neural network predicts state-dependent dual prices to steer the policy toward feasibility. \textcite{Liu.2025} generalize this idea to dual-sourcing settings with multiple suppliers. Both approaches effectively maintain inventory levels close to capacity in expectation. However, these methods rely on penalizing violations, which does not guarantee feasibility of individual actions, making them incompatible with hard constraints. 

Since order quantities typically need to be integer in practice, different approaches have been proposed for making discrete decision-making differentiable. \textcite{Bottcher.2023} learn near-optimal policies on a range of dual-sourcing problems and show that straight-through gradient estimators are effective at smoothing non-differentiable rounding. In contrast, \textcite{Alvo.2023} train policies in a continuous surrogate system and rely on post-processing at deployment time to obtain discrete decisions, which consistently outperforms the policy of \textcite{Gijsbrechts.2022} on a variety of settings. For the OWMR setting, they propose a differentiable scaling-based feasibility-enforcement mechanism that is effective for a single constraint but does not generalize to settings with more complex interdependent action constraints (see Figure \ref{fig:constraints}). 

\subsection{Decision-Focused Learning and Differentiable Optimization}
\label{sec:literature_diff_opt}
Decision-focused learning (DFL) studies the integration of machine learning and constrained optimization into a single decision pipeline. In contrast to classical predict-then-optimize, DFL methods train models with respect to downstream decision quality after the learned outputs have been processed by an optimization problem. In our setting, the main purpose of the optimization is to enforce hard, interdependent action constraints. In their recent survey, \textcite{Mandi.2024} distinguish two main ways of obtaining gradients when the downstream decision problem is discrete or otherwise non-differentiable: replacing the decision loss with a surrogate, or differentiating through a relaxed optimization problem. We focus on the contributions most relevant to our setting and refer to \textcite{Mandi.2024} for more context.

Surrogate decision losses avoid differentiating the exact decision map by replacing the downstream loss with a tractable proxy. Examples include the SPO+ loss \parencite{Elmachtoub.2022}, and perturbation-based losses \parencite{Gupta.2024} for predict-then-optimize problems, as well as Fenchel--Young losses for structured prediction \parencite{Blondel.2020}. \textcite{Hoppe.2026} recently apply this idea to DRL by embedding combinatorial optimization in the policy. This is not applicable to HDPO, since pathwise gradients rely on differentiating the realized trajectory cost through the simulator, which includes the embedded optimization, and therefore cannot easily be replaced by a surrogate.

We therefore build on differentiable optimization, which propagates gradients through the solution map of a continuous constrained optimization problem. Differentiable optimization layers have been developed for QPs \parencite{Amos.2017}, cone programs \parencite{Agrawal.2019}, and LP relaxations \parencite{Wilder.2019}. \textcite{Donti.2021} enforce hard constraints on continuous decisions by combining equality completion with an inequality-correction procedure inside the network, and \textcite{Paulus.2021} embed integer linear programs as layers with surrogate gradients. Both target single-stage decision problems, whereas in our sequential setting the decision feeds back into the state, so gradients must additionally flow through the closed loop. Within reinforcement learning, differentiable optimization layers have mainly been used to enforce safety constraints, for example during exploration \parencite{Dalal.2018}. \textcite{Markgraf.2026} note that differentiable optimization can cause rank loss in the local Jacobian and show that shifting safety constraints from the policy to the environment can be beneficial. In our setting, hard constraints are structural instead of safety artifacts. Thus, differentiable optimization allows us to combine hard action constraints with a low-variance training signal from HDPO for complex sequential decision problems.

\subsection{Research Gap}
\label{sec:gap}
Due to the combinatorial action space and interdependent hard constraints arising in many sequential decision problems, existing methods primarily rely on MILP formulations, which become computationally prohibitive in high-dimensional settings. Scalable learning-based alternatives are not yet available, since they must simultaneously handle large action spaces, hard feasibility, and delayed cost feedback. We address this methodological gap by proposing a differentiable policy that enforces discrete feasibility through a QP projection module and an integer mapping while enabling efficient end-to-end learning via pathwise gradients. We apply the method to production--inventory planning in capacitated multi-echelon networks, where these challenges are particularly pronounced due to tightly coupled resource constraints and long decision horizons.

\section{Preliminaries}
\label{sec:background}

\subsection{Problem Setting}
\label{sec:problem_generic}

We consider a finite-horizon, constrained Markov decision process (MDP) over periods $t \in \{1, \ldots, T\}$ with a combinatorial action space and state-dependent feasibility constraints. At the start of period $t$, the decision maker observes a state $\mathbf s_t \in \mathcal S$ and selects a feasible action $\mathbf x_t \in \mathcal A(\mathbf s_t) \subset \mathbb N_0^{n}$, where $n$ denotes the action dimension. We focus on the more challenging case of discrete actions, which are relevant in most OR applications due to the presence of indivisible units, but our method also applies to problems with continuous actions. The feasible action space $\mathcal A(\mathbf s_t)$ enforces action integrality and linear constraints that we group into two blocks. The first block contains $m_A$ constraints with fixed coefficient matrix $\mathbf A \in \mathbb R^{m_A \times n}$ and a state-dependent right-hand side $\mathbf b_t = \mathbf b(\mathbf s_t) \in \mathbb R^{m_A}$, while the second block contains $m_C$ constraints with fixed coefficient matrix $\mathbf C \in \mathbb R^{m_C \times n}$ and fixed right-hand side $\mathbf k \in \mathbb R^{m_C}$. This enables us to account for both dynamic constraints that depend on past decisions, such as material or workforce availability, and static constraints that are given, such as capacity or budget limits. The feasible action space is then given by
\begin{equation}
    \mathcal A(\mathbf s_t) := {\big\{ \mathbf x \in \mathbb N_0^{n} : \mathbf A \mathbf x \le \mathbf b_t,\; \mathbf C \mathbf x \le \mathbf k \big\}},
    \label{eq:feasible-set}
\end{equation}
whereas its continuous relaxation $\mathcal A^{\mathrm{rel}}(\mathbf s_t)$ is defined by replacing $\mathbb N_0^{n}$ with $\mathbb R_{\ge 0}^{n}$ in \eqref{eq:feasible-set}.

After choosing the action $\mathbf x_t$, an exogenous random vector $\mathbf \Xi_t$ is realized with sample value $\boldsymbol \xi_t$. Based on that, an immediate cost $c(\mathbf s_t, \mathbf x_t, \boldsymbol \xi_t)$ is incurred, and the state transitions to $\mathbf s_{t+1} = \Gamma(\mathbf s_t, \mathbf x_t, \boldsymbol \xi_t) \in \mathcal S$.

\paragraph{Assumptions.}
We make four assumptions on the MDP. (i) $\mathcal A^{\mathrm{rel}}(\mathbf s_t)$ is downward closed for all states $\mathbf s_t$: if $\mathbf x$ is relaxed feasible and $\mathbf 0 \le \mathbf x' \le \mathbf x$ componentwise, then so is $\mathbf x'$. Nonnegativity of $\mathbf A$ and $\mathbf C$ is a sufficient condition, and holds whenever constraints represent the consumption of limited resources, e.g., capacity, material, budget, and workforce limits. Since $\mathcal A(\mathbf s_t)$ is the intersection of its relaxation with $\mathbb N_0^{n}$, the integer feasible set inherits downward closedness, and the componentwise floor of every relaxed feasible point is an element of $\mathcal A(\mathbf s_t)$. (ii) The feasible set $\mathcal A(\mathbf s_t)$ is nonempty for every reachable state. Nonnegativity of $\mathbf b(\mathbf s_t)$ and $\mathbf k$ is a sufficient condition, and under (i) it is also necessary, since downward closure then makes the zero action feasible. (iii) For fixed uncertainty realizations $\boldsymbol \xi_t$, the per-period cost function $c$ and the transition function $\Gamma$ are deterministic and differentiable almost everywhere in the state and action $(\mathbf s_t, \mathbf x_t)$, so that pathwise gradients can propagate through the unrolled state recursion. (iv) The stochastic process $\mathbf \Xi_{1:T}$ is exogenous: its law may depend on exogenous information included in the state, such as observed covariates, but not on the actions or the endogenous part of the state. Assumption (i) is needed only for the proposed integer mapping, and ensures that a local search suffices to restore integrality. Assumption (ii) is required for the projection to be well-defined. Assumptions (iii) and (iv) represent the standard HDPO assumptions \parencite{Alvo.2023}.

\paragraph{Objective.}
A deterministic policy is a map $\pi : \mathbf s \mapsto \mathbf x \in \mathcal A(\mathbf s)$ that assigns a feasible action to every state. {While we focus on stationary policies, which are more relevant in practice, nonstationary policies can be represented by adding the time index $t$ to the state \parencite[cf.][]{Alvo.2023}}. We distinguish between the random process $\mathbf \Xi_{1:T}$ and a realized sample path $\boldsymbol \xi_{1:T}$. Conditional on an initial state $\mathbf s_1$ and a sample path realization $\boldsymbol \xi_{1:T}$, the system evolves deterministically, and the sample-path average cost of $\pi$ is
\[
    V(\mathbf s_1;\pi,\boldsymbol \xi_{1:T}) := \frac{1}{T} \sum_{t=1}^T c(\mathbf s_t,\pi(\mathbf s_t),\boldsymbol \xi_t),
    \quad \text{s.t.} \quad
    \mathbf s_{t+1} = \Gamma(\mathbf s_t, \pi(\mathbf s_t) , \boldsymbol \xi_t) \quad \forall t \in \{1, \ldots, T-1\}.
\]
We seek an optimal policy $\pi^*$ that minimizes the expected sample-path average cost, where the expectation is taken over the initial state and the exogenous stochastic process:
\[
    \pi^* \in \arg\min_{\pi} \ \mathbb{E}_{\mathbf{S}_1, \mathbf{\Xi}_{1:T}}\!\left[V(\mathbf{S}_1;\pi,\mathbf{\Xi}_{1:T})\right].
\]

\subsection{Hindsight Differentiable Policy Optimization}
\label{sec:hdpo}

We adopt the HDPO framework of \textcite{Madeka.2022}, \textcite{Bottcher.2023} and \textcite{Alvo.2023} to {optimize the policy}. We consider a parameterized differentiable policy $\pi_{\weights}$ (e.g., a neural network) with parameter vector $\weights \in \mathbb{R}^{n_{\weights}}$, and define the parameter optimization problem as
\[
    \weights^* \in \arg\min_{\weights} J(\weights) \quad \text{  with  } \quad J(\weights) := \mathbb{E}_{\mathbf{S}_1, \mathbf{\Xi}_{1:T}}\!\left[V(\mathbf{S}_1; \pi_{\weights}, \mathbf{\Xi}_{1:T})\right].
\]
Parameter updates use stochastic gradient descent (SGD) with learning rate $\eta > 0$,
\[
    \weights' = \weights - \eta \hat \nabla J(\weights),
\]
or a more advanced gradient-based stochastic optimizer such as Adam \parencite{Adam}. To estimate $\nabla J(\weights)$, we use pathwise gradients obtained by automatic differentiation through a differentiable simulator. Unlike score-function estimators such as REINFORCE \parencite{Williams.1992}, pathwise gradients propagate sensitivities directly along the trajectory. This typically yields lower-variance gradient estimates under delayed costs and permits a deterministic training policy, avoiding the training--evaluation mismatch caused by stochastic exploration. In each iteration, we sample a batch of {$N \in \mathbb{N}$} trajectories under the current policy $\pi_{\weights}$, roll them out in the differentiable simulator, and backpropagate through each sample path. This yields the gradient estimator
\[
    \hat \nabla J(\weights) = \frac{1}{N} \sum_{l=1}^N {\nabla_{\weights}} V\left(\mathbf s_1^{(l)};\pi_{\weights},\boldsymbol \xi_{1:T}^{(l)}\right).
\]
Under mild regularity conditions, this estimator is unbiased for the true gradient for a fully differentiable policy \parencite{Suh.2022}. HDPO alternates between simulating rollouts and updating the policy parameters.

\section{Constrained Hindsight Differentiable Policy Optimization}
\label{sec:method}

{We develop a policy that makes HDPO compatible with the hard constraints in \eqref{eq:feasible-set} while preserving differentiability. Starting from a potentially infeasible target action, a differentiable projection followed by an integer mapping {ensures} discrete feasibility while allowing gradients to flow end to end (Section~\ref{sec:architecture}). We then provide theoretical guarantees on the reachability of feasible actions and the selection of frontier actions that distinguish our integer mapping from naive flooring (Section~\ref{sec:guarantees}). {All proofs are provided in Appendix \ref{app:proofs}.}}

\subsection{Integer-Feasible Policy Architecture}
\label{sec:architecture}

The central challenge is that the action must satisfy hard interdependent constraints and ultimately be integer, while discrete optimization itself does not provide useful local Jacobians for HDPO. We therefore augment the neural network, which carries the learnable parameters $\weights$ and maps the state $\mathbf s_t$ to a continuous target $\mathbf z_t \in \mathbb{R}_{\ge 0}^{n}$, with two sequential modules: the differentiable projection enforces feasibility while relaxing integrality, and the integer mapping converts the result into an implementable discrete action $\mathbf x_t \in \mathcal{A}(\mathbf s_t)$. Figure~\ref{fig:full_pipeline} illustrates the resulting forward and backward passes within HDPO. {We introduce the two modules in turn, deriving the analytical gradients of the projection (Section~\ref{sec:projection}) and the straight-through gradients of the integer mapping (Section~\ref{sec:integer_mapping}).}

\input{figures/pipeline}

\subsubsection{QP Projection and Sensitivities}
\label{sec:projection}

While $\mathbf z_t$ represents the targeted action, it need not satisfy the constraints in~\eqref{eq:feasible-set}. We therefore project $\mathbf z_t$ to a feasible continuous action $\mathbf x^*_t$ by minimizing the squared weighted Euclidean distance, using weights $w_i > 0$ with $\sum_{i\in[n]} w_i = 1$ collected in the diagonal matrix $\mathbf W := \operatorname{diag}(w_i)_{i\in[n]}$. We use $[l]$ to denote the index set $\{1,\dots,l\}$ for a positive integer $l$. Larger weights penalize deviations in coordinate $i$ more strongly, so that a priori differences in scale or importance across coordinates can be encoded. The weights are exogenous to the policy and held fixed during training. 

We restate the relaxed action space of Section~\ref{sec:problem_generic} in terms of the stacked constraint matrix $\mathbf M \in \mathbb R^{m \times n}$ with $m := m_A + m_C + n$, and the stacked right-hand side $\mathbf r_t \in \mathbb R^{m}$ as
\[
\mathcal A^{\mathrm{rel}}(\mathbf s_t) =
\{\mathbf x\in\mathbb R^{n}: \mathbf M\mathbf x\le\mathbf r_t\},
\qquad
\mathbf M :=
{\renewcommand{\arraystretch}{0.6}
\begin{bmatrix}
\mathbf A\\
\mathbf C\\
-\mathbf I
\end{bmatrix}},
\qquad
\mathbf r_t :=
{\renewcommand{\arraystretch}{0.6}
\begin{bmatrix}
\mathbf b_t\\
\mathbf k\\
\mathbf 0
\end{bmatrix}},
\]
where $\mathbf M \mathbf x \le \mathbf r_t$ summarizes the state-dependent, static, and nonnegativity constraints. The projected action $\mathbf x^*_t$ is then given as the point in $\mathcal A^{\mathrm{rel}}(\mathbf s_t)$ closest to the target $\mathbf z_t$ in terms of the squared weighted Euclidean norm $\| \cdot \|_{\mathbf W}^2$. This projection $P_t$ can be computed by solving the following QP:
\begin{equation}
    \mathbf x_t^* = P_t(\mathbf z_t) :=
    \arg\min_{\mathbf x \in \mathcal A^{\mathrm{rel}}(\mathbf s_t)} f_t(\mathbf x;\mathbf z_t) \text{ with } f_t(\mathbf x;\mathbf z_t) := \frac12 \sum_{i \in [n]} w_i (x_{i} - z_{i,t})^2 = \frac12\|\mathbf x-\mathbf z_t\|_{\mathbf W}^2.
    \label{eq:qp_objective}
\end{equation}
Because $\mathbf W$ is positive definite, the objective is strongly convex \parencite{Boyd.2004}. Together with nonemptiness of $\mathcal A^{\mathrm{rel}}(\mathbf s_t)$ by Assumption (ii), this makes the projected action unique. The optimization problem can be solved with a primal-dual QP solver or a specialized algorithm for Euclidean projection onto convex sets, which returns both $\mathbf x_t^*$ and the optimal dual prices $\boldsymbol{\lambda}_t^*$ associated with the constraints $\mathbf M \mathbf x \le \mathbf r_t$.

The separability of the QP gives rise to a decomposition central to the scalability of our approach. The constraint matrix $\mathbf M$ induces a bipartite variable--constraint graph whose connected components partition the variables into independent clusters that share no constraint. Because the objective \eqref{eq:qp_objective} is separable across variables, the projection splits into one independent QP per cluster, each of which can be solved in parallel, so the per-step projection cost is governed by the size of the largest cluster rather than the global action dimension $n$. In many applications, these clusters are typically small, since constraint couplings are local, so the projection stays cheap even on large-scale instances, as demonstrated in {\ref{app:runtime}.}

For HDPO, however, we need more than the projected action itself: we also need the local Jacobians of the projection map with respect to its inputs. \textcite{Amos.2017} show that the QP map is differentiable almost everywhere and subdifferentiable everywhere. {Lemma~\ref{lem:qp_jacobians} gives closed-form expressions for the two sensitivities for a locally fixed active set based on the general expressions in \textcite{Amos.2017}.} If the active constraints are linearly dependent, the target-channel Jacobian remains well defined because it does not depend on which linearly independent subset is selected. For the right-hand-side channel, on the measure-zero set where the active set is not locally fixed, the solution map is only directionally differentiable, and we use the generalized subgradient. The sensitivities with respect to the constraint coefficients can be derived analogously if those coefficients also depend on the state. {Corollary~\ref{cor:competition} gives structural interpretations of the Jacobians: it {classifies} the response to target perturbations into three regimes, establishes that the target-channel Jacobian is nonexpansive, and describes the response to a relaxation of the state-dependent constraints.}

{Figure~\ref{fig:QP_intuition} illustrates the structural content of the two Jacobians. The Jacobian $\partial \mathbf{x}_t^*/\partial \mathbf{z}_t$ describes how infinitesimal target perturbations propagate along the active face. Corollary~\ref{cor:competition}(a) refines this effect coordinatewise: if coordinate $i$ is unconstrained at the optimum (\emph{Free}), the projection acts as the identity; if coordinate $i$ is fully determined by the active set (\emph{Pinned}), the projected action does not respond to its target; otherwise (\emph{Competing}), coordinate $i$ shares an active constraint with another coordinate, so increasing $z_{i,t}$ pushes at least one coupled coordinate downward. The Jacobian $\partial \mathbf{x}_t^*/\partial \mathbf{b}_t$ describes how relaxations of active right-hand sides enter the solution. Corollary~\ref{cor:competition}(c) shows that such relaxations are absorbed one-for-one by the active constraints, with the induced change in $\mathbf{x}_t^*$ being the minimum-$\mathbf W$-norm adjustment that achieves this absorption.}

\begin{lemma}[QP Jacobians]
\label{lem:qp_jacobians}
Let $\mathbf x_t^*$ be the unique minimizer of \eqref{eq:qp_objective}, and let $\tilde{\mathbf{M}}_t$ and $\tilde{\mathbf{r}}_t$ denote the active-constraint matrix and active right-hand side at $\mathbf{x}_t^*$, with $\mathbf{U}_t := \frac{\partial \tilde{\mathbf{r}}_t}{\partial \mathbf{b}_t}$. Suppose that the active constraints are linearly independent and the active set remains unchanged in a neighborhood of the solution. Then the projection map is locally differentiable in $(\mathbf z_t, \mathbf b_t)$, with Jacobians
\[
\frac{\partial \mathbf{x}_t^*}{\partial \mathbf{z}_t}
=
\mathbf{I}
-
\mathbf{W}^{-1}\tilde{\mathbf{M}}_t^\top
\left(\tilde{\mathbf{M}}_t\mathbf{W}^{-1}\tilde{\mathbf{M}}_t^\top\right)^{-1}
\tilde{\mathbf{M}}_t,
\quad
\frac{\partial \mathbf{x}_t^*}{\partial \mathbf{b}_t}
=
\mathbf{W}^{-1}\tilde{\mathbf{M}}_t^\top
\left(\tilde{\mathbf{M}}_t\mathbf{W}^{-1}\tilde{\mathbf{M}}_t^\top\right)^{-1}
\mathbf{U}_t.
\]
\end{lemma}

\begin{corollary}[Structure of the QP Jacobians]
\label{cor:competition}
Adopt the assumptions of Lemma~\ref{lem:qp_jacobians}, and in addition suppose {$\mathbf A, \mathbf C \ge \mathbf 0$}. For any matrix $\mathbf H$ write $[\mathbf H]_{:,i}$ for its $i$th column and write $\mathbf e_i$ for the $i$th unit vector.

\emph{(a) Target.} Each coordinate $i \in [n]$ with $x_{i,t}^* > 0$ falls into one of three regimes:
\begin{itemize}
\item \emph{Free.} If $[\tilde{\mathbf{M}}_t]_{:,i} = \mathbf{0}$, then $\frac{\partial \mathbf{x}_t^*}{\partial z_{i,t}} = \mathbf{e}_i$ and thus $\frac{\partial x_{i,t}^*}{\partial z_{i,t}}=1$.
\item Otherwise it is either:
      \emph{Pinned}, $\frac{\partial \mathbf{x}_t^*}{\partial z_{i,t}} = \mathbf{0}$; or
      \emph{Competing}, $0<\frac{\partial x_{i,t}^*}{\partial z_{i,t}}<1$ and $\frac{\partial x_{j,t}^*}{\partial z_{i,t}} < 0$ for some $j \ne i$.
\end{itemize}

{\emph{(b) Nonexpansiveness.} For every target perturbation $\Delta \mathbf z_t$, $\bigl\|\frac{\partial \mathbf{x}_t^*}{\partial \mathbf{z}_t}\,\Delta \mathbf z_t\bigr\|_{\mathbf W} \le \|\Delta \mathbf z_t\|_{\mathbf W}$.}

\emph{(c) Constraint.} $\tilde{\mathbf{M}}_t\, \frac{\partial \mathbf{x}_t^*}{\partial \mathbf{b}_t} = \mathbf{U}_t$, and for every perturbation $\Delta \mathbf b_t$,
\(
    \frac{\partial \mathbf{x}_t^*}{\partial \mathbf{b}_t} \Delta \mathbf b_t
    =
    \arg\min_{\mathbf y}
    \left\{
    \frac12\|\mathbf y\|_{\mathbf W}^2:
    \tilde{\mathbf M}_t\mathbf y
    =
    \mathbf U_t\Delta \mathbf b_t
    \right\}.
\)
\end{corollary}

\begin{figure}[t]
\centering

\begin{subfigure}[t]{0.45\textwidth}
\centering
\begin{tikzpicture}[scale=0.65, >=stealth]
    \node at (2,4.15) {\small Perturbing $\mathbf{z}$};

    \def\I{3.2}

    \fill[feasible] (0,0) -- (\I,0) -- (0,\I) -- cycle;

    \draw[->] (0,0) -- (4.1,0) node[right] {$x_1$};
    \draw[->] (0,0) -- (0,4.1) node[above] {$x_2$};

    \draw[neutral constraint] (\I,0) -- (0,\I);

    \node[below] at (\I,0) {$b$};
    \node[left]  at (0,\I) {$b$};

    \coordinate (z)  at (3.35,2.45);
    \coordinate (zp) at (3.95,2.45);

    \coordinate (x)  at (2.05,1.15);
    \coordinate (xp) at (2.35,0.85);

    \draw[projection guide] (z) -- (x);
    \draw[projection guide, draw=black!50] (zp) -- (xp);

    \draw[perturbation zarrow] (z) -- (zp);
    \draw[perturbation xarrow] (x) -- (xp);

    \node[target action, label=above:{$\mathbf{z}$}] at (z) {};
    \node[perturbed target, label=right:{$\mathbf{z}+\Delta\mathbf{z}$}] at (zp) {};

    \node[projected action, label=below left:{$\mathbf{x}^*$}] at (x) {};
    \node[perturbed projected, label=right:{$\mathbf{x}^*+\Delta\mathbf{x}$}] at (xp) {};
\end{tikzpicture}
\caption{Differentiating w.r.t. the target point.}
\end{subfigure}
\hfill
\begin{subfigure}[t]{0.45\textwidth}
\centering
\begin{tikzpicture}[scale=0.65, >=stealth]
    \node at (2,4.15) {\small Perturbing $b$};

    \def\Iold{2.5}
    \def\Inew{3.45}

    \fill[feasible region] (0,0) -- (\Iold,0) -- (0,\Iold) -- cycle;
    \fill[perturbed feasible region] (0,\Iold) -- (0,\Inew) -- (\Inew,0) -- (\Iold,0) -- cycle;

    \draw[->] (0,0) -- (4.1,0) node[right] {$x_1$};
    \draw[->] (0,0) -- (0,4.1) node[above] {$x_2$};

    \draw[neutral constraint] (\Iold,0) -- (0,\Iold);
    \draw[perturbed constraint] (\Inew,0) -- (0,\Inew);

    \node[below] at (\Iold,0) {$b$};
    \node[left]  at (0,\Iold) {$b$};

    \node[below] at (\Inew+0.4,0) {$b+\Delta b$};
    \node[left]  at (0,\Inew) {$b+\Delta b$};

    \def\zx{3.35}
    \def\zy{2.25}
    \coordinate (zR) at (\zx,\zy);

    \coordinate (x)  at ({(\zx - \zy + \Iold)/2},{(\zy - \zx + \Iold)/2});
    \coordinate (xp) at ({(\zx - \zy + \Inew)/2},{(\zy - \zx + \Inew)/2});

    \draw[projection guide] (zR) -- (x);
    \draw[projection guide, draw=black!50] (zR) -- (xp);

    \draw[perturbation xarrow] (x) -- (xp);

    \node[target action, label=above:{$\mathbf{z}$}] at (zR) {};

    \node[projected action, label=left:{$\mathbf{x}^*$}] at (x) {};
    \node[perturbed projected, label=right:{$\mathbf{x}^*+\Delta\mathbf{x}$}] at (xp) {};
\end{tikzpicture}
\caption{Differentiating w.r.t. the right-hand side.}
\end{subfigure}
\caption{Intuition for differentiating through a Euclidean projection with a single active constraint. }
\label{fig:QP_intuition}
\end{figure}

{As a map, the projection $P_t$ is firmly nonexpansive and hence $1$-Lipschitz in $\|\cdot\|_{\mathbf W}$, a classical property of metric projections onto convex sets that requires no regularity assumption \parencite[e.g.,][]{Bauschke.2017}. Corollary~\ref{cor:competition}(b) is the local counterpart for the backward pass: the target-channel Jacobian is an orthogonal projector and therefore nonexpansive. As the backward pass through the differentiable simulator multiplies the per-step Jacobians according to the chain rule, the projection neither amplifies target-side perturbations nor inflates the Lipschitz constant of the decision rule, so the target channel does not cause gradient explosion, which is a potential failure mode of HDPO \parencite[see][]{Suh.2022}.}

Corollary~\ref{cor:competition}(b) also quantifies the rank loss that \textcite{Markgraf.2026} identify in differentiable optimization layers: $\partial \mathbf x_t^*/\partial \mathbf z_t$ has rank $n - \operatorname{rank}(\tilde{\mathbf M}_t)$, so each linearly independent active constraint removes one direction from the target channel. However, in our case, the signal is routed rather than lost. By Lemma~\ref{lem:qp_jacobians}, $\partial \mathbf x_t^*/\partial \mathbf b_t$ has rank equal to the number of active state-dependent constraints, so each direction lost to such a constraint reappears in the right-hand-side channel, where it reaches the earlier decisions that made the resource scarce. Active static constraints lose rank without compensation, which is natural, since a binding fixed capacity cannot be relaxed by any decision.

In summary, the projection is unique, differentiable almost everywhere, and has stable, interpretable sensitivities with respect to the action targets. By itself, however, it does not guarantee integrality.

\subsubsection{Integer Mapping}
\label{sec:integer_mapping}

To obtain an integer action $\mathbf x_t \in \mathcal A(\mathbf s_t)$ close to $\mathbf z_t$ under the weighted Euclidean metric \eqref{eq:qp_objective} without solving a separate integer program, we use a local dual-informed mapping around $\mathbf{x}_t^*$. Denote this mapping by $\mathrm{IM}_t$, so the implemented action is $\mathbf x_t=\mathrm{IM}_t(\mathbf x_t^*,\mathbf z_t,\boldsymbol{\lambda}_t^*(\mathbf z_t))$. The mapping proceeds in three steps. First, initialize at the componentwise floor $\mathbf x_t^F:=\lfloor \mathbf{x}_t^* \rfloor$, which is integer and feasible because the relaxed feasible set is downward closed by Assumption (i). Second, compute a priority score for each coordinate,
\[
    \sigma_{i,t}
    =
    f_t(\mathbf x_t^F;\mathbf z_t)
    -
    f_t(\mathbf x_t^F + \mathbf e_i;\mathbf z_t)
    -
    [\mathbf M^\top \boldsymbol{\lambda}_t^*]_i.
\]
The first two terms capture the exact one-step change in the projection objective at the floor point, and the last term penalizes the use of implied constraints via the dual prices. Third, sort coordinates by decreasing $\sigma_{i,t}$ and traverse the list once: at each coordinate, accept the unit increment only if the resulting action remains feasible and does not increase \(f_t(\cdot;\mathbf z_t)\). This greedy local search augments the conservative floor solution with feasibility-preserving upward moves that are most promising under the local dual information.

It remains to specify how gradients pass through the integer mapping. Since the {integer mapping} is piecewise constant in $\mathbf{x}_t^*$, its exact Jacobian is zero almost everywhere and undefined at the integer boundaries, providing no useful learning signal. We therefore replace its Jacobian with a straight-through estimator (STE) \parencite[cf.][]{Bottcher.2023} that treats the integer-mapping step as the identity during backpropagation, unaffected by the target $\mathbf z_t$ and the dual prices $\boldsymbol{\lambda}_t^*$:
\[
    \widehat{\frac{\partial \mathbf x_t}{\partial \mathbf x_t^*}} := \mathbf I, \quad \widehat{\frac{\partial \mathbf x_t}{\partial \mathbf z_t}} := \mathbf 0, \quad \widehat{\frac{\partial \mathbf x_t}{\partial \boldsymbol{\lambda}_t^*}} := \mathbf 0.
\]
This surrogate acts as a linear interpolation of the underlying step function in $\mathbf{x}_t^*$ and is analogous to differentiating the continuous relaxation of a discrete optimization problem \parencite[cf.][]{Wilder.2019}. It trades biased gradients in the backward pass for forward-pass fidelity, since integer actions are used throughout training, so the visited states coincide with those at evaluation. {An alternative approach would be to train in a completely continuous surrogate system and apply our integer mapping only during evaluation. In the ablation study in Section~\ref{sec:ablation}, we show that both approaches can yield state-of-the-art performance.}

\subsection{Theoretical Guarantees}
\label{sec:guarantees}

We provide three structural guarantees for the proposed policy architecture: a projection error bound, a completeness property, and a frontier selection result. First, we bound the error introduced by the two-step mapping from $\mathbf z_t$ to $\mathbf x_t$ via $\mathbf x^*_t$ with respect to the optimal integer solution $\mathbf x^I_t$ to the discrete projection problem. Proposition~\ref{prop:hybrid_rounding_gap} bounds both the projection-metric gap to the optimal integer solution and the corresponding decision distance. The bound is expressed in absolute terms for near-feasible $\mathbf z_t$ to avoid distortions resulting from tiny denominators and otherwise in relative terms. Absolute bounds for far-from-feasible $\mathbf z_t$ are uninformative since the slope of the quadratic projection distance becomes arbitrarily steep. This shows that the two-step mapping adds limited error with respect to the projection objective and the actual decision. This justifies using the tractable two-step mapping for the nondifferentiable discrete projection problem.

\begin{proposition}[Bounded integer mapping error]
\label{prop:hybrid_rounding_gap}
Fix a state \(\mathbf s_t\) and target \(\mathbf z_t\), and write \(f_t(\mathbf x):=f_t(\mathbf x;\mathbf z_t)\). Let \(\mathbf x_t^*\in\arg\min_{\mathbf x\in\mathcal A^{\mathrm{rel}}(\mathbf s_t)} f_t(\mathbf x)\)
be the relaxed QP solution from \eqref{eq:qp_objective}, and let
\(\mathbf x_t^I\in\arg\min_{\mathbf x\in\mathcal A(\mathbf s_t)} f_t(\mathbf x)\)
be an optimal integer feasible solution. Define \(R_t:=\|\mathbf z_t-\mathbf x_t^*\|_{\mathbf W}\). Suppose that
\(\mathbf x_t\in\mathcal A(\mathbf s_t)\) is obtained from \(\mathbf x_t^*\)
via \(\mathrm{IM}_t\). Then, for any \(\rho>0\):
\begin{itemize}
\item[(i)] \emph{Objective gap.}
\[
R_t\le \rho
\ \Rightarrow\
f_t(\mathbf x_t)-f_t(\mathbf x_t^I)\le \tfrac12+\rho,
\ \
R_t>\rho
\ \Rightarrow\
\frac{f_t(\mathbf x_t)-f_t(\mathbf x_t^I)}{f_t(\mathbf x_t^I)}
\le
\frac{1}{\rho^2}+\frac{2}{\rho}.
\]
\item[(ii)] \emph{Decision distance.}
\[
R_t\le \rho
\ \Rightarrow\
\|\mathbf x_t-\mathbf x_t^I\|_{\mathbf W}
\le 1+\sqrt{1+2\rho},
\ \
R_t>\rho
\ \Rightarrow\
\frac{\|\mathbf x_t-\mathbf x_t^I\|_{\mathbf W}}{R_t}
\le
\frac{1}{\rho}+\sqrt{\frac{1}{\rho^2}+\frac{2}{\rho}}.
\]
\end{itemize}
\end{proposition}

Second, we show that the integer mapping, in contrast to naive flooring, is complete on the entire action set by analyzing the reachability for the differentiable projection and integer mapping modules. We write
\[
    \Phi_t^{\mathrm{IM}}(\mathbf z_t)
    := \mathrm{IM}_t\!\bigl(P_t(\mathbf z_t),\,\mathbf z_t,\,\boldsymbol{\lambda}_t^*(\mathbf z_t)\bigr),
    \qquad
    \Phi_t^{\mathrm{F}}(\mathbf z_t)
    := \lfloor P_t(\mathbf z_t)\rfloor,
\]
for the QP--integer-mapping and QP--floor compositions, evaluated at a target $\mathbf z_t$. We call the minimal expressiveness requirement on such maps \emph{completeness} and define it as follows.

\begin{definition}[Completeness]
\label{def:completeness}
Fix a state $\mathbf s_t \in \mathcal{S}$ and let $\mathcal Q\subseteq\mathcal A(\mathbf s_t)$. A measurable map $\mathcal{R}\colon\mathbb R_{\ge 0}^{n}\to\mathcal A(\mathbf s_t)$, as a function of the target $\mathbf z_t$, is \emph{complete on $\mathcal Q$} if the target preimage $\mathcal{R}^{-1}(\mathbf q) := \{\mathbf z_t\in\mathbb R_{\ge 0}^{n} : \mathcal{R}(\mathbf z_t)=\mathbf q\}$ has positive Lebesgue measure for every $\mathbf q\in\mathcal Q$. It is \emph{incomplete on $\mathcal Q$} at $\mathbf s_t$ if $\mathcal{R}^{-1}(\mathbf q)$ has Lebesgue measure zero for some $\mathbf q\in\mathcal Q$.
\end{definition}

Completeness requires that every feasible integer action in $\mathcal Q$ can be selected from a non-negligible set of continuous targets. This matters for learning because if an action has only a measure-zero preimage, it is practically impossible for a smoothly parameterized decision rule to select it. To locate where the two maps differ, we partition the feasible actions by constraint activity. We write
\[
    \mathcal A^{\mathrm{act}}(\mathbf s_t)
    :=
    \bigl\{\mathbf q\in\mathcal A(\mathbf s_t) :
    {[\mathbf A\mathbf q]_l} = b_{l,t} \text{ for some } l\in[m_A],
    \text{ or } {[\mathbf C\mathbf q]_j} = k_j \text{ for some } j\in[m_C]\bigr\}
\]
for the active boundary, i.e., the set of integer actions at which some constraint is active, and $\mathcal A^{\mathrm{slack}}(\mathbf s_t) := \mathcal A(\mathbf s_t)\setminus\mathcal A^{\mathrm{act}}(\mathbf s_t)$ for those at which none is. Proposition \ref{prop:completeness} contrasts the two compositions: our integer mapping is complete everywhere, whereas naive flooring is complete only away from the active boundary. This distinction is essential since the active boundary contains the actions that fully utilize some resource, which are of particular practical importance. Figure~\ref{fig:completeness} illustrates the difference in a two-dimensional single-constraint example. Since completeness of our integer mapping holds for any positive $\mathbf{W}$, the weighting rescales the target space rather than restricting the reachable action set. The incompleteness of flooring holds for every positive $\mathbf{W}$ as well, so it is structural and cannot be remedied by reweighting the projection.

\begin{proposition}[Completeness of the integer mapping and flooring]
\label{prop:completeness}
\leavevmode
\begin{itemize}
    \item[(a)] For all constraint structures $(\mathbf A, \mathbf b_t, \mathbf C, \mathbf k)$ satisfying Assumptions {(i)--(ii)}, positive weights $\mathbf{W}$, and every state $\mathbf s_t \in \mathcal{S}$, $\Phi_t^{\mathrm{IM}}$ is complete on $\mathcal A(\mathbf s_t)$ and $\Phi_t^{\mathrm{F}}$ is complete on $\mathcal A^{\mathrm{slack}}(\mathbf s_t)$.
    \item[(b)] There is a constraint structure $(\mathbf A, \mathbf b_t, \mathbf C, \mathbf k)$ satisfying Assumptions {(i)--(ii)}, and a state $\mathbf s_t \in \mathcal{S}$ for which $\Phi_t^{\mathrm{F}}$ is not complete on $\mathcal A^{\mathrm{act}}(\mathbf s_t)$ for every set of positive weights $\mathbf{W}$.
\end{itemize}
\end{proposition}

\begin{figure}[t]
\centering

\colorlet{ptColor00}{black!70}
\colorlet{ptColor01}{red!70!black}
\colorlet{ptColor10}{orange!85!black}
\colorlet{ptColor02}{teal!70!black}
\colorlet{ptColor11}{green!50!black}
\colorlet{ptColor20}{blue!65!black}

\tikzset{
    fill00/.style={fill=ptColor00!18, draw=none},
    fill01/.style={fill=ptColor01!16, draw=none},
    fill10/.style={fill=ptColor10!18, draw=none},
    fill02/.style={fill=ptColor02!18, draw=none},
    fill11/.style={fill=ptColor11!18, draw=none},
    fill20/.style={fill=ptColor20!18, draw=none},
    label00/.style={text=ptColor00},
    label01/.style={text=ptColor01},
    label10/.style={text=ptColor10},
    label02/.style={text=ptColor02},
    label11/.style={text=ptColor11},
    label20/.style={text=ptColor20}
}

\newcommand{\FillForZeroZero}{\fill[fill00] (-0.3,-0.3) rectangle (3.0,3.0);}
\newcommand{\FillForZeroOne}{\fill[fill01] (-0.3,-0.3) rectangle (3.0,3.0);}
\newcommand{\FillForOneZero}{\fill[fill10] (-0.3,-0.3) rectangle (3.0,3.0);}
\newcommand{\FillForZeroTwo}{\fill[fill02] (-0.3,-0.3) rectangle (3.0,3.0);}
\newcommand{\FillForOneOne}{\fill[fill11] (-0.3,-0.3) rectangle (3.0,3.0);}
\newcommand{\FillForTwoZero}{\fill[fill20] (-0.3,-0.3) rectangle (3.0,3.0);}
\newcommand{\DefineReachabilityPoints}{
    \coordinate (zerozero) at (0,0);
    \coordinate (zerohalf) at (0,0.5);
    \coordinate (zeroone) at (0,1);
    \coordinate (zeroonehalf) at (0,1.5);
    \coordinate (zerotwo) at (0,2);
    \coordinate (zerotop) at (0,2.55);

    \coordinate (halfzero) at (0.5,0);
    \coordinate (halfhalf) at (0.5,0.5);
    \coordinate (halfone) at (0.5,1);
    \coordinate (halfonehalf) at (0.5,1.5);

    \coordinate (onezero) at (1,0);
    \coordinate (onehalf) at (1,0.5);
    \coordinate (oneone) at (1,1);
    \coordinate (onetop) at (1,2.55);

    \coordinate (onehalfzero) at (1.5,0);
    \coordinate (onehalfhalf) at (1.5,0.5);
    \coordinate (onefiftyfivetop) at (1.55,2.55);

    \coordinate (conetop) at (0.55,2.55);
    \coordinate (coneright) at (2.55,0.55);

    \coordinate (twozero) at (2,0);
    \coordinate (twoone) at (2,1);

    \coordinate (rightzero) at (2.55,0);
    \coordinate (rightone) at (2.55,1);
    \coordinate (rightonefiftyfive) at (2.55,1.55);
    \coordinate (topright) at (2.55,2.55);
}

\begin{subfigure}[t]{0.47\textwidth}
\centering
\begin{tikzpicture}[scale=1.25, >=stealth, font=\scriptsize]
    \useasboundingbox (-0.32,-0.40) rectangle (3.10,2.82);
    \DefineReachabilityPoints

    \fill[black!7] (zerozero) -- (twozero) -- (zerotwo) -- cycle;

    \begin{scope}
        \clip (zerozero) rectangle (oneone);
        \FillForZeroZero
    \end{scope}
    \begin{scope}
        \clip (zeroone) rectangle (onetop);
        \FillForZeroOne
    \end{scope}
    \begin{scope}
        \clip (onezero) rectangle (rightone);
        \FillForOneZero
    \end{scope}
    \begin{scope}
        \clip (oneone) -- (onetop) -- (topright) -- cycle;
        \FillForZeroOne
    \end{scope}
    \begin{scope}
        \clip (oneone) -- (rightone) -- (topright) -- cycle;
        \FillForOneZero
    \end{scope}

    \begin{scope}
        \clip (zerotwo) -- (zerotop) -- (conetop) -- cycle;
        \FillForZeroTwo
    \end{scope}
    \begin{scope}
        \clip (twozero) -- (coneright) -- (rightzero) -- cycle;
        \FillForTwoZero
    \end{scope}

    \foreach \x in {0,1,2}{
        \draw[black!20,densely dotted] (\x,0) -- (\x,2.55);
    }
    \foreach \y in {0,1,2}{
        \draw[black!20,densely dotted] (0,\y) -- (2.55,\y);
    }
    \draw[-{Latex[length=1.8mm,width=1.2mm]}, line width=0.8pt]
        (-0.08,0) -- (2.65,0) node[right] {$x_1$};
    \draw[-{Latex[length=1.8mm,width=1.2mm]}, line width=0.8pt]
        (0,-0.08) -- (0,2.65) node[above] {$x_2$};

    \draw[black!55,line width=1.15pt] (zerotwo) -- (twozero);
    \draw[ptColor11!50,line width=1.2pt] (oneone) -- (topright);

    \foreach \pt in {zerozero,zeroone,zerotwo,onezero,oneone,twozero}{
        \fill[black] (\pt) circle (1.0pt);
    }
    \node[label00, anchor=north east, xshift=-2pt, yshift=-2pt] at (zerozero) {$(0,0)$};
    \node[label02, anchor=east, xshift=-3pt, yshift=3pt] at (zerotwo) {$(0,2)$};
    \node[label20, anchor=north, xshift=3pt, yshift=-2pt] at (twozero) {$(2,0)$};

    \node[label01, anchor=east, xshift=-4pt] at (zeroone) {$(0,1)$};
    \node[label10, anchor=north, yshift=-4pt] at (onezero) {$(1,0)$};
    \node[label11, anchor=west, xshift=2pt, yshift=0pt] at (oneone) {$(1,1)$};
\end{tikzpicture}
\caption{Floor rounding creates positive-area target regions for every feasible action except the active \((1,1)\), which it reaches only from the measure-zero target ray (green).}
\end{subfigure}
\hfill
\begin{subfigure}[t]{0.5\textwidth}
\centering
\begin{tikzpicture}[scale=1.25, >=stealth, font=\scriptsize]
    \useasboundingbox (-0.32,-0.40) rectangle (3.10,2.82);
    \DefineReachabilityPoints

    \fill[black!7] (zerozero) -- (twozero) -- (zerotwo) -- cycle;

    \begin{scope}
        \clip (zerozero) rectangle (halfhalf);
        \FillForZeroZero
    \end{scope}
    \begin{scope}
        \clip (zerohalf) rectangle (halfonehalf);
        \FillForZeroOne
    \end{scope}
    \begin{scope}
        \clip (halfzero) rectangle (onehalfhalf);
        \FillForOneZero
    \end{scope}
    \begin{scope}
        \clip (halfonehalf) -- (zeroonehalf) -- (zerotop) -- (onefiftyfivetop) -- cycle;
        \FillForZeroTwo
    \end{scope}
    \begin{scope}
        \clip (halfhalf) -- (onehalfhalf) -- (rightonefiftyfive) -- (topright)
            -- (onefiftyfivetop) -- (halfonehalf) -- cycle;
        \FillForOneOne
    \end{scope}
    \begin{scope}
        \clip (onehalfzero) -- (rightzero) -- (rightonefiftyfive)
            -- (onehalfhalf) -- cycle;
        \FillForTwoZero
    \end{scope}

    \foreach \x in {0,1,2}{
        \draw[black!20,densely dotted] (\x,0) -- (\x,2.55);
    }
    \foreach \y in {0,1,2}{
        \draw[black!20,densely dotted] (0,\y) -- (2.55,\y);
    }
    \draw[-{Latex[length=1.8mm,width=1.2mm]}, line width=0.8pt]
        (-0.08,0) -- (2.65,0) node[right] {$x_1$};
    \draw[-{Latex[length=1.8mm,width=1.2mm]}, line width=0.8pt]
        (0,-0.08) -- (0,2.65) node[above] {$x_2$};

    \draw[black!55,line width=1.15pt] (zerotwo) -- (twozero);

    \foreach \pt in {zerozero,zeroone,zerotwo,onezero,oneone,twozero}{
        \fill[black] (\pt) circle (1.0pt);
    }
    \node[label00, anchor=north east, xshift=-2pt, yshift=-2pt] at (zerozero) {$(0,0)$};
    \node[label02, anchor=east, xshift=-3pt, yshift=3pt] at (zerotwo) {$(0,2)$};
    \node[label20, anchor=north, xshift=3pt, yshift=-2pt] at (twozero) {$(2,0)$};

    \node[label01, anchor=east, xshift=-4pt] at (zeroone) {$(0,1)$};
    \node[label10, anchor=north, yshift=-4pt] at (onezero) {$(1,0)$};
    \node[label11, anchor=west, xshift=2pt] at (oneone) {$(1,1)$};
\end{tikzpicture}
\caption{Our integer mapping creates positive-area target regions for all integer feasible points, making them reachable from positive-measure target sets.}
\end{subfigure}

\caption{Preimages of the integer feasible points under floor rounding (left) and our integer mapping (right) for $x_1 + x_2 \le 2$ with $x_1, x_2 \in \mathbb{N}_0$.}
\label{fig:completeness}
\end{figure}

Third, we show that the integer mapping selects an action on the frontier of the feasible set for ambitious targets. While completeness is an expressiveness property, it does not describe which action is selected for a given target. Proposition~\ref{prop:frontier_reachability} shows that when every coordinate of the projection residual $\mathbf z_t - \mathbf x_t^*$ exceeds half a unit, the integer mapping returns an action on the integer frontier. The integer frontier is the set of feasible actions for which some unit increment is infeasible. This is desirable in practice, because leaving resources unused is wasteful if action targets are ambitious. When the QP can be decomposed, the result naturally applies to each cluster separately. 

\begin{proposition}[Frontier selection]
\label{prop:frontier_reachability}
Define the integer frontier of the feasible set as
\[
    \mathcal F(\mathbf s_t)
    := \bigl\{\mathbf y \in \mathcal A(\mathbf s_t)
       :\; \exists\, i \in [n] \text{ such that }
       \mathbf y + \mathbf e_i \notin \mathcal A(\mathbf s_t)\bigr\}.
\]
If the projection residual satisfies $z_{i,t}-x_{i,t}^* > \frac12$ for every $i \in [n]$, then \(\Phi_t^{\mathrm{IM}}(\mathbf z_t) \in \mathcal F(\mathbf s_t)\).
\end{proposition}

In summary, the three results in this section show that the policy's discrete output is well behaved. The two-step mapping adds only a bounded rounding error. Furthermore, unlike naive flooring, which remains incomplete on the active boundary under every choice of weights for certain instances, our integer mapping can always reach the boundary actions that matter most in tightly constrained systems and, when the target is ambitious enough, pushes decisions onto the integer frontier.

\section{Production--Inventory Planning in General Networks}
\label{sec:application}

We describe the application of our method to a centralized periodic-review production--inventory planning problem on a general multi-echelon network. The network consists of items $\mathcal{I} = \mathcal{I}^C \cup \mathcal{I}^E$, partitioned into component items $\mathcal{I}^C$ that feed downstream production stages and end items $\mathcal{I}^E$ that face external customer demand ($\mathcal{I}^C \cap \mathcal{I}^E = \emptyset$). Items are linked by bill-of-material (BOM) relationships encoded in the matrix $\mathbf{A} \in \mathbb{N}_0^{|\mathcal{I}^C| \times |\mathcal{I}|}$, where $A_{i,j}$ denotes the number of units of component item $i \in \mathcal{I}^C$ required to produce one unit of item $j \in \mathcal{I}$. We assume the BOM graph is acyclic. Production is subject to deterministic item-specific lead times $\tau_i \in \mathbb{N}$ and shared static resource constraints. A set of resources $\mathcal{G}$ is available with $k_g \in \mathbb{N}$ units per period, and the consumption matrix $\mathbf{C} \in \mathbb{N}_0^{|\mathcal{G}| \times |\mathcal{I}|}$ specifies the units $C_{g,i}$ of resource $g \in \mathcal{G}$ consumed per unit of item $i \in \mathcal{I}$. This formulation allows for general convergent--divergent multi-echelon structures with shared resources and contains serial, assembly, and distribution networks as special cases.

External demand $\Xi_{i,t} \in \mathbb{N}_0$ arises only for end items $i \in \mathcal{I}^E$ and is modeled as an exogenous stochastic process that may be non-stationary over time, with realization $\xi_{i,t}$ observed in period $t$. Excess end-item demand is backordered. Demand for components is endogenous through the BOM. An exogenous observable feature vector $\mathbf{s}_t^{\text{feat}}$ available at the start of period $t$ summarizes demand-relevant covariates to forecast future demand. Costs are incurred for holding inventory at a per-period unit cost {$h_i \in \mathbb{R}_{\ge 0}$} for each item $i \in \mathcal{I}$ and for backordered end-item demand at a per-period unit penalty {$p_i \in \mathbb{R}_{\ge 0}$} for each end item $i \in \mathcal{I}^E$. We do not model setup times or setup costs. At the tactical level considered here, the review period is long relative to changeover frequency, so each item is produced in nearly every period and setup time is approximately constant across periods. The setup time can therefore be subtracted from the available capacity $\mathbf{k}$ rather than modeled through endogenous setup decisions. To the extent that setup costs represent the opportunity cost of the capacity consumed by changeovers \parencite{Karmarkar.1987}, they are captured by the same reduction. The tactical decision is thus how much to produce rather than whether to set up.

At the start of period $t$, the state $\mathbf{s}_t$ contains on-hand component inventories $\mathbf o_t^C$, end-item inventory levels, outstanding pipeline orders, and demand-relevant features. The action $\mathbf{x}_t \in \mathbb{N}_0^{|\mathcal{I}|}$ specifies the integer production quantities released for all items in period $t$. Production is constrained by material availability and shared resources, encoded respectively as $\mathbf A\mathbf x_t \le \mathbf o_t^C$ and $\mathbf C\mathbf x_t \le \mathbf k$. The transition $\Gamma$ and cost $c$ are differentiable almost everywhere, and demand is exogenous, satisfying Assumptions~(iii)--(iv). This instantiates the generic MDP of Section~\ref{sec:problem_generic} with action dimension $n = |\mathcal{I}|$, $m_A = |\mathcal{I}^C|$ state-dependent constraints with right-hand side $\mathbf b_t := \mathbf o_t^C$, and $m_C = |\mathcal{G}|$ state-independent constraints. {Since material and resource consumption as well as on-hand inventories are nonnegative, {$\mathbf A \ge \mathbf 0, \mathbf C \ge \mathbf 0, \mathbf b_t \ge \mathbf 0, \mathbf k \ge \mathbf 0$} hold in every period, satisfying Assumptions~{(i)--(ii)}.} The full MDP of the problem is described in Appendix~\ref{app:mdp}. Following \textcite{Xie.2026}, we introduce an inductive bias into the policy architecture by considering a differentiable decision rule that outputs a state-dependent {echelon base-stock level}, from which the action targets are derived using the order-up-to rule. Appendix~\ref{app:bs_target} gives the full construction and Section \ref{sec:ablation} provides the numerical justification for this choice.
\section{Computational Study}
\label{sec:experiments}

We evaluate the proposed policy across three sets of experiments: small instances with a known optimal policy in Section~\ref{sec:toy}, larger capacitated networks based on \textcite{Tempelmeier.1996} in Section~\ref{sec:tempelmeier}, and the industry-scale 2025 ASML Research Challenge in Section~\ref{sec:asml}. An ablation study isolates the contribution of each policy module in Section~\ref{sec:ablation}. Since existing HDPO approaches are not directly applicable under hard, interdependent action constraints (see Section \ref{sec:literature}), it also serves as a comparison against design choices made in earlier work. Runtimes of our policy compared to the baselines are reported in {\ref{app:runtime}}.

\paragraph{Our Policy} We use a fully connected neural network with two hidden layers and continuously differentiable exponential linear unit (CELU) activations \parencite[cf.][]{Bottcher.2023}. The width is $32$ for the small instances and $128$ for the larger instances, including the case study. This is the only hyperparameter that we vary across instances to match network expressiveness with the instance complexity. Parameter optimization uses Adam \parencite{Adam} with a base learning rate of $\eta = 10^{-3}$, momentum {parameters $(\beta_1, \beta_2) = (0.9, 0.999)$} and a batch size of $1{,}024$. The prioritization weights $w_i$ in the QP are set inversely proportional to each item's long-run average gross {requirement $\bar d_i$, which equals the mean external demand for end items $i \in \mathcal I^E$ and is given recursively by $\bar d_i := \sum_{j \in \mathcal I} A_{i,j}\,\bar d_j$ for component items $i \in \mathcal I^C$,} and normalized to sum to one. This prioritizes items with lower demand to balance the structural disadvantage of low-volume items in distance-based allocation. As shown in Section \ref{sec:ablation}, {the choice of weights is not decisive for policy performance.} Further implementation details of our policy are listed in {\ref{app:policy_implementation}}.

\subsection{Baseline Policies} 
We provide an overview of the baseline policies with details deferred to {\ref{app:benchmarks}}. {The baselines stem from \textcite{vanDijck.2026}. We use the code from their official repository, along with their hyperparameter settings for the ASML case study}. The hyperparameters for the other instances are given in {\ref{app:baseline_hyperparameters}}.

\paragraph{Base-Stock Policies} We consider different variants of {structured echelon base-stock policies}. The policy raises the echelon inventory position {$\mathrm{EIP}_{i,t}$} of each item $i$ up to a potentially time-varying base-stock level $B_{i,t}$ and allocates production resources to minimize the maximum relative shortfall. Details on the allocation policy are provided in {\ref{app:benchmarks}}. Static base-stock levels ($\mathrm{SBS}$) use the long-run mean $m'_i$ and standard deviation $s'_i$ of echelon lead-time demand, $B_{i,t} = \lceil m'_i + \theta_i\, s'_i \rceil$, and apply under i.i.d.\ demand. Dynamic forecast-based levels ($\mathrm{DBS}$) use the forecasted conditional mean $m'_{i,t}$ and standard deviation $s'_{i,t}$ given $\mathbf{s}_t^{\text{feat}}$, $B_{i,t} = \lceil m'_{i,t} + \theta_i\, s'_{i,t} \rceil$, and apply under nonstationary demand. {$\mathrm{DBS}$ is only used in the case study, since it is equivalent to $\mathrm{SBS}$ for stationary demand}. {\ref{app:benchmarks}} lays out how the echelon lead-time demand moments are computed. The item-specific safety-stock factors $\boldsymbol\theta := (\theta_i)_{i \in \mathcal I}$ are optimized using the black-box simulation-optimization algorithm CMA-ES \parencite{Hansen.2023}. The single-factor variants $\overline{\mathrm{SBS}}$ and $\overline{\mathrm{DBS}}$ replace $\theta_i$ by a global $\theta$ optimized by 1-D simulation-based grid search. In the small instances of Section~\ref{sec:toy}, we additionally report $\mathrm{SBS}^*$, the globally optimal SBS policy obtained via a simulation-based exhaustive enumeration.

\paragraph{MILP Policies} We additionally include a multi-stage stochastic program ($\mathrm{MSP}$) and its deterministic counterpart ($\mathrm{DET}$), both deployed in a rolling horizon. At each period, $\mathrm{MSP}$ builds a scenario tree by sampling demand forward from the current state {over a lookahead horizon of $H$ periods beyond the current one}, and implements only the first-stage production decision before re-solving from the updated state. The tree lets production at later nodes adapt to the realized demand history (recourse), while nonanticipativity constraints prevent it from reacting to demand that has not yet been observed. It structurally matches the approach of \textcite{Thevenin.2021} and \textcite{Fleuren.2025}. $\mathrm{DET}$ is the special case with a single successor per node and demand fixed to its inflated conditional-mean forecast $(1+\gamma) \cdot \expect{\boldsymbol \Xi_{t:t+H} \mid \mathbf s^{\text{feat}}_t}$ with demand-inflation parameter $\gamma \ge 0$. With a single successor per node, $\mathrm{DET}$ has no recourse. We include it since it is close to what many companies still use for production planning, although we expect it to be outperformed by policies that explicitly account for uncertainty. The full formulation is given in {\ref{app:mss}}.

To improve the computational efficiency of $\mathrm{MSP}$, we solve the LP relaxation and use a feasible rounding procedure to obtain integer values for the current decision. The rounding is structurally identical to our integer mapping, but it sorts the coordinates by the fractional part of the LP solution, since evaluating its effect on the objective would require re-solving the LP. \textcite{vanDijck.2026} found this to provide statistically indistinguishable results from the {MILP} policy, while requiring 8 times less computation time.

\subsection{Small Instances}
\label{sec:toy}

We investigate the smallest supply network with interdependent action constraints to compare against the optimal policy. The network has $|\mathcal{I}|=3$ items, with $\mathcal{I}^C=\{1\}$ being the shared component and $\mathcal{I}^E=\{2, 3\}$ the end items fed by a unit BOM $\mathbf A = \begin{bsmallmatrix}0 & 1 & 1\end{bsmallmatrix}$. Following the terminology of \textcite{Tempelmeier.1996}, we distinguish between a \emph{noncyclic} resource structure in which every resource is confined to a single BOM level, and a \emph{cyclic} one in which at least one resource is shared by an item and one of its (indirect) components. Specifically, we consider two resource structures: noncyclic dedicated resources with $\mathbf C = \mathbf I$ and capacities $(k_1, k_2, k_3) = (3, 1, 2)$, and cyclic resources with $\mathbf C = \begin{bsmallmatrix}1 & 1 & 0\\0 & 0 & 1\end{bsmallmatrix}$ and capacities $(k_1, k_2) = (4, 2)$. In the noncyclic setting, the projection decomposes into $\mathcal{I}^C$ and $\mathcal{I}^E$, whereas the cyclic setting is not decomposable. All lead times are fixed to $\tau_i=1$.

We generate a set of instances by varying three parameters: the average resource utilization $u \in \{0.5, 0.8\}$, the variance-to-mean ratio $v \in \{1, 2\}$, and the shortage-cost ratio $r \in \{0.8, 0.9, 0.95\}$. Demand is {i.i.d.\ across} periods and independent across end items, with {an empty feature vector $\mathbf{s}_t^{\text{feat}}$}. {{The mean end-item demand is set such that the average gross requirements satisfy $\sum_{i \in \mathcal I} C_{g,i}\,\bar d_i = u\, k_g$} for every resource $g \in \mathcal{G}$, i.e., the expected per-period consumption of resource $g$ equals a fraction $u$ of its capacity. We use $u$ to vary the mean demand instead of the capacities {so that $k_g$ remain integers}. The demand distribution is Poisson when $v = 1$ and negative binomial with $\operatorname{Var}(\Xi_i) = v\,\expect{\Xi_i}$ when $v > 1$. Instead of the coefficient of variation, we use the variance-to-mean ratio $v$ because it directly controls overdispersion for discrete demand: when the mean changes, fixing $v$ keeps the same distribution family and only requires changing its parameters. We fix holding costs to $(h_1, h_2, h_3) = (1, 6, 4)$ across instances and determine shortage costs such that $r = p_i / (p_i + h_i)$ holds.} This yields a compact testbed of 24 instances across different levels of congestion, demand variability, and shortage severity.

We compare our method and $\mathrm{SBS}^*$ against the optimal average cost obtained by {modified policy iteration initialized by $\mathrm{SBS}^*$} on a truncated state space. The state-space truncation bounds are increased until further enlargement leaves the optimal policy unchanged. The policies are evaluated {exactly using the Bellman expectation equations}. Table \ref{tab:toy_results} shows the gap of the average cost of our method and {$\mathrm{SBS}^*$} relative to the optimal average cost across the 24 instances. Our method achieves optimality gaps below $2\%$ with only {one} exception and an average gap of {$1.0\%$} across all instances. We outperform {$\mathrm{SBS}^*$} on 20/24 instances{, and its average gap of {$2.0\%$} is twice as large as ours}. {The optimality gap of $\mathrm{SBS}^*$ is largest in cyclic systems, at high utilization, high demand variability, and low shortage-cost ratios. These are also the instances where our method's advantage is largest, reaching $4.7$ percentage points. $\mathrm{SBS}^*$ outperforms our method in only four instances, all with $v=1$ and $r \ge 0.9$, and three out of four with $u=0.5$.} Since $\mathrm{SBS}^*$ is the globally optimal policy within its class, {the gap of $\mathrm{SBS}^*$} reflects the structural limitation of fixed-allocation base-stock policies under interdependent action constraints. {In contrast,} our method learns near-optimal policies across different levels of congestion, demand variability, and shortage severity.

\begin{table}[t]
  \centering
  \caption{Optimality gaps (\%) for Ours and $\mathrm{SBS}^*$ on the small instances. The smaller gap is bold.}
  \label{tab:toy_results}
  \renewcommand{\arraystretch}{1.0}
  \setlength{\tabcolsep}{4pt}
  \footnotesize
  \begin{tabular}{llccc|ccc|ccc|ccc}
    \toprule
    & & \multicolumn{6}{c|}{$u=0.5$} & \multicolumn{6}{c}{$u=0.8$} \\
    \cmidrule(lr){3-8} \cmidrule(lr){9-14}
    & & \multicolumn{3}{c|}{$v=1$} & \multicolumn{3}{c|}{$v=2$} & \multicolumn{3}{c|}{$v=1$} & \multicolumn{3}{c}{$v=2$} \\
    \cmidrule(lr){3-5} \cmidrule(lr){6-8} \cmidrule(lr){9-11} \cmidrule(lr){12-14}
    Structure & Policy $\qquad r=$ & 0.80 & 0.90 & 0.95 & 0.80 & 0.90 & 0.95 & 0.80 & 0.90 & 0.95 & 0.80 & 0.90 & 0.95 \\
    \midrule
    \multirow{2}{*}{Cyclic} & Ours & \textbf{1.00} & \textbf{0.48} & 0.98 & \textbf{0.61} & \textbf{0.60} & \textbf{0.67} & \textbf{1.53} & \textbf{1.13} & \textbf{1.23} & \textbf{0.44} & \textbf{0.99} & \textbf{1.23} \\
    & SBS$^*$ & 1.12 & 1.09 & \textbf{0.44} & 1.37 & 1.23 & 0.84 & 4.01 & 2.87 & 2.73 & 5.09 & 3.83 & 3.28 \\
    \cmidrule(lr){1-14}
    \multirow{2}{*}{\shortstack{Non-\\cyclic}} & Ours & \textbf{0.52} & 0.31 & 0.89 & \textbf{0.60} & \textbf{0.93} & \textbf{0.46} & \textbf{1.15} & \textbf{1.26} & 2.50 & \textbf{1.42} & \textbf{1.37} & \textbf{1.88} \\
    & SBS$^*$ & 0.58 & \textbf{0.28} & \textbf{0.71} & 1.22 & 1.46 & 0.76 & 2.73 & 2.06 & \textbf{1.76} & 3.55 & 2.43 & 1.99 \\
    \bottomrule
  \end{tabular}
\end{table}

\subsection{Larger Instances}
\label{sec:tempelmeier}

{We proceed to test our method on two sets of larger instances based on the general capacitated network structures from \textcite[][Problem Class A and B]{Tempelmeier.1996}. Class A consists of $|\mathcal{I}|=10$ items, with $|\mathcal{I}^E|=4$ end items, and $|\mathcal{G}|=3$ resource groups, whereas class B consists of $|\mathcal{I}|=40$ items, with $|\mathcal{I}^E|=6$ end items, and $|\mathcal{G}|=6$ resource groups. The networks are depicted in Figure \ref{fig:tempelmeier_10} (a) and Figure \ref{fig:tempelmeier_40} (a). Following \textcite{Tempelmeier.1996}, we consider 2 constraint structures with cyclic and noncyclic resource groups listed in Figure \ref{fig:tempelmeier_10} (b) and Figure \ref{fig:tempelmeier_40} (b), as well as 5 utilization profiles listed in Figure \ref{fig:tempelmeier_10} (c) and Figure \ref{fig:tempelmeier_40} (c), with $k_g$ rounded to integers. The BOM and resource consumption matrices are binary. For class A, the projection decomposes into the resource groups A, B, C in the noncyclic setting, whereas the cyclic setting does not decompose. For class B, the cyclic setting decomposes into A--B, C, D--E, F, and the noncyclic setting into A--B, C, D, E, F.} 

In addition, we assume a value-added structure for the holding costs, where the holding cost of item $i$ is given by $h_i = 1 + \sum_{j \in \mathcal{I}^C} A_{j,i} h_j$, so that the most upstream items have $h_i = 1$. Lead times are $\tau_i=1$ for all items. The shortage-cost ratio is set to $r=0.9$ for all end items. Demand is {i.i.d.\ across} periods and independent across end items, with {an empty feature vector $\mathbf{s}_t^{\text{feat}}$}. Raw demand is gamma distributed and stochastically rounded to integers using a Bernoulli distribution to preserve the expected demand, which is given by \(\expect{\boldsymbol \Xi} = [7, 3, 5, 10]^\top\) for class A {and \(\expect{\boldsymbol \Xi} = [4, 2, 3, 6, 2, 3]^\top\) for class B.} We consider 2 coefficients of variation $\mathrm{cov} \in \{0.5, 0.9\}$ for the raw gamma demand, leading to a testbed of 20 instances per problem class. The planning horizon is $T=60$ periods.

\begin{figure}[t]
    \centering
    \tikzset{
        netnode/.style={circle, draw, minimum size=5.2mm, inner sep=0pt, font=\scriptsize}
    }

    \begin{subfigure}[b]{0.36\textwidth}
        \centering
        \begin{tikzpicture}[>=stealth, font=\scriptsize]
            \def\xgap{1.30}
            \def\ygap{0.60}

            \newcommand{\comp}[3]{\node[netnode] (n#1) at ({#2*\xgap},{-#3*\ygap}) {#1};}
            \newcommand{\bom}[2]{\draw[->, line width=0.55pt, shorten >=1pt, shorten <=1pt] (n#1.east) -- (n#2.west);}

            \comp{8}{0}{1.5}
            \comp{9}{0}{2.5}
            \comp{10}{0}{3.5}

            \comp{5}{1}{1.5}
            \comp{6}{1}{2.5}
            \comp{7}{1}{3.5}

            \comp{1}{2}{1.0}
            \comp{2}{2}{2.0}
            \comp{3}{2}{3.0}
            \comp{4}{2}{4.0}

            \bom{8}{5} \bom{8}{6}
            \bom{9}{6} \bom{9}{7}
            \bom{10}{7}

            \bom{5}{1} \bom{5}{2}
            \bom{6}{2} \bom{6}{3}
            \bom{7}{3} \bom{7}{4}
        \end{tikzpicture}
        \caption{Supply network.}
        \label{fig:tempelmeier_10_network}
    \end{subfigure}\hfill
    \begin{subfigure}[b]{0.30\textwidth}
        \centering
        \renewcommand{\arraystretch}{1.05}
        \setlength{\tabcolsep}{5pt}
        \footnotesize
        \begin{tabular}{lll}
            \toprule
            Res. & Noncyclic & Cyclic \\
            \midrule
            A & 1--4  & 1, 2, 6 \\
            B & 5--7  & 3, 4, 7 \\
            C & 8--10 & 5, 8--10 \\
            \bottomrule
        \end{tabular}
        \caption{Resource groups.}
        \label{tab:tempelmeier_10_resources}
    \end{subfigure}\hfill
    \begin{subfigure}[b]{0.28\textwidth}
        \centering
        \renewcommand{\arraystretch}{1.05}
        \setlength{\tabcolsep}{5pt}
        \footnotesize
        \begin{tabular}{c|c|c|c}
            \toprule
            Profile & 0.9 & 0.7 & 0.5 \\
            \midrule
            1 & A--C & --   & --   \\
            2 & --   & A--C & --   \\
            3 & --   & --   & A--C \\
            4 & A    & B    & C    \\
            5 & C    & B    & A    \\
            \bottomrule
        \end{tabular}
        \caption{Utilization profiles.}
        \label{tab:tempelmeier_10_profiles}
    \end{subfigure}

    \caption{10-item resource and BOM structures from \textcite[][Problem Class A]{Tempelmeier.1996}.}
    \label{fig:tempelmeier_10}
\end{figure}

\begin{table}[t]
  \centering
  \caption{Average costs per period on the 10-item instances of problem class A. The lowest cost in each column is bold, as is any entry that is not significantly higher than it (paired t-test with $p<0.01$).}
\label{tab:tempelmeier_10_results}
  \renewcommand{\arraystretch}{1.0}
  \setlength{\tabcolsep}{4pt}
  \footnotesize
  \newcommand{\tempelmeieremptycells}{& & & & & & & & & &}
  \begin{tabular}{llcc|cc|cc|cc|cc}
    \toprule
    & & \multicolumn{2}{c|}{Profile 1} & \multicolumn{2}{c|}{Profile 2} & \multicolumn{2}{c|}{Profile 3} & \multicolumn{2}{c|}{Profile 4} & \multicolumn{2}{c}{Profile 5} \\
    \cmidrule(lr){3-4} \cmidrule(lr){5-6} \cmidrule(lr){7-8} \cmidrule(lr){9-10} \cmidrule(lr){11-12}
    Structure & Policy & {$\mathrm{cov}=0.5$} & 0.9 & 0.5 & 0.9 & 0.5 & 0.9 & 0.5 & 0.9 & 0.5 & 0.9 \\
    \midrule
    \multirow{6}{*}{Cyclic} & Ours & \textbf{231.84} & \textbf{503.26} & \textbf{191.98} & \textbf{380.91} & \textbf{190.79} & \textbf{367.15} & \textbf{202.39} & \textbf{432.54} & \textbf{199.17} & \textbf{418.91} \\
    & SBS & 245.81 & 558.49 & 192.62 & 389.93 & \textbf{190.64} & 369.69 & 212.03 & 477.97 & 202.59 & 427.31 \\
    & $\overline{\mathrm{SBS}}$ & 252.49 & 576.76 & 194.07 & 393.85 & 192.43 & 374.16 & 221.40 & 499.99 & 202.85 & 435.53 \\
    & MSP ($H{=}16$) & 245.71 & 541.60 & 208.32 & 410.91 & 207.25 & 399.57 & 218.47 & 466.51 & 216.27 & 451.22 \\
    & MSP ($H{=}8$) & 260.06 & 608.93 & 208.32 & 413.35 & 207.26 & 399.21 & 222.68 & 492.85 & 217.51 & 467.23 \\
    & DET ($H{=}8$) & 266.92 & 564.46 & 225.31 & 427.70 & 224.80 & 431.76 & 242.23 & 489.25 & 271.11 & 525.35 \\
    \cmidrule(lr){1-12}
    \multirow{6}{*}{\shortstack{Non-\\cyclic}} & Ours & \textbf{205.78} & \textbf{436.41} & \textbf{190.54} & \textbf{371.98} & \textbf{190.65} & \textbf{366.52} & \textbf{201.90} & \textbf{419.04} & \textbf{198.87} & \textbf{407.52} \\
    & SBS & 209.92 & 451.78 & 191.06 & 375.32 & \textbf{190.39} & 369.46 & 210.02 & 459.12 & 201.28 & 413.98 \\
    & $\overline{\mathrm{SBS}}$ & 211.67 & 460.89 & 193.03 & 380.56 & 192.43 & 373.27 & 211.57 & 467.18 & 201.20 & 426.12 \\
    & MSP ($H{=}16$) & 220.98 & 467.57 & 207.67 & 403.44 & 207.27 & 398.84 & 217.19 & 451.38 & 215.10 & 441.45 \\
    & MSP ($H{=}8$) & 224.87 & 498.95 & 207.68 & 403.43 & 207.27 & 398.53 & 220.31 & 472.58 & 216.22 & 453.29 \\
    & DET ($H{=}8$) & 242.97 & 510.73 & 224.25 & 426.06 & 224.79 & 430.65 & 229.47 & 489.63 & 268.09 & 533.19 \\
    \bottomrule
  \end{tabular}
\end{table}

{We compare our method against $\mathrm{SBS}$, $\overline{\mathrm{SBS}}$, and, for the smaller problem class A, also DET and MSP. Their lookahead horizon is twice the maximum echelon lead time ($H=8$), and for MSP, we additionally consider four times the maximum echelon lead time ($H=16$). For class B, both MILP policies are computationally prohibitive: as shown in \ref{app:runtime}, MSP would require up to 100 CPU days per instance for $1{,}000$ evaluation trajectories, and tuning the demand inflation factor of DET would likewise require up to four CPU weeks per instance.} All policies are evaluated on a common set of $1{,}000$ demand trajectories with identical initial states and demand realizations. The initial state distribution for evaluation is generated by a $30$-period warmup simulation under the optimized SBS policy. 

{The results for problem class A are given in Table \ref{tab:tempelmeier_10_results}. Our method outperforms all benchmarks in all but two instances, where the average costs of $\mathrm{SBS}$ are not statistically distinct. These are the ones with low resource utilization of 0.5 and low demand variance with $\mathrm{cov}=0.5$. Cost reductions of our method reach up to {$9.9\%$} relative to $\mathrm{SBS}$. $\mathrm{SBS}$ improves on $\overline{\mathrm{SBS}}$ by up to $4.4\%$, except in noncyclic Profile 5 with $\mathrm{cov}=0.5$, where both are tied. With $H=8$, MSP cannot outperform DET in cyclic high-utilization instances (profiles 1 and 4) with high demand variability ($\mathrm{cov}=0.9$), and is strictly worse than $\mathrm{SBS}$ across all instances. This changes for $H=16$, where MSP consistently outperforms DET and beats $\mathrm{SBS}$ in the regimes where it loses to DET with $H=8$. This indicates that a longer lookahead is required for MSP to reach competitive performance when capacity is tight, but does not matter when capacity is loose. Potentially, even longer lookahead horizons could further improve performance, but would be computationally prohibitive, since the runtime already increases up to sixfold when doubling the lookahead horizon (see \ref{app:runtime}). Since DET relies on a tuned demand inflation factor, it can compensate for the limited lookahead.}

In line with the findings from Section \ref{sec:toy}, the savings compared to SBS are particularly large when downstream capacity is tight, as is the case in Profiles 1 and 4, and demand is variable. {The tight upstream capacity in Profile 5 seems less challenging for SBS, since our method saves little on these instances.} The cost reduction of our method compared to $\mathrm{SBS}$ tends to be higher under the cyclic resource structure {(3.7\%) than} in the noncyclic setting (2.3\%). MSP's performance relative to our method is best where SBS's is poor, {i.e.,} under high resource utilization, and is insensitive to the resource structure.

{For problem class B, the results are given in Table \ref{tab:tempelmeier_40_results}. Our method outperforms $\mathrm{SBS}$ in 15 out of 20 instances, with relative savings ranging up to 15.3\%, which is larger than the largest savings from problem class A. Two instances are statistically tied, and $\mathrm{SBS}$ is better in three, all with low demand variability ($\mathrm{cov}=0.5$) and low downstream utilization of 0.5. Most findings from before carry over: the largest savings are again observed with high downstream utilization (profiles 1 and 4) and high demand variability. In contrast to problem class A, the average cost reduction is comparable for cyclic (3.2\%) and noncyclic (3.8\%) instances, which is consistent with the fact that both decompose similarly. As before, $\mathrm{SBS}$ consistently outperforms $\overline{\mathrm{SBS}}$ by up to 18.0\%. Item-specific optimization pays off especially when downstream capacity is tight, but its computational cost grows quickly with the number of items: tuning the 40 {safety-stock} factors of $\mathrm{SBS}$ takes twice as long as training our method, although it is faster for the 10 items of problem class A (see \ref{app:training_time})}.

\begin{figure}[t]
    \centering
    \tikzset{
        netnode/.style={circle, draw, minimum size=5.2mm, inner sep=0pt, font=\scriptsize}
    }

    \begin{subfigure}[b]{0.28\textwidth}
        \centering
        \begin{tikzpicture}[>=stealth, font=\scriptsize]
            \def\xgap{1.0}
            \def\ygap{0.52}

            \newcommand{\comp}[3]{\node[netnode] (n#1) at ({#2*\xgap},{-#3*\ygap}) {#1};}
            \newcommand{\bom}[2]{\draw[->, line width=0.55pt, shorten >=1pt, shorten <=1pt] (n#1.east) -- (n#2.west);}

            \comp{30}{0}{1.0} \comp{31}{0}{2.0} \comp{32}{0}{3.0}
            \comp{33}{0}{4.0} \comp{34}{0}{5.0} \comp{35}{0}{6.0}
            \comp{36}{0}{7.0} \comp{37}{0}{8.0} \comp{38}{0}{9.0}
            \comp{39}{0}{10.0} \comp{40}{0}{11.0}

            \comp{20}{1}{1.5} \comp{21}{1}{2.5} \comp{22}{1}{3.5}
            \comp{23}{1}{4.5} \comp{24}{1}{5.5} \comp{25}{1}{6.5}
            \comp{26}{1}{7.5} \comp{27}{1}{8.5} \comp{28}{1}{9.5}
            \comp{29}{1}{10.5}

            \comp{11}{2}{2.0} \comp{12}{2}{3.0} \comp{13}{2}{4.0}
            \comp{14}{2}{5.0} \comp{15}{2}{6.0} \comp{16}{2}{7.0}
            \comp{17}{2}{8.0} \comp{18}{2}{9.0} \comp{19}{2}{10.0}

            \comp{7}{3}{3.0} \comp{8}{3}{5.0}
            \comp{9}{3}{7.0} \comp{10}{3}{9.0}

            \comp{1}{4}{1.5} \comp{2}{4}{3.5} \comp{3}{4}{5.0}
            \comp{4}{4}{6.5} \comp{5}{4}{8.0} \comp{6}{4}{10.5}

            \bom{30}{20} \bom{30}{21}
            \bom{31}{21} \bom{31}{22}
            \bom{32}{22} \bom{32}{23}
            \bom{33}{23} \bom{33}{24}
            \bom{34}{24} \bom{34}{25}
            \bom{35}{25} \bom{35}{26}
            \bom{36}{26} \bom{36}{27}
            \bom{37}{27} \bom{37}{28}
            \bom{38}{28} \bom{38}{29}
            \bom{39}{29}
            \bom{40}{29}

            \bom{20}{11} \bom{20}{12}
            \bom{21}{12} \bom{21}{13}
            \bom{22}{13}
            \bom{23}{14}
            \bom{24}{14} \bom{24}{15}
            \bom{25}{15} \bom{25}{16}
            \bom{26}{16} \bom{26}{17}
            \bom{27}{17}
            \bom{28}{18}
            \bom{29}{18} \bom{29}{19}

            \bom{11}{7} \bom{12}{7} \bom{13}{7}
            \bom{14}{8} \bom{15}{8}
            \bom{16}{9} \bom{17}{9}
            \bom{18}{10} \bom{19}{10}

            \bom{7}{1} \bom{7}{2}
            \bom{8}{2} \bom{8}{3}
            \bom{9}{3} \bom{9}{4} \bom{9}{5}
            \bom{10}{5} \bom{10}{6}
        \end{tikzpicture}
        \caption{Supply network.}
        \label{fig:tempelmeier_40_network}
    \end{subfigure}\hfill
    \begin{subfigure}[b]{0.37\textwidth}
        \centering
        \renewcommand{\arraystretch}{1.05}
        \setlength{\tabcolsep}{5pt}
        \footnotesize
        \begin{tabular}{lll}
            \toprule
            Res. & Noncyclic & Cyclic \\
            \midrule
            A & 1--3   & 1--3, 5 \\
            B & 4--6   & 4, 6, 18--19 \\
            C & 7--10  & 7--10 \\
            D & 11--19 & 11--15, 25--29, 39--40 \\
            E & 20--29 & 16, 17, 20--24, 31 \\
            F & 30--40 & 30, 32--38 \\
            \bottomrule
        \end{tabular}
        \caption{Resource groups.}
        \label{tab:tempelmeier_40_resources}
    \end{subfigure}\hfill
    \begin{subfigure}[b]{0.27\textwidth}
        \centering
        \renewcommand{\arraystretch}{1.05}
        \setlength{\tabcolsep}{5pt}
        \footnotesize
        \begin{tabular}{c|c|c|c}
            \toprule
            Profile & 0.9  & 0.7  & 0.5  \\
            \midrule
            1 & A--F & --   & --   \\
            2 & --   & A--F & --   \\
            3 & --   & --   & A--F \\
            4 & A, B & C, D & E, F \\
            5 & E, F & C, D & A, B \\
            \bottomrule
        \end{tabular}
        \caption{Utilization profiles.}
        \label{tab:tempelmeier_40_profiles}
    \end{subfigure}

    \caption{40-item resource and BOM structures from \textcite[][Problem Class B]{Tempelmeier.1996}.}
    \label{fig:tempelmeier_40}
\end{figure}
\begin{table}[t]
  \centering
  \renewcommand{\arraystretch}{1.0}
  \setlength{\tabcolsep}{4pt}
  \footnotesize
  \caption{Average costs per period on the 40-item instances of problem class B. The lowest cost in each column is bold, as is any entry that is not significantly higher than it (paired t-test with $p<0.01$).}
  \label{tab:tempelmeier_40_results}
  \begin{tabular}{llcc|cc|cc|cc|cc}
    \toprule
    & & \multicolumn{2}{c|}{Profile 1} & \multicolumn{2}{c|}{Profile 2} & \multicolumn{2}{c|}{Profile 3} & \multicolumn{2}{c|}{Profile 4} & \multicolumn{2}{c}{Profile 5} \\
    \cmidrule(lr){3-4} \cmidrule(lr){5-6} \cmidrule(lr){7-8} \cmidrule(lr){9-10} \cmidrule(lr){11-12}
    Structure & Policy & {$\mathrm{cov}=0.5$} & 0.9 & 0.5 & 0.9 & 0.5 & 0.9 & 0.5 & 0.9 & 0.5 & 0.9 \\
    \midrule
    \multirow{3}{*}{Cyclic} & Ours & \textbf{1310.89} & \textbf{2810.46} & \textbf{1166.37} & \textbf{2229.94} & 1175.54 & \textbf{2193.56} & \textbf{1280.94} & \textbf{2693.82} & \textbf{1233.29} & \textbf{2532.89} \\
    & SBS & 1362.24 & 3058.71 & \textbf{1164.16} & 2258.87 & \textbf{1154.67} & \textbf{2193.07} & 1359.61 & 3063.31 & 1241.06 & 2602.20 \\
    & $\overline{\mathrm{SBS}}$ & 1545.66 & 3549.01 & 1296.48 & 2495.65 & 1295.41 & 2473.48 & 1576.65 & 3734.18 & 1333.17 & 2743.55 \\
    \cmidrule(lr){1-12}
    \multirow{3}{*}{\shortstack{Non-\\cyclic}} & Ours & \textbf{1379.62} & \textbf{3009.26} & \textbf{1170.05} & \textbf{2255.63} & 1164.10 & \textbf{2190.38} & \textbf{1364.38} & \textbf{2930.07} & 1223.27 & \textbf{2485.69} \\
    & SBS & 1437.52 & 3274.14 & 1173.18 & 2275.61 & \textbf{1155.67} & 2214.53 & 1500.51 & 3457.56 & \textbf{1211.27} & 2508.32 \\
    & $\overline{\mathrm{SBS}}$ & 1645.84 & 3799.48 & 1297.65 & 2534.62 & 1295.13 & 2477.86 & 1681.13 & 4058.66 & 1311.63 & 2675.94 \\
    \bottomrule
  \end{tabular}
\end{table}

\subsection{Case Study}
\label{sec:asml}

We consider the benchmark instance from the 2025 ASML Research Challenge, jointly organized by ASML and the International Society of Inventory Research \parencite{ASML_supply_chain_planning_challenge}, {including several variations.} The instance is designed to capture the operational characteristics of real-world supply chains in the high-tech industry. It consists of $|\mathcal{I}|=14$ items, including $|\mathcal{I}^E|=3$ end items, connected through a general BOM with $|\mathcal{G}|=9$ resource groups, as depicted in Figure \ref{fig:asml_network}. The resource structure is noncyclic, and the projection decomposes into the nine resource groups, the largest of which contains three items. Each component is required exactly once for every end item and each unit consumes one resource unit. The planning horizon is $T=60$ months and the average utilization of all resource groups is either $0.79$ or $0.83$. The exogenous demand process is nonstationary and mimics real-world patterns in the high-tech industry. The aggregate market demand follows a cyclical AR(8) process with long-run mean $\mu = 15$ and innovation variance $\sigma^2 = 4$, using autoregressive parameters \(\boldsymbol{\phi} = [0.6, 0.5, -0.2, -0.2, 0.2, -0.2, 0.1, -0.1]^\top\). The market demand is then split stochastically across the end products based on a Dirichlet distribution with parameters $\boldsymbol \alpha$ combined with an unbiased rounding procedure. Thus, demand is also correlated across end items. The expected demand shares for the end items are $\frac7{15}$, $\frac13$, and $\frac15$, respectively, and {$\boldsymbol \alpha = \left[\frac{35}{3}, \frac{25}{3}, 5\right]^\top$}. The feature vector $\mathbf s^{\text {feat}}_t$ contains the last eight aggregate demand observations before rounding. The shortage-cost ratio is $r=0.9$ for all end items. The remaining instance parameters and resource groups for the base case are given in Table~\ref{tab:case_study}. 

{We further consider three variations proposed by \textcite[][]{vanDijck.2026}: (i) less stochastic product splits with Dirichlet parameters $3 \boldsymbol \alpha = [35, 25, 15]^\top$ (setting $\alpha^+$), (ii) higher backorder costs with shortage-cost ratio $r=0.95$ for all end items (setting $p^+$), and (iii) higher resource utilization of $0.83$--$0.88$ with capacities decreased to $k_g - 1$ (setting $k^-$). We additionally compare our method to DBS on an instance with even higher resource utilization of $0.88$--$0.94$ with capacities $k_g - 2$ (setting $k^{--}$) to validate the above findings.}

\begin{table}[t]
  \centering
  \caption{Overview of instance parameters for the ASML Research Challenge from \textcite{ASML_supply_chain_planning_challenge}.}
  \label{tab:case_study}\renewcommand{\arraystretch}{0.75}
  \setlength{\tabcolsep}{6pt}
  \begin{tabular}{l|c|c|c|c|c|c|c|c|c|c|c|c|c|c}
  
    \toprule
    $i \in \mathcal{I}$ & \textbf{1} & \textbf{2} & \textbf{3} & \textbf{4} & \textbf{5} & \textbf{6} & \textbf{7} & \textbf{8} & \textbf{9} & \textbf{10} & \textbf{11} & \textbf{12} & \textbf{13} & \textbf{14} \\
    \midrule
    $\tau_i$      & 5 & 8 & 6 & 7 & 8 & 3 & 3 & 3 & 3 & 4 & 3 & 2 & 2 & 3 \\
    $h_i$         & 0.3 & 0.5 & 0.2 & 0.3 & 0.4 & 0.6 & 0.7 & 0.8 & 0.9 & 1.0 & 1.7 & 6.0 & 7.0 & 8.0 \\
    \midrule
    $g \in \mathcal{G}$ 
                   & A & B & C & D & E 
                   & \multicolumn{3}{c|}{F} 
                   & \multicolumn{2}{c|}{G} 
                   & H 
                   & \multicolumn{3}{c}{I} \\
    $k_g$          
                   & 18 & 19 & 18 & 19 & 19 
                   & \multicolumn{3}{c|}{18} 
                   & \multicolumn{2}{c|}{19} 
                   & 19 
                   & \multicolumn{3}{c}{19} \\
    \bottomrule
  \end{tabular}
\end{table}

\begin{figure}[t]
    \centering
    \tikzset{
        netnode/.style={circle, draw, minimum size=5mm, inner sep=0pt, font=\small},
        capbox/.style={}
    }

    \begin{tikzpicture}[>=stealth, font=\small]
        \def\upstreamgap{0.66}
        \def\samegap{0.55}
        \def\diffgap{0.55}

        \pgfmathsetmacro{\ylefttop}{2.00}
        \pgfmathsetmacro{\ymidtop}{2.00}
        \pgfmathsetmacro{\yrighttop}{1.00}

        \pgfmathsetmacro{\ycone}{\ylefttop}
        \pgfmathsetmacro{\yctwo}{\ycone-\upstreamgap}
        \pgfmathsetmacro{\ycthree}{\yctwo-\upstreamgap}
        \pgfmathsetmacro{\ycfour}{\ycthree-\upstreamgap}
        \pgfmathsetmacro{\ycfive}{\ycfour-\upstreamgap}

        \pgfmathsetmacro{\ycsix}{\ymidtop}
        \pgfmathsetmacro{\ycseven}{\ycsix-\samegap}
        \pgfmathsetmacro{\yceight}{\ycseven-\samegap}
        \pgfmathsetmacro{\ycnine}{\yceight-\diffgap}
        \pgfmathsetmacro{\ycten}{\ycnine-\samegap}
        \pgfmathsetmacro{\yceleven}{\ycten-\diffgap}

        \pgfmathsetmacro{\yctwelve}{\yrighttop}
        \pgfmathsetmacro{\ycthirteen}{\yctwelve-\samegap}
        \pgfmathsetmacro{\ycfourteen}{\ycthirteen-\samegap}

        \node[netnode] (c1) at (0.0,  \ycone) {1};
        \node[netnode] (c2) at (0.0,  \yctwo) {2};
        \node[netnode] (c3) at (0.0,  \ycthree) {3};
        \node[netnode] (c4) at (0.0,  \ycfour) {4};
        \node[netnode] (c5) at (0.0,  \ycfive) {5};

        \node[netnode] (c6)  at (3.45, \ycsix) {6};
        \node[netnode] (c7)  at (3.45, \ycseven) {7};
        \node[netnode] (c8)  at (3.45, \yceight) {8};
        \node[netnode] (c9)  at (3.45, \ycnine) {9};
        \node[netnode] (c10) at (3.45, \ycten) {10};
        \node[netnode] (c11) at (3.45, \yceleven) {11};

        \node[netnode] (c12) at (7.05, \yctwelve) {12};
        \node[netnode] (c13) at (7.05, \ycthirteen) {13};
        \node[netnode] (c14) at (7.05, \ycfourteen) {14};

        \draw[->] (c1.east) -- (c6.west) node[pos=0.78, above] {};
        \draw[->] (c1.east) -- (c7.west) node[pos=0.78, above] {};
        \draw[->] (c1.east) -- (c8.west) node[pos=0.78, above] {};

        \draw[->] (c2.east) -- (c9.west)  node[pos=0.78, above] {};
        \draw[->] (c2.east) -- (c10.west) node[pos=0.78, above] {};

        \draw[->] (c3.east) -- (c11.west) node[pos=0.78, above] {};
        \draw[->] (c4.east) -- (c11.west);
        \draw[->] (c5.east) -- (c11.west);

        \draw[->] (c6.east) -- (c12.west) node[pos=0.32, above] {};
        \draw[->] (c7.east) -- (c13.west) node[pos=0.32, above] {};
        \draw[->] (c8.east) -- (c14.west) node[pos=0.32, above] {};

        \draw[->] (c9.east)  -- (c12.west);
        \draw[->] (c9.east)  -- (c13.west);
        \draw[->] (c10.east) -- (c14.west);

        \draw[->] (c11.east) -- (c12.west);
        \draw[->] (c11.east) -- (c13.west);
        \draw[->] (c11.east) -- (c14.west);

        \begin{pgfonlayer}{background}
            \node[capbox, fit=(c1)] {};
            \node[capbox, fit=(c2)] {};
            \node[capbox, fit=(c3)] {};
            \node[capbox, fit=(c4)] {};
            \node[capbox, fit=(c5)] {};
            \node[capbox, fit=(c6)(c7)(c8)] {};
            \node[capbox, fit=(c9)(c10)] {};
            \node[capbox, fit=(c11)] {};
            \node[capbox, fit=(c12)(c13)(c14)] {};
        \end{pgfonlayer}
    \end{tikzpicture}

    \caption{14-item supply network structure of the ASML Research Challenge from \textcite{ASML_supply_chain_planning_challenge}.}
    \label{fig:asml_network}
\end{figure}

We compare our method to SBS, DBS, $\overline{\mathrm{SBS}}$, $\overline{\mathrm{DBS}}$, DET and MSP, both with $H=30$. All baselines use the hyperparameters from \textcite{vanDijck.2026}. $\overline{\mathrm{SBS}}$ is the original baseline of the ASML Research Challenge and DBS is the best-known policy for the instances. We evaluate all policies on a fixed set of $1{,}000$ demand trajectories per setting with identical initial states and demand realizations. The initial state distribution for evaluation is generated by a $30$-period warmup simulation under the optimized DBS policy. 

{Table \ref{tab:asml_results} shows that our method outperforms all baselines across the four settings. We outperform the original challenge baseline $\overline{\mathrm{SBS}}$ by $6.9$--$8.7\%$. Relative to DBS, which is consistently the best-performing baseline, our method reduces costs by $0.9 \pm 0.3\%$ in the base case (reported as relative gap $\pm$ 99\% paired confidence interval). Savings reach $1.8 \pm 0.6\%$ when capacities are tighter ($k^-$). This trend {continues} with even higher resource utilization ($k^{--}$), which leads to savings of $3.1 \pm 0.6\%$, confirming the finding that our method outperforms the base-stock policy the most when capacity is tight. Compared to MSP, our method reduces costs by between $2.4 \pm 0.6\%$ for $k^-$ and $5.5 \pm 1.0\%$ for $p^+$. As before, MSP is strongest when resources are highly utilized. While MSP may benefit from a larger scenario tree or a longer lookahead horizon, it already requires 11 CPU days per instance for the evaluation, four orders of magnitude more than our policy (see \ref{app:runtime}).}

{DBS finds a very different holding-backorder cost trade-off than MSP, stocking significantly more to avoid backorders. Our method strikes a middle ground, incurring lower holding costs than DBS and lower backorder costs than MSP, while outperforming both in terms of total costs. For lower demand variance ($\alpha^+$) and higher resource utilization ($k^-$), our method even saves on both cost components compared to MSP. With higher shortage penalty ($p^+$), the backorder costs of DBS and our method barely move, whereas those of MSP increase by 27\%, which explains the large cost gap.} 

Since ASML's newest lithography systems exceed US\$300 million in value per unit \parencite{vanDijck.2026}, the capital tied up in inventory is considerable. The $3.4$--$6.1$\% reduction in holding costs compared to DBS consequently represents economically significant savings. Further, both DBS and MSP require an explicit model of future demand to {make} decisions, whereas our policy directly maps the observed demand history to production decisions.

\begin{table}[t]
\centering
\caption{Average cost per period across the ASML instances. Hold.\ / Back.\ splits it into average holding and backorder costs per period. The lowest cost in each column is bold (paired t-test with $p<0.01$).}
\label{tab:asml_results}
\setlength{\tabcolsep}{5pt}
\begin{tabular}{lrrrrrrrr}
\toprule
 & \multicolumn{2}{c}{Base case}
& \multicolumn{2}{c}{$\alpha^+$}
& \multicolumn{2}{c}{$p^+$}
& \multicolumn{2}{c}{$k^-$} \\
\cmidrule(lr){2-3}
\cmidrule(lr){4-5}
\cmidrule(lr){6-7}
\cmidrule(lr){8-9}
Policy
& Cost & Hold.\ / Back.
& Cost & Hold.\ / Back.
& Cost & Hold.\ / Back.
& Cost & Hold.\ / Back. \\
\midrule
Ours
& \textbf{275.4} & 204.2 / 71.2
& \textbf{248.4} & 183.5 / 64.9
& \textbf{325.0} & 251.5 / 73.5
& \textbf{294.3} & 202.9 / 91.4 \\

DBS
& 278.0 & 211.4 / 66.6
& 252.3 & 191.3 / 61.0
& 329.2 & 261.5 / 67.7
& 299.6 & 216.1 / 83.5 \\

$\overline{\mathrm{DBS}}$
& 286.2 & 215.5 / 70.7
& 260.0 & 194.7 / 65.4
& 341.2 & 267.0 / 74.2
& 318.6 & 229.0 / 89.6 \\

SBS
& 296.9 & 231.3 / 65.5
& 270.2 & 210.0 / 60.2
& 347.8 & 280.5 / 67.3
& 313.9 & 236.9 / 77.0 \\

$\overline{\mathrm{SBS}}$
& 299.5 & 230.4 / 69.0
& 272.1 & 216.1 / 56.0
& 350.8 & 287.6 / 63.3
& 316.1 & 236.1 / 80.0 \\

MSP
& 286.6 & 202.6 / 84.0
& 258.0 & 187.9 / 70.1
& 343.9 & 236.9 / 107.0
& 301.5 & 207.8 / 93.7 \\

DET
& 330.5 & 216.0 / 114.5
& 295.9 & 210.5 / 85.4
& 426.9 & 276.3 / 150.7
& 364.0 & 216.5 / 147.4 \\
\bottomrule
\end{tabular}
\end{table}

\subsection{Ablation Study}
\label{sec:ablation}
We perform an ablation study on the ASML instances to isolate the contributions of the key components of our method. Starting with the full method, we consider four variants, each reflecting design choices made in prior work. Variant $-\text{STE}$ follows \textcite{Alvo.2023} by training in the continuous surrogate system and applying {the} integer mapping only at evaluation time, thereby removing the STE during training. Variant $-\text{BS}$ removes the echelon base-stock inductive bias and lets the neural network output production targets directly \parencite{Madeka.2022}. Variant $-\text{IM}$ replaces the dual-informed integer mapping with a simple rule that always floors production decisions to the nearest smaller integer as done by \textcite{Bottcher.2023}. Variant $-\text{KKT}$ keeps the same forward pass but applies an STE to the entire projection{ instead of differentiating its KKT conditions}. By Corollary~\ref{cor:competition}, this reproduces the Jacobians of Lemma~\ref{lem:qp_jacobians} at an empty active set: every coordinate is treated as \emph{Free} rather than \emph{Pinned} or \emph{Competing}, and the right-hand-side channel vanishes, as if no constraint were binding. In this case, the QP could be replaced by a MILP or any other non-differentiable projection method used in earlier work \parencite[e.g.,][]{Kaynov.2024}. {Additionally, we consider Variant $-\mathbf{W}$, which removes the QP weights and instead uses uniform weights $w_i = \frac1n$.}

The results are given in Table \ref{tab:asml_ablation_results}. $-\text{KKT}$ consistently leads to divergence during training, and is, hence, not reported in the table. {This is because the surrogate sensitivities mistakenly suggest that larger targets uniformly lead to larger actions.} Since $-\text{KKT}$ reproduces the naive learning signal of a non-differentiable projection within HDPO, such methods do not merely perform worse than ours, but cannot be trained end to end at all. This demonstrates that the analytical projection gradients are essential for stable training, highlighting the importance of our proposed differentiable projection module. {Among the remaining variants, removing the QP weights ($-\mathbf W$) or the STE in the integer mapping ($-\text{STE}$) does not measurably change performance: the gaps are at most $0.6\%$ and the confidence intervals include zero, with $-\text{STE}$ in $k^-$ ($0.6 \pm 0.6\%$) borderline. Both variants still have lower mean costs than all baselines. This confirms that the choice of weights is not decisive and shows that, on these instances, training in the continuous surrogate system and applying the integer mapping only at evaluation time is sufficient. In contrast, $-\text{BS}$ and $-\text{IM}$ increase costs significantly in all settings. Without the echelon base-stock inductive bias, $-\text{BS}$ increases costs by $4.6\text{--}10.1\%$, which shows the effectiveness of incorporating structural policy insights in line with \textcite{Xie.2026}.  $-\text{IM}$ increases costs the most in the $k^-$ setting with the tighter capacity ($35.9\%$ vs.\ $11.9\text{--}16.1\%$ otherwise).} This is consistent with the incompleteness of floor rounding on active actions shown in Proposition \ref{prop:completeness}. In particular, the crucial integer actions on the active boundary of the feasible set are often not reachable by floor rounding, as shown in Figure~\ref{fig:completeness}, leading to significant performance degradation. {Leaving capacity unused is most costly when capacity is tight, and this trend continues for even higher utilization ($k^{--}$), where costs increase by 78\% compared to our integer mapping. Overall, the ablation indicates that the three modules are essential: the base-stock target generation ($-\text{BS}$), the differentiable projection ($-\text{KKT}$), and the dual-informed integer mapping ($-\text{IM}$). All of them play a crucial role in achieving cost savings relative to state-of-the-art baselines and in ensuring stable training. The QP weights and the STE in the integer mapping have no measurable effect on these instances.}

\begin{table}[t]
\centering
\caption{Ablation of the neural policy across problem variants. Gaps are relative to the full model (Ours), in percent of its mean cost, with the paired 99\% confidence interval of the gap.}
\label{tab:asml_ablation_results}
\setlength{\tabcolsep}{5pt}
\begin{tabular}{lrrrrrrrr}
\toprule
 & \multicolumn{2}{c}{Base case}
& \multicolumn{2}{c}{$\alpha^+$}
& \multicolumn{2}{c}{$p^+$}
& \multicolumn{2}{c}{$k^-$} \\
\cmidrule(lr){2-3}
\cmidrule(lr){4-5}
\cmidrule(lr){6-7}
\cmidrule(lr){8-9}
Variant
& Cost & Gap $\pm$ CI
& Cost & Gap $\pm$ CI
& Cost & Gap $\pm$ CI
& Cost & Gap $\pm$ CI \\
\midrule
Ours
& 275.4 & --
& 248.4 & --
& 325.0 & --
& 294.3 & -- \\

\hline

$-\mathbf W$
& 275.3 & $-0.1 \pm 0.2$
& 248.6 & $0.1 \pm 0.2$
& 325.5 & $0.1 \pm 0.3$
& 294.8 & $0.2 \pm 0.4$ \\

$-$STE
& 275.5 & $0.0 \pm 0.2$
& 248.4 & $0.0 \pm 0.3$
& 324.7 & $-0.1 \pm 0.2$
& 296.0 & $0.6 \pm 0.6$ \\

$-$BS
& 288.1 & $4.6 \pm 0.8$
& 273.6 & $10.1 \pm 0.8$
& 345.6 & $6.3 \pm 0.8$
& 313.0 & $6.4 \pm 0.7$ \\

$-$IM
& 308.4 & $11.9 \pm 2.0$
& 288.3 & $16.1 \pm 2.8$
& 368.6 & $13.4 \pm 2.5$
& 400.0 & $35.9 \pm 6.9$ \\

\bottomrule
\end{tabular}
\end{table}

\section{Conclusion}
\label{sec:conclusion}

We studied end-to-end policy learning for constrained sequential decision problems with combinatorial, state-dependent action spaces, focusing on capacitated multi-echelon production--inventory planning as an application. The key obstacle is that hard, interdependent feasibility constraints break the differentiable scaling and Lagrangian relaxation schemes on which existing policies rely. To close this gap, we proposed a three-module differentiable policy: a differentiable decision rule proposes continuous action targets from the system state, a strongly convex QP projects them onto the feasible set of the continuous relaxation while supplying analytical sensitivities through implicit differentiation of its optimality conditions, and a dual-informed integer mapping restores integrality while preserving feasibility and providing surrogate gradients through a straight-through estimator. We showed that our method outperforms both classical policies and rolling-horizon multi-stage stochastic programming on a set of instances of increasing scale and complexity, including a realistic industry benchmark from the 2025 ASML Research Challenge with nonstationary demand. {We thus demonstrate state-of-the-art performance with MILP-like flexibility and millisecond inference speed.}

For managers, the results demonstrate when an end-to-end learned policy is worth its added complexity. When capacity is ample, well-tuned echelon base-stock policies with balanced allocation remain close to \change{the state of the art}. This changes as utilization, especially downstream, rises, demand becomes more variable, and resources are shared across echelons. In those regimes, how much to stock and how to ration scarce capacity across competing items become tightly coupled decisions. Our policy learns to optimize both jointly, whereas conventional policies separate them, which leads to cost savings in those instances.

A structural advantage of embedding constraint handling inside the policy is that the operational setting enters only through the QP: additional shared resources, alternative capacity structures, or changing material balances alter the QP coefficients but leave the neural policy, {integer mapping}, and training procedure untouched. In this sense, the policy mirrors the formulation-driven workflow of classical mixed-integer programming, where operational constraints are added to the model rather than to the algorithm. Transfer learning across constraint structures with the same architecture is an interesting avenue for future work.

Other research directions follow directly from the key assumptions we made. We assume differentiable system dynamics and exogenous uncertainty to apply the HDPO framework. Where these do not hold, the differentiable simulator can be replaced by a learned critic that provides policy gradients, as in deterministic actor--critic reinforcement learning. While this extends the applicability of our method, it also introduces additional challenges in terms of sample efficiency and training stability. When the relaxed action space is not downward closed, the componentwise floor no longer needs to be feasible, so our integer mapping requires an additional repair step to recover a feasible starting point. Further, we assume resource limits are given, although capacity can be strategically expanded through long-term investment. Using the projection sensitivities, the resource limits could be optimized alongside the policy to facilitate system design. More broadly, our method can be transferred to other constrained sequential decision problems, such as financial hedging and manufacturing process control.

\bibliographystyle{INFORMS-IJOC-Template-26-02-2026/informs2014}
\bibliography{references}

\newpage
\begin{APPENDIX}{}
  
\raggedbottom

\makeatletter
\renewcommand{\theHsection}{app.\arabic{section}}
\makeatother

\section{Production--Inventory Planning in General Networks}
\label{app:mdp}

\paragraph{State and sequence of events.} The state $\mathbf{s}_t$ is taken at the start of period $t$, after pipeline arrivals from period $t - \tau_i$ have been received but before any production decision or demand realization. It comprises the on-hand component inventory $o^C_{i,t} \in \mathbb{N}_0$ for $i \in \mathcal{I}^C$, the net inventory of end items $o^E_{i,t} \in \mathbb{Z}$ for $i \in \mathcal{I}^E$ (negative values represent backlog), the in-transit pipeline $(x_{i,t-l})_{l = 1, \ldots, \tau_i - 1}$ of orders released but not yet completed, and the demand-feature vector $\mathbf{s}_t^{\text{feat}}$. Formally,
\begin{equation*}
    \mathbf{s}_t := \big( (o^C_{i,t})_{i \in \mathcal{I}^C},\; (o^E_{i,t})_{i \in \mathcal{I}^E},\; (x_{i,t-l})_{i \in \mathcal{I},\, l = 1, \ldots, \tau_i - 1},\; \mathbf{s}_t^{\text{feat}} \big).
\end{equation*}

\begin{figure}[H]
    \centering
    \setlength{\unitlength}{0.75mm}%
\begin{picture}(200,60)
\thicklines
\put(0,35){\vector(1,0){200}}

\put(2,37){\makebox(0,5)[c]
{\shortstack[c]{\scriptsize $\dots$}}}

\put(195,37){\makebox(0,5)[c]
{\shortstack[c]{\scriptsize $\dots$}}}

\put(10,35){\line(0,1){5}}
\put(188,35){\line(0,1){5}}

\put(10,27){\makebox(0,5)[c]
{\shortstack[c]{\scriptsize \textbf{Start}\\\scriptsize\textbf{ period $t$}}}}

\put(188,27){\makebox(0,5)[c]
{\shortstack[c]{\scriptsize \textbf{End}\\\scriptsize\textbf{ period $t$}}}}

\put(35,35){\vector(0,1){7}}
\put(35,45){\makebox(0,15)[c]
{\shortstack[c]{\scriptsize \textbf{Pipeline arrivals.}\\
 \scriptsize Receipt of $x_{i,t-\tau_i}$.}}}

\put(75,35){\vector(0,-1){7}}
\put(75,10){\makebox(0,15)[c]
{\shortstack[c]{\scriptsize \textbf{Production decision.}\\
 \scriptsize $\mathbf{x}_t \in \mathcal{A}(\mathbf{s}_t)$ released;\\
 \scriptsize $\tilde o_{i,t} = o^C_{i,t} - \sum_{j \in \mathcal{I}} A_{i,j} x_{j,t}$,
 \scriptsize $i \in \mathcal{I}^C$.}}}

\put(120,35){\vector(0,1){7}}
\put(120,45){\makebox(0,15)[c]
{\shortstack[c]{\scriptsize \textbf{Demand realization.}\\
 \scriptsize $\xi_{i,t}$ observed;\\
 \scriptsize $\tilde o_{i,t} = o^E_{i,t} - \xi_{i,t}$,
 \scriptsize $i \in \mathcal{I}^E$.}}}

\put(160,35){\vector(0,-1){7}}
\put(160,15){\makebox(0,15)[c]
{\shortstack[c]{\scriptsize \textbf{Cost accrual.}\\
 \scriptsize Costs $c_t$ incurred.}}}

\end{picture}
    \vspace{-8mm}
    \caption{Sequence of events within period $t$.}
    \label{fig:timeline}
\end{figure}

The events within period $t$ unfold in the following order, illustrated in Figure~\ref{fig:timeline}:
\begin{enumerate}
    \item \textbf{Pipeline arrivals.} Production released $\tau_i$ periods earlier completes, so that $o^C_{i,t}$ and $o^E_{i,t}$ already reflect receipt of $x_{i,t-\tau_i}$.
    \item \textbf{Production decision.} The action $\mathbf{x}_t \in \mathbb{N}_0^{|\mathcal{I}|}$ is chosen, with $x_{i,t}$ the quantity of item $i$ released into the pipeline. {This consumes on-hand inventory of all input components in a convergent BOM structure.} The post-consumption component inventory is the intermediate variable
    \begin{equation*}
        \tilde o_{i,t} := o^C_{i,t} - \sum_{j \in \mathcal{I}} A_{i,j}\, x_{j,t}, \qquad i \in \mathcal{I}^C.
    \end{equation*}
    \item \textbf{Demand realization.} Demand realization $\xi_{i,t}$ for end item $i \in \mathcal{I}^E$ is observed and met from on-hand stock, yielding the post-demand end-item inventory
    \begin{equation*}
        \tilde o_{i,t} := o^E_{i,t} - \xi_{i,t}, \qquad i \in \mathcal{I}^E.
    \end{equation*}
    \item \textbf{Cost accrual.} Holding costs are charged on positive intermediate inventory and shortage costs on unmet end-item demand, {giving the per-period cost function}
    \begin{equation*}
        c(\mathbf{s}_t, \mathbf{x}_t, \boldsymbol \xi_t) = \sum_{i \in \mathcal{I}} h_i (\tilde o_{i,t})^+ + \sum_{i \in \mathcal{I}^E} p_i (-\tilde o_{i,t})^+.
    \end{equation*}
\end{enumerate}
{Given the action $\mathbf x_t$ and demand realization $\boldsymbol \xi_t$, the state transitions deterministically as $\mathbf{s}_{t+1} = \Gamma(\mathbf{s}_t, \mathbf{x}_t, \boldsymbol \xi_t)$, where the transition function $\Gamma$ shifts the pipeline forward by one period and updates on-hand inventories with the arrivals from period $t-\tau_i+1$:}
\[
\begin{aligned}
    o^C_{i,t+1} &= \tilde o_{i,t} + x_{i,t-\tau_i+1}, \qquad i \in \mathcal{I}^C,\\
    o^E_{i,t+1} &= \tilde o_{i,t} + x_{i,t-\tau_i+1}, \qquad i \in \mathcal{I}^E.
\end{aligned}
\]

\section{Echelon Base-Stock Target Generation}
\label{app:bs_target}

We instantiate the differentiable decision rule of Section~\ref{sec:method} with a structured parameterization tailored to inventory control. Rather than predicting production requests independently per item, the neural network first outputs a vector $\mathbf B_t \in \mathbb{R}_{\ge 0}^{|\mathcal I|}$ of state-dependent echelon base-stock levels, specifying desired echelon inventory positions that are then converted into production targets. This structure acts as an inductive bias \parencite{Xie.2026}, embedding problem-specific structure while preserving flexibility. As shown in the ablation study in Section~\ref{sec:ablation}, it improves performance in complex real-world supply networks. The echelon inventory position of item $i$ at time $t$ is
\[
    \mathrm{EIP}_{i,t}
    :=
    o^E_{i,t}
    +
    \sum_{l=1}^{\tau_i-1} x_{i,t-l},
    \qquad i \in \mathcal I^E,
\]
\[
    \mathrm{EIP}_{i,t}
    :=
    o^C_{i,t}
    +
    \sum_{l=1}^{\tau_i-1} x_{i,t-l}
    +
    \sum_{j \in \mathcal I} A_{i,j}\,\mathrm{EIP}_{j,t},
    \qquad i \in \mathcal I^C.
\]
End items provide the non-recursive base case, where $\mathrm{EIP}_{i,t}$ reduces to the local inventory position because end items have no downstream successors. Component items are then computed by propagating upward through the acyclic BOM, making the recursion well-defined. Including this BOM term generalizes the classical echelon-stock notion of \textcite{Clark.1960} to general bill-of-material structures: $A_{i,j}\,\mathrm{EIP}_{j,t}$ counts the units of component $i$ embedded in the echelon position of downstream item $j$, so $\mathrm{EIP}_{i,t}$ aggregates all material committed to item $i$, whether on hand, in the pipeline, or absorbed into a downstream stage, automatically coordinating upstream actions with downstream inventory. The corresponding continuous production target $\mathbf z_t$ is then defined componentwise as
\[
    \mathbf z_{t} := (\mathbf B_{t} - \mathbf{EIP}_{t})^+.
\]
The projection and integer-mapping modules of Section~\ref{sec:method} then map $\mathbf z_t$ to a feasible integer production decision.

  \section{Baseline Policy Details}\label{app:benchmarks}
\subsection{{Base-Stock} Policy}
\paragraph{Shortfall-priority allocation}

Given $\mathbf B_t$ with $B_{i,t} > 0$ for all $i \in \mathcal I$, echelon inventory positions $\mathbf{EIP}_{t}$, and on-hand component inventories $\mathbf o^C_t$, the allocator implements the nested min--max relative shortfall (\textit{waterfilling}) rule of \textcite{vanDijck.2026}. The relative shortfall of item $i$ at time $t$ given the current cumulative assignment $x_{i,t}$ is
\[
y_{i,t} \;=\; \frac{\big(B_{i,t} - \mathrm{EIP}_{i,t} - x_{i,t}\big)^+}{B_{i,t}}.
\]
At each step, the next unit is assigned to the item with the largest $y_{i,t}$ if feasible. When allocating to the leading item becomes infeasible, allocation moves to the next-most-short feasible item. The procedure terminates when no further production is feasible, or all shortfalls are eliminated.

\paragraph{Echelon lead-time demand}
Let $a^{\mathrm{tot}}_{i,j}$ denote the total number of units of item $i \in \mathcal I$ required per unit of end item $j \in \mathcal I^E$, obtained by summing the products of BOM coefficients over all directed paths from $i$ to $j$. We set $a^{\mathrm{tot}}_{j,j}=1$ and $a^{\mathrm{tot}}_{i,j}=0$ if $i$ does not feed $j$. Let $L_{j,i}$ denote the corresponding echelon lead time, computed as the longest path from $i$ to $j$ over per-item lead times $\tau_l$ plus a unit increment for the periodic review, and set $L_{j,i}=0$ for unserved end items. For demand that is i.i.d.\ across periods and independent across end items, the static mean and variance of the echelon lead-time demand for item $i$ are
\[
m'_i \;=\; \sum_{j \in \mathcal I^E} a^{\mathrm{tot}}_{i,j}\,L_{j,i}\, \expect{\Xi_j},
\qquad
(s'_i)^2 \;=\; \sum_{j \in \mathcal I^E} \bigl(a^{\mathrm{tot}}_{i,j}\bigr)^2 L_{j,i}\, \operatorname{Var}(\Xi_j).
\]
The aggregate requirement is squared in the variance because all paths to the same end item are driven by the same demand realization. For demand with cross-item or temporal dependence, the moments of the aggregate lead-time demand additionally account for the corresponding covariance terms. For DBS, the conditional mean and standard deviation $(m'_{i,t}, s'_{i,t})$ are obtained from the closed-form recursions of the AR$(p)$ aggregate-demand model. We refer to \textcite[][Appendix B]{vanDijck.2026} for the full derivation.

\subsection{Multi-Stage Stochastic Programming} \label{app:mss}
The multi-stage stochastic program is formalized below. Following the action-then-demand timing of Appendix~\ref{app:mdp}, the production decision indexed by a node is made before observing the demand realization associated with that node. In addition to the notation of Section~\ref{sec:application}, which consists of the BOM $\mathbf{A}$, resource consumption $\mathbf{C}$, capacities $\mathbf{k}$, lead times $\tau_i$, and unit holding and backlog costs $h_i$ and $p_i$, it uses the following notation.

\noindent \textbf{Sets}
\begin{conditions*}
    \mathcal{T} & {Lookahead} periods relative to the current period, $\mathcal{T} = \{0, 1, \ldots, H\}$ for a {lookahead horizon} of $H$ periods, indexed by $t$, with $t = 0$ the current period (root) \\
    \mathcal{N} & Nodes of the scenario tree, indexed by $n$ or $n^{\prime}$ \\
    \mathcal{N}_n^A & Ancestors of node $n \in \mathcal{N}$, i.e., the nodes on the path from the root to $n$ including itself
\end{conditions*}
\vspace{2mm}
\noindent \textbf{Parameters}
\begin{conditions*}
    \xi_{i,n} \in \mathbb{N}_0 & Demand of end item $i \in \mathcal{I}^E$ sampled at node $n \in \mathcal{N}$ \\
    t(n) \in \mathcal{T} & Time slot associated with node $n \in \mathcal{N}$ \\
    a(n) \in \mathcal{N} & Parent node of node $n \in \mathcal{N}$, which is undefined for the root node in time slot $0$ \\
    \rho_n \in (0, 1] & Probability of reaching node $n \in \mathcal{N}$, with $\sum\limits_{n \in \mathcal{N}\,:\, t(n) = t} \rho_n = 1$ for all $t \in \mathcal{T}$ \\
    \Bar{o}_i \in \mathbb{Z} & Initial inventory of item $i \in \mathcal{I}$, which can be negative for end items $i \in \mathcal{I}^E$ \\
    \Bar{x}_{i,t} \in \mathbb{N}_0 & Initial in-transit pipeline of item $i \in \mathcal{I}$ arriving $t$ periods from now, with $\Bar{x}_{i,t} = 0$ for $t \geq \tau_i$
\end{conditions*}
\vspace{2mm}
\noindent \textbf{Decision variables}
\begin{conditions*}
    x_{i,n} \in \mathbb{N}_0 & Pre-demand production quantity of item $i \in \mathcal{I}$ at node $n \in \mathcal{N}$ \\
    o_{i,n} \in \mathbb{N}_0 & Inventory of item $i \in \mathcal{I}$ at node $n \in \mathcal{N}$ \\
    b_{i,n} \in \mathbb{N}_0 & Backlog of end item $i \in \mathcal{I}^E$ at node $n \in \mathcal{N}$
\end{conditions*}
\begin{subequations}
\begin{align}
    \min \ \ \ & \sum_{n \in \mathcal{N}} \rho_n \left(\sum_{i \in \mathcal{I}} h_i o_{i,n} + \sum_{i \in \mathcal{I}^E} p_i b_{i,n}\right) && & \label{MSS-PIP objective function} \\[5pt]
    \text{s.t.} \ \ \ & \Bar{o}_i + \sum_{t^{\prime}=1}^{t(n)} \Bar{x}_{i,t^{\prime}} + \hspace{-2mm} \sum_{\substack{n^{\prime} \in \mathcal{N}_n^A \\ t(n^{\prime}) \leq t(n) - \tau_i}} \hspace{-2mm} x_{i,n^{\prime}} - \sum_{n^{\prime} \in \mathcal{N}_n^A} \xi_{i,n^{\prime}} + b_{i,n} = o_{i,n} && i \in \mathcal{I}^E, n \in \mathcal{N} & \label{MSS-PIP constraints: end-item balance} \\[8pt]
    & \Bar{o}_i + \sum_{t^{\prime}=1}^{t(n)} \Bar{x}_{i,t^{\prime}} + \hspace{-2mm} \sum_{\substack{n^{\prime} \in \mathcal{N}_n^A \\ t(n^{\prime}) \leq t(n) - \tau_i}} \hspace{-2mm} x_{i,n^{\prime}} - \sum_{j \in \mathcal{I}} A_{i,j} \left( \sum_{n^{\prime} \in \mathcal{N}_n^A} x_{j,n^{\prime}} \right) = o_{i,n} && i \in \mathcal{I}^C, n \in \mathcal{N} & \label{MSS-PIP constraints: component balance} \\[8pt]
    & \sum_{i \in \mathcal{I}} C_{g,i}\, x_{i,n} \leq k_g && g \in \mathcal{G}, n \in \mathcal{N} & \label{MSS-PIP constraints: capacity} \\
    & x_{i,n} = x_{i,n^{\prime}} && i \in \mathcal{I}; n, n^{\prime} \in \mathcal{N} : a(n) = a(n^{\prime}) & \label{MSS-PIP constraints: non-anticipativity} \\
    & x_{i,n}, o_{i,n} \in \mathbb{N}_0 && i \in \mathcal{I}, n \in \mathcal{N} & \label{MSS-PIP constraints: integrality} \\
    & b_{i,n} \in \mathbb{N}_0 && i \in \mathcal{I}^E, n \in \mathcal{N} & \label{MSS-PIP constraints: integrality backlog}
\end{align}
\end{subequations}

The objective \eqref{MSS-PIP objective function} is to minimize the expected total holding and backlog costs. Constraints \eqref{MSS-PIP constraints: end-item balance} set the inventory and backlog values for the {end items}, while constraints \eqref{MSS-PIP constraints: component balance} specify the inventory levels for components, both at each node of the scenario tree. Constraints \eqref{MSS-PIP constraints: capacity} enforce the shared resource capacities. Constraints \eqref{MSS-PIP constraints: non-anticipativity} impose nonanticipativity: nodes with the same parent differ only in the demand drawn in the current period, which is not yet observed when production is released (Appendix~\ref{app:mdp}). Lastly, constraints \eqref{MSS-PIP constraints: integrality} and \eqref{MSS-PIP constraints: integrality backlog} impose {nonnegativity} and integrality on all decision variables.

\section{Experimental Configuration}
\label{app:experimental_configuration}

We implement our policy and simulation in PyTorch \parencite{PyTorch} and integrate the open-source {CPU-bound active-set solver from Moreau \parencite{moreau}} for the differentiable projection layer. The MILP policies are implemented using Gurobi 12 with a non-binding time limit of 3600 seconds. We run experiments distributed across cores, with one Intel Xeon Platinum 8280L CPU at 2.70 GHz per process. 

\subsection{Implementation Details of Our Policy}
\label{app:policy_implementation}

For small instances, we train for $500$ epochs, and for larger instances and the case study, we train for $1{,}000$ epochs to ensure stable convergence across instances. To approach long-run performance on small instances, we use trajectories of length $200$ during training, with the initial state distribution generated by a $100$-period warmup simulation under the current policy. For the larger instances, we use the planning horizon $T$ for each trajectory with an additional warmup period of $\lfloor T/2\rfloor$ periods. 

We augment the state with the echelon inventory position as a handcrafted feature. We scale the state, action, and costs during training to improve numerical stability. The inventory levels and pipelines are scaled by the mean demand per period, whereas the echelon inventory position is scaled by the mean demand during the echelon lead time. For the case study, the aggregate demand lags are scaled by the long-run mean. The immediate costs are scaled by the average costs per period of $\overline{\mathrm{SBS}}$. While this is not strictly necessary, it avoids the need to tune the learning rate separately for each instance. Before each update, we clip the gradient norm to 10. To keep the neural network output normalized, the echelon base-stock network uses softplus as the output activation and then scales by the base-stock levels of $\overline{\mathrm{SBS}}$. The raw action network in the ablation study uses a sigmoid as the output activation and then scales by the lowest limit of any resource that the respective item consumes.

\subsection{Baseline Configurations}
\label{app:baseline_hyperparameters}

For the ASML case study, we use the complete baseline configurations of \textcite{vanDijck.2026} without modification. This includes the fitted item-specific and global safety-stock factors $\boldsymbol\theta$ and $\theta$ for the {base-stock} policies as well as the lookahead horizon of $H=30$ for the MILP policies with demand inflation factor $\gamma$ for DET, and the scenario tree branches $[6, 4, 3, 2, 2] + [1]\times(H-5)$ for MSP. We therefore do not retune the baseline policies for the ASML instances.

For the remaining instances, we fit the parameters separately using the following hyperparameters. For MSP, we consider both two times and four times the maximum echelon lead time for the lookahead horizon $H$ to account for the downstream effect of upstream production. We fix the branching structure to $[6,4,3,2,2]+[1]\times(H-5)$ across instances, yielding $288$ scenarios. For DET, we fix the lookahead horizon $H$ to two times the maximum echelon lead time and tune $\gamma$ via a simulation-based grid search over $[0.1,0.9]$ in increments of $0.05$. For $\overline{\mathrm{SBS}}$, we tune $\theta$ via a simulation-based grid search over $[1, 7]$ in increments of $0.1$. All simulation-based grid searches are performed using the same set of $1{,}000$ demand trajectories of length $T$ preceded by $\lfloor T/2 \rfloor$ warmup periods for each candidate value. The intervals are chosen such that the best-found value for all instances is strictly within the search space. Following \textcite{vanDijck.2026}, we optimize $\boldsymbol \theta$ for SBS using CMA-ES and follow their configuration. We set the population size to $4 + \lfloor 3.0 \cdot \log(\dim(\boldsymbol \theta)) \rfloor$, and the number of parents to half of that. We initialize the population mean to the best global $\theta$ with an initial step size of $1$. The optimization runs for 200 iterations, and candidates within each iteration $l$ are evaluated on the same set of $1{,}000 + 40l$ demand trajectories of length $T$ preceded by $\lfloor T/2 \rfloor$ warmup periods. The total training budget for our method and CMA-ES is thus approximately $1{,}000{,}000$ trajectories per instance. We use the same training trajectories for our policy and the baseline policies.

\section{Runtimes and Scalability}\label{app:runtime}

We compare the runtimes of our policy against the baseline policies. First, we report the average training times in Section \ref{app:training_time} and then the inference times in Section \ref{app:inference_time}. All runtimes are for a single core.

\subsection{Training Times}
\label{app:training_time}

{All evaluated methods except MSP require some form of training or tuning. The training protocols are detailed in Section \ref{app:experimental_configuration}. The resulting averaged runtimes for our policy and the baseline policies are reported in Table~\ref{tab:runtime_training}. $\overline{\mathrm{SBS}}$ is by far the fastest to tune, with almost negligible runtimes of under 10 minutes, while DET requires the longest time to tune, with about 25 hours for the ASML instances and potentially up to 29 days for the cyclic 40-item instances of Class B (extrapolated from the per-action inference times). The safety-stock factor optimization for SBS via CMA-ES is faster than the training of our method for the small instances of problem class A, but takes almost twice as long for the larger instances of problem class B. For the ASML instances, DBS requires about 10\% longer than SBS due to the forecast computation. DET is the only method whose tuning time significantly depends on the resource structure (cyclic vs. noncyclic, see Section \ref{app:inference_time}). In summary, our method not only has state-of-the-art performance, but also competitive training times and better scalability than the best baselines, which are SBS and DBS. }

\begin{table}[t]
\centering
\footnotesize
\caption{Approximate average training times for our policy and the baselines on different instances.}
\label{tab:runtime_training}
\begin{tabular}{llrrr}
\toprule
& & \multicolumn{1}{c}{Class A} & \multicolumn{1}{c}{Class B} & \multicolumn{1}{c}{ASML} \\
Policy & Procedure & \multicolumn{1}{c}{10 items} & \multicolumn{1}{c}{40 items} & \multicolumn{1}{c}{14 items} \\
\midrule
Ours & $1{,}000$ training epochs & 2.3\,h & 8.2\,h & 3.2\,h \\
SBS & CMA-ES, $200$ iterations & 1.4\,h & 16.4\,h & 5.1\,h \\
$\overline{\mathrm{SBS}}$ & grid search & 2\,min & 6\,min & 3\,min \\
DET & grid search & 2.9\,h & 20\,h--29\,d$^{\dagger}$ & 24.6\,h \\
\bottomrule
\multicolumn{5}{@{}l}{\rule{0pt}{2.6ex}$^{\dagger}$\,Extrapolated from per-action inference times.}
\end{tabular}
\end{table}

\subsection{Inference Times}
\label{app:inference_time}
{We compare the inference time of the proposed policy against MSP solved as an LP with $H=2L_{\max}$ and $H=4L_{\max}$ as well as DET with $H=2L_{\max}$ solved as a MILP. Here, $L_{\max}$ is the maximum echelon lead time. Since the base-stock policies have negligible inference times, they are not compared here. We consider the three sets of network structures proposed by \textcite{Tempelmeier.1996} containing 10, 40, and 100 items. For all networks, we use the instance with capacity profile 1 (utilization of 0.9) and $\mathrm{cov}=0.9$, with other parameters as determined in Section \ref{sec:tempelmeier}. For our method, we report the average inference time per action of a trained policy for different batch sizes $N$. The upper panel of Table~\ref{tab:runtime_merged} compares the optimization problem sizes of the methods, in terms of the number of variables and constraints. Our policy solves at most a single QP with 10--100 variables per period and 19--185 constraints. When decomposing the problem, the largest QP has $40$ variables and 72 constraints for the cyclic 100-item instance. MSP and DET optimize over the entire lookahead horizon, which amounts to up to $2.5$ million and $9{,}460$ variables, respectively.}

{The lower panel of Table~\ref{tab:runtime_merged} and Figure~\ref{fig:runtime_scaling} show the average time per action over 60 periods along a common trajectory. Even without decomposing the QP, our policy generates a single action ($N=1$) within $1.2$--$4.5$ ms for 10 to 100 items, and its latency grows sublinearly by only a factor of $4$ over this tenfold increase in size. In contrast, MSP with $H=2L_{\max}$ grows from $1.7$ s to $681$--$874$ s, and with $H=4L_{\max}$ it requires up to $1{,}050$ s and finds no feasible solution within the one-hour limit at 100 items. Our policy is thus three to five orders of magnitude faster than MSP. DET is competitive only on the smallest instances. With $H=2L_{\max}$, our policy is faster than DET by factors of $3.1$--$8.5$, $25$--$710$, and $240$--$91{,}000$ with 10, 40, and 100 items. Batching further reduces latency to $0.03$--$1.9$ ms per action at $N=1000$. If the QP decomposes, the decomposed projection is $1.7$--$5.7$ times faster than the joint QP at $N=1000$ for 40 and 100 items. At $N=1$ it is up to $1.9$ times slower.}

{The problem sizes and runtimes for the ASML base case are reported in Table~\ref{tab:policy_runtime_asml}. For a single action ($N=1$), our policy is 10,000 times faster than MSP and 38 times faster than DET. In this instance, the joint QP is faster than the decomposition because the latter solves multiple smaller QPs sequentially. Batching further reduces average latency, making our inference 570,000 times faster than MSP and 2,200 times faster than DET at $N=1000$.}

\begin{figure}[t]
\centering
\includegraphics[width=\linewidth]{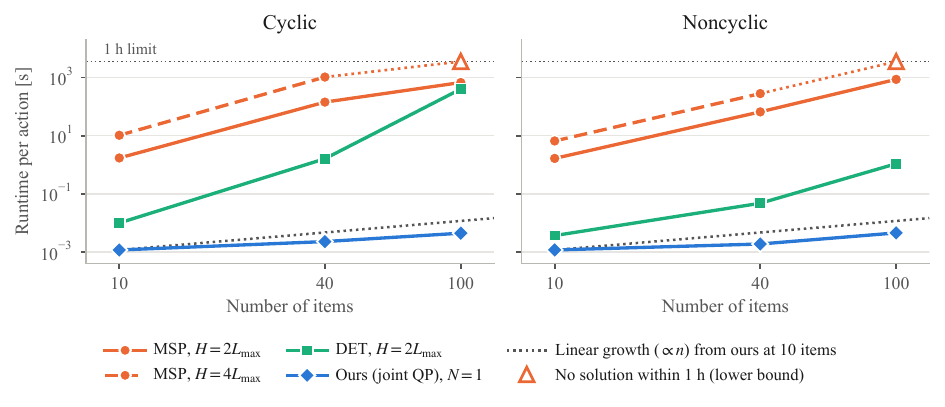}
\caption{Runtime per action on the networks from \textcite{Tempelmeier.1996} with capacity profile 1 and $\mathrm{cov}=0.9$.}
\label{fig:runtime_scaling}
\end{figure}

\begin{table}[t]
\centering
\caption{Problem sizes (number of variables and constraints) and mean runtime per action in seconds across networks from \textcite{Tempelmeier.1996} with capacity profile 1 and {$\mathrm{cov}=0.9$}.}
\label{tab:runtime_merged}
\resizebox{\linewidth}{!}{%
\begin{tabular}{ll*{9}{r}}
\toprule
& & \multicolumn{1}{c}{DET}
  & \multicolumn{2}{c}{MSP}
  & \multicolumn{3}{c}{Ours (decomposed)}
  & \multicolumn{3}{c}{Ours (joint QP)} \\
\cmidrule(lr){3-3}
\cmidrule(lr){4-5}
\cmidrule(lr){6-8}
\cmidrule(lr){9-11}
Items & 
& $H=2L_{\max}$
& $H=2L_{\max}$ & $H=4L_{\max}$
& {$N=1$} & {$N=100$} & {$N=1000$}
& {$N=1$} & {$N=100$} & {$N=1000$} \\
\midrule
\multicolumn{11}{l}{\emph{Problem size}} \\
10 & Variables & 192 & 33{,}552 & 88{,}848 & \multicolumn{3}{c}{3--10} & \multicolumn{3}{c}{10} \\
& Constraints & 104 & 21{,}044 & 50{,}996 & \multicolumn{3}{c}{4--19} & \multicolumn{3}{c}{19} \\
\addlinespace
40 & Variables & 1{,}032 & 219{,}300 & 516{,}516 & \multicolumn{3}{c}{4--20} & \multicolumn{3}{c}{40} \\
& Constraints & 552 & 128{,}780 & 287{,}756 & \multicolumn{3}{c}{9--40} & \multicolumn{3}{c}{70} \\
\addlinespace
100 & Variables & 4{,}730 & 1{,}167{,}450 & 2{,}529{,}690 & \multicolumn{3}{c}{3--40} & \multicolumn{3}{c}{100} \\
& Constraints & 2{,}420 & 626{,}000 & 1{,}322{,}960 & \multicolumn{3}{c}{4--72} & \multicolumn{3}{c}{185} \\
\midrule
\multicolumn{11}{l}{\emph{Runtime per action [s]}} \\
10 & Cyclic
& $9.96\times10^{-3}$
& $1.72\times10^{0}$ & $1.04\times10^{1}$
& $1.15\times10^{-3}$ & $8.71\times10^{-5}$ & $3.30\times10^{-5}$
& $1.17\times10^{-3}$ & $9.66\times10^{-5}$ & $3.34\times10^{-5}$ \\

10 & Noncyclic
& $3.67\times10^{-3}$
& $1.66\times10^{0}$ & $6.62\times10^{0}$
& $1.97\times10^{-3}$ & $1.91\times10^{-4}$ & $4.96\times10^{-5}$
& $1.17\times10^{-3}$ & $9.43\times10^{-5}$ & $3.24\times10^{-5}$ \\
\addlinespace

40 & Cyclic
& $1.62\times10^{0}$
& $1.44\times10^{2}$ & $1.05\times10^{3}$
& $3.60\times10^{-3}$ & $4.04\times10^{-4}$ & $1.36\times10^{-4}$
& $2.28\times10^{-3}$ & $2.97\times10^{-4}$ & $2.32\times10^{-4}$ \\

40 & Noncyclic
& $4.75\times10^{-2}$
& $6.64\times10^{1}$ & $2.83\times10^{2}$
& $3.50\times10^{-3}$ & $3.79\times10^{-4}$ & $1.39\times10^{-4}$
& $1.88\times10^{-3}$ & $2.93\times10^{-4}$ & $2.41\times10^{-4}$ \\
\addlinespace

100 & Cyclic
& $4.05\times10^2$
& $6.81\times10^{2}$ & $>3.60\times10^{3}$
& $6.78\times10^{-3}$ & $8.76\times10^{-4}$ & $4.51\times10^{-4}$
& $4.43\times10^{-3}$ & $1.88\times10^{-3}$ & $1.80\times10^{-3}$ \\

100 & Noncyclic
& $1.09\times10^{0}$
& $8.74\times10^{2}$ & $>3.60\times10^{3}$
& $7.11\times10^{-3}$ & $7.55\times10^{-4}$ & $3.27\times10^{-4}$
& $4.53\times10^{-3}$ & $1.92\times10^{-3}$ & $1.85\times10^{-3}$ \\
\bottomrule
\end{tabular}%
}
\end{table}

\begin{table}[t]
\centering
\footnotesize
\caption{Problem sizes and mean runtime per action in seconds for the ASML case study base instance.}
\label{tab:policy_runtime_asml}
\begin{tabular}{l*{8}{r}}
\toprule
& \multicolumn{1}{c}{DET}
  & \multicolumn{1}{c}{MSP}
  & \multicolumn{3}{c}{Ours (decomposed)}
  & \multicolumn{3}{c}{Ours (joint QP)} \\
\cmidrule(lr){2-2}
\cmidrule(lr){3-3}
\cmidrule(lr){4-6}
\cmidrule(lr){7-9}
& $H=30$ & $H=30$
& {$N=1$} & {$N=100$} & {$N=1000$}
& {$N=1$} & {$N=100$} & {$N=1000$} \\
\midrule
Variables
& 930 & 239{,}754
& \multicolumn{3}{c}{1--3}
& \multicolumn{3}{c}{14} \\
Constraints
& 690 & 181{,}900
& \multicolumn{3}{c}{1--8}
& \multicolumn{3}{c}{34} \\
\addlinespace
Runtime per action
& $5.8\times10^{-2}$ & $1.55\times10^{1}$
& $2.42\times10^{-3}$ & $1.67\times10^{-4}$ & $4.41\times10^{-5}$
& $1.52\times10^{-3}$ & $7.78\times10^{-5}$ & $2.69\times10^{-5}$ \\
\bottomrule
\end{tabular}
\end{table}

\newpage
\section{Proofs}
\label{app:proofs}
\begin{proof}{Proof of Lemma~\ref{lem:qp_jacobians}}
Let {$\boldsymbol{\lambda}\in\mathbb R_{\ge 0}^{m}$} denote the Lagrange multipliers associated with the constraints $\mathbf M \mathbf x \le \mathbf r_t$. The Lagrangian of \eqref{eq:qp_objective} is
\[
\mathcal L(\mathbf{x},\boldsymbol{\lambda})
=
\frac12 {(\mathbf{x}-\mathbf{z}_t)^\top} \mathbf{W}(\mathbf{x}-\mathbf{z}_t)
+\boldsymbol{\lambda}^\top(
\mathbf{M}\mathbf{x}-\mathbf{r}_t).
\]
Under a locally fixed linearly independent active set, the reduced KKT system for
\eqref{eq:qp_objective} is
\begin{equation}
\begin{bmatrix}
\mathbf{W} & \tilde{\mathbf{M}}_t^\top\\
\tilde{\mathbf{M}}_t & \mathbf{0}
\end{bmatrix}
\begin{bmatrix}
\mathbf{x}_t^*\\
\boldsymbol{\tilde \lambda}_t^*
\end{bmatrix}
=
\begin{bmatrix}
\mathbf{W}\mathbf{z}_t\\
\tilde{\mathbf{r}}_t
\end{bmatrix},
\label{eq:reduced_kkt}
\end{equation}
where $\boldsymbol{\tilde \lambda}_t^*$ collects the optimal Lagrange multipliers (dual prices) of the active constraints. On a neighborhood where the active set is fixed, \eqref{eq:reduced_kkt} is a linear system whose coefficient matrix is invertible because $\mathbf W \succ \mathbf 0$ and the active rows are linearly independent. Its solution $(\mathbf x_t^*, \boldsymbol{\tilde \lambda}_t^*)$ therefore depends affinely, hence differentiably, on the inputs $(\mathbf z_t, \mathbf b_t)$. Differentiating \eqref{eq:reduced_kkt} with respect to $\mathbf z_t$ gives
\[
\mathbf{W}\,\frac{\partial \mathbf{x}_t^*}{\partial \mathbf{z}_t}
+\tilde{\mathbf{M}}_t^\top\frac{\partial \boldsymbol{\tilde\lambda}_t^*}{\partial \mathbf{z}_t}
=\mathbf{W},
\qquad
\tilde{\mathbf{M}}_t\,\frac{\partial \mathbf{x}_t^*}{\partial \mathbf{z}_t}=\mathbf{0}.
\]
Solving the first equation for $\frac{\partial \mathbf{x}_t^*}{\partial \mathbf{z}_t}$ and substituting into the second yields
$\frac{\partial \boldsymbol{\tilde\lambda}_t^*}{\partial \mathbf{z}_t}
=\bigl(\tilde{\mathbf{M}}_t\mathbf{W}^{-1}\tilde{\mathbf{M}}_t^\top\bigr)^{-1}\tilde{\mathbf{M}}_t$;
substituting this back gives the first claim. Differentiating \eqref{eq:reduced_kkt} with respect to $\mathbf b_t$ instead gives
\[
\mathbf{W}\,\frac{\partial \mathbf{x}_t^*}{\partial \mathbf{b}_t}
+\tilde{\mathbf{M}}_t^\top\frac{\partial \boldsymbol{\tilde\lambda}_t^*}{\partial \mathbf{b}_t}
=\mathbf{0},
\qquad
\tilde{\mathbf{M}}_t\,\frac{\partial \mathbf{x}_t^*}{\partial \mathbf{b}_t}=\mathbf{U}_t,
\]
and the same elimination yields the second claim. For the general sensitivity formulas for QPs, see \textcite{Amos.2017}.
\end{proof}

\begin{proof}{Proof of Corollary~\ref{cor:competition}, part (a) (resource competition)}
Define
\(
\boldsymbol{\Pi}_t
:=
\mathbf I
-
\mathbf W^{-1/2}\tilde{\mathbf M}_t^\top
\left(\tilde{\mathbf M}_t\mathbf W^{-1}\tilde{\mathbf M}_t^\top\right)^{-1}
\tilde{\mathbf M}_t\mathbf W^{-1/2}.
\)
Then $\boldsymbol{\Pi}_t$ is the Euclidean orthogonal projector onto
\(\ker(\tilde{\mathbf M}_t\mathbf W^{-1/2})\), and the formula in
Lemma~\ref{lem:qp_jacobians} reads
\(\partial \mathbf x_t^*/\partial \mathbf z_t
=\mathbf W^{-1/2}\boldsymbol{\Pi}_t\mathbf W^{1/2}\).
Since $\mathbf W$ is diagonal, its $i$th column and the corresponding diagonal entry are
\begin{equation}
\frac{\partial \mathbf x_t^*}{\partial z_{i,t}}
=
w_i^{1/2}\,\mathbf W^{-1/2}\boldsymbol{\Pi}_t\mathbf e_i,
\qquad
\frac{\partial x_{i,t}^*}{\partial z_{i,t}}
=
[\boldsymbol{\Pi}_t]_{ii}
=
\mathbf e_i^\top \boldsymbol{\Pi}_t\mathbf e_i
=
\|\boldsymbol{\Pi}_t\mathbf e_i\|_2^2
\in [0,1],
\label{eq:jacobian_column}
\end{equation}
where the last equality uses that $\boldsymbol{\Pi}_t$ is symmetric and idempotent, and the upper bound uses that orthogonal projections do not increase the Euclidean norm. Both expressions vanish exactly when $\boldsymbol{\Pi}_t\mathbf e_i=\mathbf 0$, so the column is zero if and only if its diagonal entry is. Lemma~\ref{lem:qp_jacobians} also gives $\tilde{\mathbf{M}}_t\, \partial \mathbf x_t^*/\partial \mathbf z_t = \mathbf{0}$.

\emph{Free regime.} If coordinate $i$ appears in no active row, the $i$th column of $\tilde{\mathbf M}_t$ is zero, so the formula in Lemma~\ref{lem:qp_jacobians} gives $\partial \mathbf x_t^*/\partial z_{i,t}=\mathbf e_i$ and thus $\partial x_{i,t}^*/\partial z_{i,t}=1$.

\emph{Pinned and Competing regimes.} Suppose coordinate $i$ appears in some active row. Since $x_{i,t}^*>0$, the lower-bound row $-\mathbf e_i^\top$ is not active, and the lower-bound row $-\mathbf e_l^\top$ of any other coordinate $l\ne i$ has a zero entry in column $i$; coordinate $i$ therefore appears in an active $\mathbf A$/$\mathbf C$ row. It suffices to show that if $\partial \mathbf x_t^*/\partial z_{i,t}$ contains no strictly negative entry, it is identically zero. Assume $\partial x_{l,t}^*/\partial z_{i,t} \ge 0$ for every $l \in [n]$. For each active row $k$,
\[
\sum_{l \in [n]} \tilde{M}_{kl}\,\frac{\partial x_{l,t}^*}{\partial z_{i,t}} = 0.
\]
For active lower-bound rows $-\mathbf e_l^\top$, this equation gives $\partial x_{l,t}^*/\partial z_{i,t}=0$ directly. For active $\mathbf A$/$\mathbf C$ rows, every summand is the product of two nonnegative factors, because {$\mathbf A, \mathbf C \ge \mathbf 0$}. Since the sum is zero, each summand must vanish. Thus, any coordinate appearing with a positive coefficient in an active $\mathbf A$/$\mathbf C$ row has $\partial x_{l,t}^*/\partial z_{i,t}=0$. Any remaining coordinate $l\ne i$ appears in no active row, so the $l$th column of $\tilde{\mathbf M}_t$ is zero, and the formula in Lemma~\ref{lem:qp_jacobians} gives $\partial x_{l,t}^*/\partial z_{i,t}=0$. Coordinate $i$ appears with a positive coefficient in an active $\mathbf A$/$\mathbf C$ row, so the same argument also gives $\partial x_{i,t}^*/\partial z_{i,t}=0$. Hence, $\partial \mathbf x_t^*/\partial z_{i,t}=\mathbf 0$ (Pinned regime). Otherwise, $\partial \mathbf x_t^*/\partial z_{i,t}$ contains at least one strictly negative entry, and since $\partial x_{i,t}^*/\partial z_{i,t} \ge 0$ that entry is off-diagonal (Competing regime). In this case $\partial \mathbf x_t^*/\partial z_{i,t}\ne\mathbf 0$, so $\partial x_{i,t}^*/\partial z_{i,t}>0$ by \eqref{eq:jacobian_column}. Since coordinate $i$ appears in an active row, $\mathbf e_i\notin \ker(\tilde{\mathbf M}_t\mathbf W^{-1/2})$, so $\boldsymbol{\Pi}_t\mathbf e_i\ne\mathbf e_i$ and therefore $\partial x_{i,t}^*/\partial z_{i,t}<1$.

Finally, suppose every active row $k$ with $\tilde{M}_{ki} > 0$ has $\tilde{M}_{kl} = 0$ for all $l \ne i$, so that these rows are scalar bounds on $x_{i,t}$. Row $k$ of $\tilde{\mathbf M}_t\,\partial \mathbf x_t^*/\partial z_{i,t} = \mathbf 0$ then collapses to $\tilde{M}_{ki}\,\partial x_{i,t}^*/\partial z_{i,t} = 0$, so $\partial x_{i,t}^*/\partial z_{i,t}=0$ and hence $\partial \mathbf x_t^*/\partial z_{i,t}=\mathbf 0$ by \eqref{eq:jacobian_column}. The Competing regime therefore requires some active row $k$ with $\tilde{M}_{ki} > 0$ and $\tilde{M}_{kj} > 0$ for some $j \ne i$.
\end{proof}

\begin{proof}{Proof of Corollary~\ref{cor:competition}, part (b) (nonexpansiveness)}
As shown in the proof of part (a), \(\partial \mathbf x_t^*/\partial \mathbf z_t=\mathbf W^{-1/2}\boldsymbol{\Pi}_t\mathbf W^{1/2}\), where \(\boldsymbol{\Pi}_t\) is the Euclidean orthogonal projector onto \(\ker(\tilde{\mathbf M}_t\mathbf W^{-1/2})\). Equivalently, \(\partial \mathbf x_t^*/\partial \mathbf z_t\) is the \(\mathbf W\)-orthogonal projector onto \(\ker\tilde{\mathbf M}_t\). Hence, for every perturbation \(\mathbf h\),
\[
    \Bigl\|\frac{\partial \mathbf x_t^*}{\partial \mathbf z_t}\mathbf h\Bigr\|_{\mathbf W}
    =
    \|\boldsymbol{\Pi}_t\mathbf W^{1/2}\mathbf h\|_2
    \le
    \|\mathbf W^{1/2}\mathbf h\|_2
    =
    \|\mathbf h\|_{\mathbf W},
\]
where the inequality uses that orthogonal projections do not increase the Euclidean norm.
\end{proof}

\begin{proof}{Proof of Corollary~\ref{cor:competition}, part (c) (constraint filling)}
By Lemma~\ref{lem:qp_jacobians},
\[
\tilde{\mathbf{M}}_t\,
\frac{\partial \mathbf{x}_t^*}{\partial \mathbf{b}_t}
=
\tilde{\mathbf{M}}_t \mathbf{W}^{-1} \tilde{\mathbf{M}}_t^\top
\bigl(\tilde{\mathbf{M}}_t \mathbf{W}^{-1} \tilde{\mathbf{M}}_t^\top\bigr)^{-1}
\mathbf{U}_t
=
\mathbf{U}_t. 
\]
For a perturbation \(\Delta\mathbf b_t\), consider the equality-constrained
quadratic problem
\[
    \min_{\mathbf y}\ \frac12\|\mathbf y\|_{\mathbf W}^2
    \quad\text{s.t.}\quad
    \tilde{\mathbf M}_t\mathbf y=\mathbf U_t\Delta\mathbf b_t .
\]
With {Lagrange} multipliers $\boldsymbol \nu$, its KKT conditions are
\[
    \mathbf W\mathbf y-\tilde{\mathbf M}_t^\top\boldsymbol\nu=\mathbf0,
    \qquad
    \tilde{\mathbf M}_t\mathbf y=\mathbf U_t\Delta\mathbf b_t ,
\]
so, using the linear independence of the active rows,
\[
    \mathbf y
    =
    \mathbf W^{-1}\tilde{\mathbf M}_t^\top
    \left(\tilde{\mathbf M}_t\mathbf W^{-1}\tilde{\mathbf M}_t^\top\right)^{-1}
    \mathbf U_t\Delta\mathbf b_t
    =
    \frac{\partial \mathbf{x}_t^*}{\partial \mathbf{b}_t}\Delta\mathbf b_t .
\]
Strict convexity gives uniqueness, proving the minimum-norm
characterization.
\end{proof}

\begin{proof}{Proof of Proposition~\ref{prop:hybrid_rounding_gap}, part (i) (objective gap)}
Since
\(\mathbf x_t^*\in\mathcal A^{\mathrm{rel}}(\mathbf s_t)\subseteq\mathbb R_{\ge0}^{n}\),
the componentwise floor \(\mathbf x_t^F = \lfloor \mathbf x_t^* \rfloor\) satisfies
\(\mathbf 0 \le \mathbf x_t^F \le \mathbf x_t^*\), so downward closedness of
\(\mathcal A^{\mathrm{rel}}(\mathbf s_t)\) by Assumption (i) gives
\(\mathbf x_t^F \in \mathcal A^{\mathrm{rel}}(\mathbf s_t) \cap \mathbb N_0^{n} = \mathcal A(\mathbf s_t)\).
Write \(\boldsymbol{\delta}_t := \mathbf x_t^*-\mathbf x_t^F\) for the rounding remainder and
\(\Delta_t := \|\boldsymbol{\delta}_t\|_{\mathbf W}\) for its weighted norm. By definition of the integer mapping,
\(\mathbf x_t\in\mathcal A(\mathbf s_t)\) and
\(f_t(\mathbf x_t)\le f_t(\mathbf x_t^F)\). Since
\(\mathcal A(\mathbf s_t)\subseteq \mathcal A^{\mathrm{rel}}(\mathbf s_t)\)
and \(\mathbf x_t^I\) is integer optimal,
\[
0\le f_t(\mathbf x_t)-f_t(\mathbf x_t^I)
\le f_t(\mathbf x_t)-f_t(\mathbf x_t^*)
\le f_t(\mathbf x_t^F)-f_t(\mathbf x_t^*).
\]
Using \(\mathbf x_t^F=\mathbf x_t^*-\boldsymbol{\delta}_t\),
\[
f_t(\mathbf x_t^F)-f_t(\mathbf x_t^*)
=
\frac12\|\boldsymbol{\delta}_t\|_{\mathbf W}^2
+\boldsymbol{\delta}_t^\top \mathbf W(\mathbf z_t-\mathbf x_t^*)
\le \frac12\Delta_t^2+\Delta_t R_t,
\]
where the inequality follows from Cauchy--Schwarz. Since
\(0\le\delta_{i,t}<1\), we have
\(\Delta_t^2=\sum_{i\in[n]} w_i\delta_{i,t}^2\le\sum_{i\in[n]} w_i=1\)
by the weight normalization.
Thus \(f_t(\mathbf x_t)-f_t(\mathbf x_t^I)\le \frac12+R_t\). If
\(R_t\le\rho\),
inserting gives the first claim.

Now suppose \(R_t>\rho\). Since \(f_t(\mathbf x_t^*)=\frac12R_t^2\) and
\(f_t(\mathbf x_t^I)\ge f_t(\mathbf x_t^*)>0\),
\[
\frac{f_t(\mathbf x_t)-f_t(\mathbf x_t^I)}{f_t(\mathbf x_t^I)}
\le
\frac{\frac12\Delta_t^2+\Delta_t R_t}{\frac12R_t^2}
=
\frac{\Delta_t^2}{R_t^2}+2\frac{\Delta_t}{R_t}
\le
\frac{1}{\rho^2}+\frac{2}{\rho},
\]
using \(\Delta_t\le 1\) and \(R_t>\rho\). This proves the claim.
\end{proof}

\begin{proof}{Proof of Proposition~\ref{prop:hybrid_rounding_gap}, part (ii) (decision distance)}
We adopt \(\mathbf x_t^F\), \(\boldsymbol{\delta}_t\), and \(\Delta_t\) from the proof of part (i).
The first bound follows from the projection inequality:
\[
\frac12\|\mathbf x_t^I-\mathbf x_t^*\|_{\mathbf W}^2
\le
f_t(\mathbf x_t^I)-f_t(\mathbf x_t^*)
\le
f_t(\mathbf x_t^F)-f_t(\mathbf x_t^*).
\]
By the same expansion and Cauchy--Schwarz step as in the proof of part (i),
\[
f_t(\mathbf x_t^F)-f_t(\mathbf x_t^*)
\le
\frac12\Delta_t^2+\Delta_t R_t.
\]
Hence
\(\|\mathbf x_t^I-\mathbf x_t^*\|_{\mathbf W}\le\sqrt{\Delta_t^2+2\Delta_t R_t}\).
Since \(0\le\delta_{i,t}<1\), we also have \(\Delta_t\le 1\), giving
\(\|\mathbf x_t^I-\mathbf x_t^*\|_{\mathbf W}\le\sqrt{1+2R_t}\).

It remains to bound the distance from \(\mathbf x_t\) to \(\mathbf x_t^*\).
Since
\(\mathbf x_t=\mathbf x_t^F+\boldsymbol{\eta}_t
=\mathbf x_t^*-\boldsymbol{\delta}_t+\boldsymbol{\eta}_t\), where
\(\boldsymbol{\eta}_t\in\{0,1\}^{n}\) collects the unit increments accepted by
\(\mathrm{IM}_t\), each coordinate satisfies
\(|x_{i,t}-x_{i,t}^*|=|\eta_{i,t}-\delta_{i,t}|\le 1\). Therefore
\(
\|\mathbf x_t-\mathbf x_t^*\|_{\mathbf W}^2
=
\sum_{i\in[n]} w_i(\eta_{i,t}-\delta_{i,t})^2
\le 1.
\)
By the triangle inequality,
\[
\|\mathbf x_t-\mathbf x_t^I\|_{\mathbf W}
\le
\|\mathbf x_t-\mathbf x_t^*\|_{\mathbf W}
+\|\mathbf x_t^I-\mathbf x_t^*\|_{\mathbf W}
\le
1+\sqrt{\Delta_t^2+2\Delta_t R_t} \le
1+\sqrt{1+2R_t}.
\]
If \(R_t\le\rho\),
monotonicity in \(R_t\) gives the first claim. If \(R_t>\rho\), dividing by \(R_t\) along with monotonicity yields the second claim.
\end{proof}

\begin{proof}{Proof of Proposition~\ref{prop:completeness}, part (a) (completeness of the integer mapping)}
Fix \(\mathbf q\in\mathcal A(\mathbf s_t)\). Since \(\mathbf q\) is relaxed feasible, \(P_t(\mathbf q)=\mathbf q\). By continuity of the projection map, choose {\(0<\bar\varepsilon<1/2\)} such that, for every
\(\mathbf z_t=\mathbf q+\boldsymbol\varepsilon\in\mathbb R_{\ge0}^{n}\) with {\(|\varepsilon_i|<\bar\varepsilon\)} when \(q_i\ge1\) and
{\(0\le\varepsilon_i<\bar\varepsilon\)} when \(q_i=0\),
\[
    \|P_t(\mathbf z_t)-\mathbf q\|_\infty<1 .
\]
Thus, for \(\mathbf x_t^*=P_t(\mathbf z_t)\), the floor point satisfies
\(x_{i,t}^F=0\) if \(q_i=0\), and
\(x_{i,t}^F\in\{q_i-1,q_i\}\) if \(q_i\ge1\). Hence \(\mathbf x_t^F\)
is componentwise below \(\mathbf q\) and differs from it by at most one unit in each coordinate.

During the IM traversal, every intermediate point below \(\mathbf q\) is feasible, because \(\mathbf q\in\mathcal A(\mathbf s_t)\) and \(\mathcal A(\mathbf s_t)\) is downward closed by Assumption (i). Moreover, the one-step objective change
\(f_t(\mathbf x+\mathbf e_i;\mathbf z_t)-f_t(\mathbf x;\mathbf z_t)=w_i(x_i+\tfrac12-z_{i,t})\)
depends only on coordinate \(i\), so the acceptance decisions below do not depend on the traversal order induced by the priority scores. If the current coordinate satisfies
\(x_i=q_i-1\), then
\[
    f_t(\mathbf x+\mathbf e_i;\mathbf z_t)
    -
    f_t(\mathbf x;\mathbf z_t)
    =
    w_i(q_i-\tfrac12-z_{i,t})
    =
    -w_i(\tfrac12+\varepsilon_i)<0,
\]
so IM accepts the missing increment. If \(x_i=q_i\), then
\[
    f_t(\mathbf x+\mathbf e_i;\mathbf z_t)
    -
    f_t(\mathbf x;\mathbf z_t)
    =
    w_i(q_i+\tfrac12-z_{i,t})
    =
    w_i(\tfrac12-\varepsilon_i)>0,
\]
so IM rejects any increment beyond \(\mathbf q\). Therefore every target
in this relative neighborhood of \(\mathbf q\) maps to \(\mathbf q\).
The preimage of \(\mathbf q\) thus has positive Lebesgue measure, and
\(\Phi_t^{\mathrm{IM}}\) is complete.
\end{proof}

\begin{proof}{Proof of Proposition~\ref{prop:completeness}, part (a) (completeness of flooring on slack actions)}
Fix $\mathbf q\in\mathcal A^{\mathrm{slack}}(\mathbf s_t)$. By definition no $\mathbf A$ or $\mathbf C$ constraint is active, so $\mathbf A\mathbf q<\mathbf b_t$ and $\mathbf C\mathbf q<\mathbf k$ componentwise. Hence there is $\varepsilon\in(0,1)$ with $\mathbf A(\mathbf q+\varepsilon\mathbf 1)\le\mathbf b_t$ and $\mathbf C(\mathbf q+\varepsilon\mathbf 1)\le\mathbf k$, i.e., $\mathbf q+\varepsilon\mathbf 1\in\mathcal A^{\mathrm{rel}}(\mathbf s_t)$. By downward-closedness the box $[\mathbf 0,\mathbf q+\varepsilon\mathbf 1]\subseteq\mathcal A^{\mathrm{rel}}(\mathbf s_t)$, and it is full-dimensional because $\mathbf q+\varepsilon\mathbf 1>\mathbf 0$. Hence $\mathbf p:=\mathbf q+\tfrac{\varepsilon}{2}\mathbf 1$ lies in its interior, and therefore in the interior of $\mathcal A^{\mathrm{rel}}(\mathbf s_t)$, denoted $\operatorname{int}\mathcal A^{\mathrm{rel}}(\mathbf s_t)$, as well as in the open cell $(\mathbf q,\mathbf q+\mathbf 1)$. Choose a ball $B(\mathbf p,\delta)\subseteq\operatorname{int}\mathcal A^{\mathrm{rel}}(\mathbf s_t)\cap(\mathbf q,\mathbf q+\mathbf 1)$. For every $\mathbf z_t\in B(\mathbf p,\delta)$ we have $\mathbf z_t\in\mathcal A^{\mathrm{rel}}(\mathbf s_t)$, so the projection fixes it, $P_t(\mathbf z_t)=\mathbf z_t$, and therefore $\lfloor P_t(\mathbf z_t)\rfloor=\lfloor\mathbf z_t\rfloor=\mathbf q$. Thus $B(\mathbf p,\delta)\subseteq(\Phi_t^{\mathrm{F}})^{-1}(\mathbf q)$ has positive Lebesgue measure. Since $\mathbf q\in\mathcal A^{\mathrm{slack}}(\mathbf s_t)$ was arbitrary, $\Phi_t^{\mathrm{F}}$ is complete on $\mathcal A^{\mathrm{slack}}(\mathbf s_t)$.
\end{proof}

\begin{proof}{Proof of Proposition~\ref{prop:completeness}, part (b) {(incompleteness of flooring on active actions)}}
Fix any positive weights $\mathbf W$ and take the single resource constraint $x_1 + x_2 \le 2$ and {$\mathbf q = [1,1]^\top$}, which is active since $x_1 + x_2 = 2$ at $\mathbf q$, so $\mathbf q\in\mathcal A^{\mathrm{act}}(\mathbf s_t)$. The only point of $\mathcal A^{\mathrm{rel}}(\mathbf s_t)$ in $[1,2)^2$ is $\mathbf q$ itself, since $x_1, x_2 \ge 1$ forces $x_1 + x_2 \ge 2$ with equality only at $\mathbf q$. Hence $\Phi_t^{\mathrm{F}}(\mathbf z_t) = \mathbf q$ if and only if $P_t(\mathbf z_t) = \mathbf q$. By the optimality condition of the weighted projection, $\mathbf x_t^* = P_t(\mathbf z_t)$ if and only if $\mathbf x_t^*\in\mathcal A^{\mathrm{rel}}(\mathbf s_t)$ and $\mathbf W(\mathbf z_t - \mathbf x_t^*) \in N(\mathbf x_t^*)$, where $N(\mathbf x)$ denotes the normal cone of $\mathcal A^{\mathrm{rel}}(\mathbf s_t)$ at $\mathbf x$. The floor preimage of $\mathbf q$ therefore equals $P_t^{-1}(\mathbf q) = \mathbf q + \mathbf W^{-1}N(\mathbf q)$. Only the resource constraint is active at $\mathbf q$, so {$N(\mathbf q) = \{[a,a]^\top : a\ge 0\}$} and {$\mathbf W^{-1}N(\mathbf q) = \{[a/w_1, a/w_2]^\top : a\ge 0\}$}, whence the preimage is a one-dimensional ray of Lebesgue measure zero in $\mathbb R_{\ge 0}^2$. Since $\mathbf W$ was arbitrary, $\Phi_t^{\mathrm{F}}$ is not complete on $\mathcal A^{\mathrm{act}}(\mathbf s_t)$ for any positive weights.
\end{proof}

\begin{proof}{Proof of Proposition~\ref{prop:frontier_reachability}}
Let \(\mathbf y := \Phi_t^{\mathrm{IM}}(\mathbf z_t)\) and let
\(S \subseteq [n]\) denote the set of coordinates the IM
incremented, so \(y_i = x_{i,t}^F + 1\) for \(i\in S\)
and \(y_i = x_{i,t}^F\) for \(i\notin S\). It suffices to
exhibit some \(i\in[n]\) with \(\mathbf y + \mathbf e_i \notin
\mathcal A(\mathbf s_t)\).

\textit{Step 1: \(S\ne[n]\).} Suppose, for contradiction, that
\(S = [n]\), so \(\mathbf y = \mathbf x_t^F + \mathbf 1\).
Because the IM accepts only feasible increments,
\(\mathbf y\in\mathcal A(\mathbf s_t)\subseteq
\mathcal A^{\mathrm{rel}}(\mathbf s_t)\). Fix \(i\in[n]\) and write
\(r_i := z_{i,t}-x_{i,t}^*\) for the projection residual and
\(c_i := y_i-x_{i,t}^* = 1-\bigl(x_{i,t}^*-x_{i,t}^F\bigr)\) for the step taken
in coordinate \(i\). Since \(0\le x_{i,t}^*-x_{i,t}^F<1\) and \(r_i>\tfrac12\)
by assumption, we have \(0<c_i\le1<2r_i\) and therefore
\(
    |z_{i,t}-y_i| = |r_i-c_i| < r_i = |z_{i,t}-x_{i,t}^*| ,
\)
that is, the increment moves coordinate \(i\) {toward} the target without
overshooting it far enough to increase the distance. Weighting by \(w_i>0\) and
summing over \(i\in[n]\) gives
\(f_t(\mathbf y;\mathbf z_t)<f_t(\mathbf x_t^*;\mathbf z_t)\) for the
relaxed-feasible point \(\mathbf y\), contradicting the optimality of
\(\mathbf x_t^*\) in \eqref{eq:qp_objective}. Hence \(S\ne[n]\).

\textit{Step 2: pick \(i\in[n]\setminus S\) and show
\(\mathbf y + \mathbf e_i \notin \mathcal A(\mathbf s_t)\).} The IM
considered coordinate \(i\) with \(x_i = x_{i,t}^F\) at some intermediate state \(\mathbf x\le\mathbf y\) and
rejected the unit increment. That increment strictly decreases the
projection objective,
\(
    f_t(\mathbf x + \mathbf e_i;\mathbf z_t) - f_t(\mathbf x;\mathbf z_t)
    = w_i\bigl(x_{i,t}^F + \tfrac12 - z_{i,t}\bigr) < 0 ,
\)
because \(z_{i,t}-x_{i,t}^F\ge z_{i,t}-x_{i,t}^*>\tfrac12\). Since the IM
accepts every feasible increment that does not increase the objective, the
rejection must be due to infeasibility,
\(\mathbf x + \mathbf e_i \notin \mathcal A(\mathbf s_t)\), and downward
closedness lifts this to \(\mathbf y+\mathbf e_i\notin\mathcal A(\mathbf s_t)\)
because \(\mathbf x+\mathbf e_i\le\mathbf y+\mathbf e_i\).
Hence \(\mathbf y \in \mathcal F(\mathbf s_t)\).
\end{proof}

\end{APPENDIX}

\end{document}